\documentclass[11pt, a4paper]{article}
\usepackage{preamble}
\title{Manifold limit for the training of
shallow\\ graph convolutional
neural networks}
\author{Johanna Tengler$^\ast$, Christoph Brune$^\ast$, and José A. Iglesias\thanks{Mathematics of Imaging \& AI, Department of Applied Mathematics, University of Twente, the Netherlands\\ \hspace*{5mm}(\texttt{j.tengler{@}utwente.nl, c.brune{@}utwente.nl, jose.iglesias{@}utwente.nl}).}}
\date{}

\begin{document}

\maketitle

\begin{abstract}
    We study the discrete-to-continuum consistency of the training of shallow graph convolutional neural networks (GCNNs) on proximity graphs of sampled point clouds under a manifold assumption. Graph convolution is defined spectrally via the graph Laplacian, whose low-frequency spectrum approximates that of the Laplace-Beltrami operator of the underlying smooth manifold, and shallow GCNNs of possibly infinite width are linear functionals on the space of measures on the parameter space. From this functional-analytic perspective, graph signals are seen as spatial discretizations of functions on the manifold, which leads to a natural notion of training data consistent across graph resolutions. To enable convergence results, the continuum parameter space is chosen as a weakly compact product of unit balls, with Sobolev regularity imposed on the output weight and bias, but not on the convolutional parameter. The corresponding discrete parameter spaces inherit the corresponding spectral decay, and are additionally restricted by a frequency cutoff adapted to the informative spectral window of the graph Laplacians. Under these assumptions, we prove $\Gamma$-convergence of regularized empirical risk minimization functionals and corresponding convergence of their global minimizers, in the sense of weak convergence of the parameter measures and uniform convergence of the functions over compact sets. This provides a formalization of mesh and sample independence, and hence of transferability across discretizations, for the training of such networks.

    \vspace{0.3cm}\noindent
\end{abstract}
\blfootnote{2020 Mathematics Subject Classification (MSC): 68T07, 46N10, 49J45, 58J50}

\listoftodos
\section{Introduction}
Across a variety of machine learning scenarios, it is common to assume that data points lie on a smooth manifold, which is commonly referred to as the manifold assumption. One such scenario involves high-dimensional data where the intrinsic dimensionality is believed to be much lower. Another is when data is sampled on points of a physical surface in three-dimensional space, as can be the case in applications of medical imaging, computer vision, and scientific machine learning. In both situations, local extrinsic distances between data points play a central role in determining the relation between them, and using these explicitly provides both reductions of complexity that makes formulating algorithms on such datasets practical, and improved predictions through respecting the intrinsic structure of the data. This role of such distances is apparent in unsupervised clustering algorithms such as Cheeger cuts \cite{Chu97, Lux07} and in semi-supervised learning \cite{EngHoo19} where only a small number of data points are labeled, and the eventual predictions are in large part determined by the geometry of the assumed underlying smooth manifold.

Constructing geometric proximity graphs over such samples with an appropriately selected neighborhood scale yields a graph Laplacian whose low-frequency spectrum accurately approximates that of the Laplace-Beltrami operator of the underlying manifold (see e.g. \cite{Belkin:2006}, \cite{Burago:2015}, \cite{GarGerHeiSle20}). 
This fact has been extensively leveraged in semi-supervised learning scenarios in which one seeks to assign one label to each data point, and their consistency when the number of data points tends to infinity \cite{TriKapSamSan20, DunSleStuTho20, CalSleTho23}, resulting in limits represented by variational energies on Sobolev spaces and partial differential equations, which help to shed light onto the initial discrete problems. Here, instead we consider problems taking \emph{graph signals} defined on those data points as input and assigning labels to such signals. In the first instance, we are concerned with supervised classification of those graph signals assigning a real-valued label, but we expect our techniques to generalize to more complex outputs.

This type of task coupled with the manifold assumption mentioned above naturally leads to the use of geometric machine learning methods (see \cite{BroBruCohVel21} for a general introduction), which have seen a surge of interest in recent years and in a wide variety of applications. One of the cornerstone architectures in this paradigm is that of graph convolutional neural networks (GCNNs) \cite{KipWel17, GamIsuLeuRib20}, where convolution is defined spectrally as a weight sharing method that leverages the intrinsic structure of the data.
 This naturally raises questions of mesh independence and transferability, namely to what extent a learned model depends on the particular graph discretization on which it was trained, and whether it is transferable to a different discretization of the same underlying structure. Within this work, we take an admittedly restrictive definition of transferability referring to closeness of the outputs of networks arising from different random realizations and resolutions but trained with data which is consistent between them. 

In this work, we approach these questions through studying the consistency of training shallow GCNNs under graph refinement in the following sense: assuming that graphs and training signals arise as discretizations or samplings of underlying continuum structures, we show that learning outcomes of shallow GCNNs remain stable under refinement of the discretization. More precisely, we establish convergence of the network realizations under convergence of discrete parametrizations and show that, up to subsequences, limits of trained discrete parametrizations yield trained manifold networks.

\paragraph*{Roadmap of the approach.}

Our approach is motivated by two functional-analytic considerations. 
First, the Euclidean inner product used in the spectral definition of graph convolution is itself an $L^2$-structure, and in the context of discretizations, it approximates the canonical $L^2$-inner product on the manifold. This $L^2$-structure provides the natural analytic setting for signal and parameter spaces.
Second, shallow neural networks admit a dual representation as linear functionals over measures on the parameter space, encompassing both infinite-width limits as well as finite networks.
This functional-analytic perspective allows us to analyze shallow GCNNs via weak limits in terms of neuron evaluation functionals or weak-* limits in terms of parameter distributions, and to study their training through the $\Gamma$-convergence of empirical risk minimization (ERM) functionals. Within this framework we show our main convergence results, namely that if training data across graph resolutions describe the same phenomenon, that is, the graph signals are discretizations of respective continuum signals, then the discrete training problems are indeed a \emph{consistent discretization} of their continuum counterpart.  

Motivated by the Convolution Theorem for Euclidean domains, convolution on graphs and manifolds is defined as multiplication in the spectral domain which gives rise to an $L^2$-structure of the signal and parameter spaces.
Here, the convolution operations of signals on graphs and manifolds are defined via the eigendecompositions of the graph Laplacians and the Laplace-Beltrami operator. The space of signals on the nodes $\M_n$ of a graph is simply $\R^n$, which naturally extends to the continuum setting, where signals are modeled as functions on a manifold $\M$. A coherent view of this is to define signals as $L^2$-functions defined on the nodes of the graphs or on the manifold, respectively. This perspective also applies to the parameter spaces with $\Theta_n\subset (\R^n)^3$ and $\Theta\subset (L^2(\M))^3$.

Although trained networks are nonlinear in their parameters these parametrizations become linear when expressed by measures over the parameter space (see e.g. \cite{Bach:2017}).
These linear parametrizations are well-defined by the Riesz-Markov representation theorem provided the parameter spaces are locally compact and such that, for every fixed signal, the neural response maps 
are continuous in the parameter triple and vanishing at infinity. This is easily satisfied in the finite-dimensional setting by restricting the parameter spaces to be closed balls, while in the infinite-dimensional setting the continuity and compactness requirement compete, which calls for the right choice of topology.
Restriction to closed balls with the weak topology on $L^2$ achieves compactness, while the inner product structure of the maps $\psi$, $\psi_n$ taking a signal and a parameter to the corresponding neuron output, in the following referred to as neural response maps, requires additional $H^\alpha$ regularity on the first and last parameter to retain control over the $L^2$-norm. This is important for the required continuity of the neural response maps, but also required for the uniform convergence result. Since we want to infer consistency from the spectral approximation this requires imposing this regularity assumption also on the first and last parameter of the discrete parameter spaces, as well as a truncation to the meaningful part of the graph spectra, i.e. those elements spanned by the first $K(n)$ eigenvectors, where $K(n)$ denotes the frequency-cutoff.

The duality framework thus provides a natural foundation for studying approximation, convergence, and generalization in shallow GCNNs.
To study the discrete-to-continuum limits, we employ spatial and spectral projections from the continuum to the discrete spaces, along with their respective adjoint operators. The spectral projections map onto the $K(n)$-dimensional subspaces of the discrete or continuum $L^2$ or $H^\alpha$ spaces spanned by their respective eigenbases. Combining those we define a projection between the continuum and discrete parameter spaces which can be seen as a minimal norm extension operator.
\paragraph*{Main results.} We show that the discrete training problems converge to a continuum training problem under projection of discrete parametrizations onto the continuum parameter space:

First, we establish uniform convergence of the neural responses to signals that originate from the same signal on the manifold: $$\psi_n(u_n, \mQ_{n, \alpha}(\cdot))\xrightarrow{n\to\infty} \psi(u, \cdot)\;\text{ in } C(\Theta, w),\quad \text{ for all }u\in L^2(\M, \mu).$$
This ensures consistency of neuron evaluation across resolutions for graph signals whose limit is an underlying manifold signal, in terms of the $TL^2$ convergence first introduced in \cite{GarciaTrillos:2015}.

Second, we prove $\Gamma$-convergence of the total variation regularized empirical risk minimization (ERM) functionals. Given input-output pairs from a target functional on the space of manifold signals, the discrete ERM functionals (defined on $\cM(\Theta)$) $\Gamma$-converge to the continuum ERM functional restricted to a compact ball in $\cM(\Theta)$ equipped with the weak-* topology. The discrete ERM functionals are extended to the space $\cM(\Theta)$ by assigning the value $+\infty$ to measures that do not correspond to valid projections of measures on the discrete parameter space, that is if $(\mQ_{n, \alpha}^*\mQ_{n, \alpha})_\# \nu \neq \nu$.
    As a result, minimizers of the discrete ERM functionals converge weak-* to minimizers of the continuum ERM functional, ensuring consistency of the trained networks.

\paragraph*{Organization of the paper.}
In \Cref{setup-and-notation} we establish the considered setting, terms and notation of this paper, namely we define shallow GCNNs and formalize the manifold assumption, and the notion of consistency used in this paper. In \Cref{sec:main-results-discussion} we summarize our main results and discuss the construction of the parameter spaces as well as some restrictions of the shallow GCNNs studied in this paper. Furthermore, we describe a synthetic classification example to underpin the main results and the discussion.  In \Cref{spectral-approximation} we collect spectral approximation results from \cite{GarGerHeiSle20} and \cite{Burago:2015}, verify their applicability to our setting and state them in a form that is convenient for the sequel.
In \Cref{discrete-continuum-erm-convergence} we define spectral projection maps between the discrete and continuum spaces and prove some preliminary convergence results about shallow GCNNs in order to finally state and prove our main results \Cref{uniform-convergence} and \Cref{gamma-convergence} in \Cref{gamma}. 
\paragraph*{Related works.}
There are several works studying transferability, approximation and generalization properties of GCNNs, such as \cite{LevEtAl21, Le:2023, MasKutLev25, Maskey:2023, Wang:2022, Wang:2024}. 

Transferability compares the outputs of a model on different graphs, for instance on graphs of different sizes, relying on some notion of graphs describing the same underlying structure.  In contrast, in the present work, we are interested in the consistency of the training of GCNNs, so in convergence of minimizers of risk objectives corresponding to graphs consistent across resolutions with consistent training data. 

The authors in \cite{LevEtAl21} analyze transferability of GCNNs by the means of the error of the filter of a continuum Laplacian applied to a signal with the interpolation of the application of a filter of a graph Laplacian to a discretized signal. They show that this transferability error of the filter is controlled by the transferability error of the Laplacian and a so-called signal consistency error, i.e. the error between the signal and the interpolation of its discretization. While we also consider graphs that are consistent in the sense that they are sampled from the same underlying space, in our setting the graph Laplacians are required to approximate the Laplace-Beltrami spectrally in a low-frequency window, while \cite{LevEtAl21} consider graphs whose action on sampled continuum low-frequency modes approximates the action of the continuum-Laplacian under sampling. In their paper, the space of graph signals is limited to have finite bandwidth. Indeed, the fact that only the low-frequency graph-spectrum reliably approximates that of the underlying continuum space makes restriction to low-frequencies unavoidable, however we impose such restrictions on the parameter space, and adapt the discretization scheme accordingly.

Some works on the manifold convolution setting we treat here are \cite{Wang:2022} and \cite{Wang:2024}, where the authors first claim a bound on the error of commuting the filter operation with the discretization, and subsequently, for filters where this bound is uniform across signals, they show that this error remains controlled through the application of the nonlinearity.
Notably, the discretization considered in \cite{Wang:2022, Wang:2024} again requires the space of input manifold signals to be band-limited. We prove a similar type of bound in \Cref{convergence-discrete-continuum-convolution-operation}, which we later use to show consistency of the training, as in $\Gamma$-convergence of the empirical risk minimization functionals. Furthermore, in \cite{Wang:2024} the authors consider manifold filters whose response to high frequencies is almost constant. Rather than imposing assumptions on the manifold filter, we shift the restriction to the discretized filter by truncating its spectral representation to the first $K(n)$ frequencies, where the spectral cutoff $K(n)$ is chosen such that it diverges. However, this requires imposing an assumption on the asymptotic behavior of the spectral gap of the manifold. In particular we provide a quantification of this admissible spectral window in terms of resolution and this gap assumption.

There are also a number of works relying on general notions of convergence and limiting objects for graphs applicable beyond the geometric setting. One such framework is that of graphons, as considered for instance in \cite{Maskey:2023}, which does not apply in our setting for two reasons. The first is technical, in that graphons arise as limit objects of sequences of dense graphs, whereas the geometric proximity graphs considered here are such that their number of edges per node is bounded relative to the local sampling density, due to being constructed from a kernel with compact support. The second reason is conceptual, since in graphon-based models the locations of nodes do not play a role and the emphasis is in their relations through the edges, while in the geometric setting the main object are the node locations, with the edges chosen only to ensure enough communication between them. Another type of convergence that is applicable to sparser graphs is the notion of action convergence of graphops as used in \cite{Le:2023}. However, our models are based on convolutions defined in the spectral domain, and as the authors point out, the relations between spectral definitions and this type of convergence are not clear at present.

Finally, let us also mention that further architectures based on spectral constructions are considered in \cite{LevMonBreBro19}, and generalization results derived from the manifold assumption and upscaling in the number of graph vertices are provided in \cite{WanCerRib25}.

\section{Setup and notation}\label{setup-and-notation}
A comprehensive list of the notation used in the paper is provided in Table \ref{tab:TableOfNotation}. In the remainder of this section, we introduce the framework for our subsequent results.

\subsection{Point cloud proximity graphs and the manifold assumption}\label{point-cloud-graphs-manifold-ass}
The manifold assumption refers to the physically motivated idea that data points, which are a priori vectors in $\R^d$, lie on a smooth manifold embedded in $\R^d$, implying an underlying geometric structure of the data sample that one can attempt to exploit. By connecting data points if their Euclidean distance is below a certain threshold, the neighborhood size, one obtains a neighborhood graph from the point cloud, a spatial discretization of the underlying manifold. However, what is ultimately desired is to induce the geometry of the underlying Riemannian manifold onto the graph. The geometry of a Riemannian manifold is completely determined by its Riemannian metric, which itself can be recovered from the Laplace-Beltrami operator: to see this, one can just notice that its principal symbol, the cometric in \eqref{eq:coordLBO}, can be recovered through the action of the operator on coordinate functions, see \cite[Sec.~2.9]{Tay23}. This suggests that if one has a suitable notion of a graph Laplacian approximating the former in an appropriate sense, then the graphs are also a geometric approximation. For a compact Riemannian manifold, the Laplace–Beltrami operator has a discrete spectrum, and thus a natural notion of approximation between such Laplacians is spectral convergence, which has been studied extensively (e.g. \cite{Hein:2005, Hein:2007, Singer:2006, Belkin:2006, Burago:2015, GarciaTrillos:2018, GarGerHeiSle20}).

To obtain a good approximation of the lower-frequency spectrum, the neighborhood size must shrink appropriately with increasing resolution, i.e. the number of nodes, but also account for the sampling error. This will be made explicit in the following, and elaborated more in \Cref{spectral-approximation}, where we will state the main results, Theorems 4 (or its Corollary 1) and 5, from \cite{GarGerHeiSle20} which establish convergence rates for the convergence of eigenvalues and eigenvectors of the graph Laplacians towards eigenvalues and eigenfunctions of the Laplace-Beltrami operator on the underlying manifold.

Let $\M \subset \R^d$ be an $m$-dimensional compact connected $C^\infty$ Riemannian submanifold without boundary embedded in $\R^d$, with $d\geq m$. 

We denote by $\mu \in \cP(\R^d)$ (Radon probability measures in $\R^d$) the normalized $m$-dimensional Hausdorff measure restricted to $\M$, that is
\[\mu = \frac{1}{\H^m(\M)} \H^m \mres \M, \quad \text{where }(\H^m \mres \M)(A) = \H^m(A \cap \M) \text{ for all Borel sets } A \subseteq \R^d,\]
which satisfies $\mu ={\Vol_\M(\M)}^{-1}\Vol_\M$,
where $\Vol_\M$ is the volume form on $\M$ arising from the Riemannian metric $g$ induced by the embedding of $\M$ into $\R^d$,
since following \cite[Prop.~12.6]{Taylor:2006} it holds that \[\Vol_\M(A)=\H^m(A)=(\H^m\mres \M)(A) \qquad \text{ for every Borel set } A\subset \M.\]
From $\mu$ we sample an i.i.d. sequence of points $(x_n)_{n \in \N} \subset \R^d$, and denote the first $n$ of them and corresponding empirical measure as 
\[\M_n := \{x_1, \ldots, x_n\}, \quad \mu_n := \frac{1}{n} \sum_{i=1}^n \delta_{x_i}.\]
By the general version in \cite{Var58} (see also \cite[Thm.~11.4.1]{Dud02}) of the Glivenko-Cantelli theorem, almost surely the measures $\mu_n$ converge weakly to $\mu$. In the following, we restrict ourselves to samplings such that the corresponding empirical measures $\mu_n$ converge weakly to $\mu.$ Moreover, this convergence can be quantified in terms of the $\infty$-Wasserstein distance  
\begin{equation}\label{infty-wasserstein}\eps_n:= d_\infty(\mu, \mu_n) := \inf_{T_\# \mu = \mu_n} \esssup_{x \in \M} d_{\M}(x,T(x)),\end{equation}
where $d_{\M}$ is the Riemannian geodesic distance on $\M$ and the infimum is taken over all Borel measurable maps from $\M$ to $\M$. The bounds
\begin{equation}\label{probabilistic-bounds-eps}
    \eps_n \leq C \frac{(\log n)^{1/m}}{n^{1/m}} \text{ for } m \geq 3, \quad \eps_n \leq C \frac{(\log n)^{3/4}}{n^{1/2}} \text{ for } m = 2,
\end{equation}
hold with high probability by \cite[Thm.~2]{GarGerHeiSle20}, which in turn is a generalization to the manifold setting of the classical bounds with the same rate for the uniform distribution on $(0,1)^m$ of \cite{ShoYuk91} and \cite{LeiSho89}.

Let $\{h_n\}_{n\in\N}\subset \R_{>0}$ be a sequence of neighborhood sizes subject to \cite[Assumption 3]{GarGerHeiSle20}, i.e. such that
    \begin{equation}\label{assumption3}
        (m+5)\eps_n < h_n<\min\left\{1, \frac{i_0}{10}, \frac{1}{\sqrt{mK}}, \frac{R}{\sqrt{27m}}\right\}, \quad \text{ and }\quad \frac{\eps_n}{h_n}\xrightarrow{n\to\infty}0,
    \end{equation}
where $i_0$ is the injectivity radius of the manifold $\M$, $K$ a global upper bound on the absolute value of sectional curvatures of $\M$, and $R$ the reach of $\M$.
Given an interaction strength $\eta:[0,+\infty) \to [0,+\infty)$ with support in $[0,1]$ and normalized such that
\[\int_{\R^m} \eta(|x|) \dd x = 1,\] 
we define, as in \cite{GarGerHeiSle20}, the weights
\begin{equation}\label{graph-weights-eta}W_n = \big(w_{ij}\big)_{i,j=1}^n, \quad w_{ij} = \frac{1}{n h_n^m} \eta \left(\frac{|x_i-x_j|}{h_n}\right),\end{equation}
in which $|x_i-x_j|$ denotes the Euclidean distance in $\R^d$. Putting these together, we get an undirected graph
$G_n = (\M_n, W_n)$ where we keep the edge set implicit in the weights, that is, two points/vertices $x_i,x_j$ are connected if and only if $w_{ij} >0$. Note that, taking into account the bounds on $\eps_n$ from \eqref{probabilistic-bounds-eps}, such a choice of neighborhood sizes $h_n$ subject to the second condition of \eqref{assumption3} is the one corresponding to the critical regime above which the graphs start to be connected with high probability (see \cite[Ch.~13]{Pen03} or \cite[Thm.~3.4]{BobKah18}) a connected graph for every $n\in\N$.

We define the ``surface tension'' associated to $\eta$ as
\[\sigma_{\!\eta} := \int_{\R^m} |x \cdot e_1|^2 \,\eta(|x|) \dd x.\]
This kind of quantity naturally appears as multiplicative factor in the convergence of functionals defined by double integrals to those involving classical derivatives as started in the pioneering work \cite{Bourgain:2001}, see also the proof of \cite[Lem.~4.2]{GarciaTrillos:2015} for an illustrative computation deriving its appearance. Accounting for this factor to aim for convergence to the continuous counterpart, for a function $v:\M_n \to \R$, the unnormalized graph Laplacian $\Delta_n v: \M_n \to \R$ is defined by
\begin{equation}\label{def-unnormalized-graph-laplacian}
    (\Delta_n v)(x_i) = \frac{2}{\sigma_{\!\eta} h_n^2} \sum_{j=1}^n w_{ij} \big(v(x_i) - u(x_j)\big).
\end{equation}
On the continuum level, the Laplace-Beltrami operator on $\M$ is defined for $v \in C^2(\M)$ by 
\[ \Delta v = - \frac{1}{\Vol_{\M}(\M)}\div_g \nabla_g v, \quad \text{ so that } \int_\M w \Delta v \dd \Vol_{\M} = \int_\M g(\nabla_g w, \nabla_g v) \dd \mu,\]
where $\nabla_g$ and $\div_g$ are the gradient and divergence arising from the metric $g$ on $\M$. In coordinates, this leads to the expression
\begin{equation}\label{eq:coordLBO} \Delta v = -\frac{1}{\Vol_{\M}(\M)\sqrt{\det g}} \sum_{i,j=1}^m \partial_j \left( \sqrt{\det g} \, g^{ij}\,\partial_i v \right).\end{equation}
Denote by  $\lambda^\bn_k$ the $k$-th eigenvalue of the $n$-th graph Laplacian $\Delta_n$, where $0=\lambda^\bn_1\leq\lambda^\bn_2 \leq \ldots \lambda^\bn_n$.
Since $\M$ is a compact Riemannian manifold, its Laplace-Beltrami operator has a discrete spectrum $0=\lambda_1\leq\lambda_2 \leq \ldots \uparrow +\infty$ (see e.g.\cite[Thm.~4.3.1]{Lablee:2015}). Accordingly, we denote by $\lambda_k$, $k\in \N$ the $k$-th eigenvalue of the Laplace-Beltrami operator on $\M$. Moreover, we denote by 
\begin{equation}\label{def-spectral-gap}\gamma_{k}:=\min\big\{|\lambda_k-\lambda_j| \mid j\in\N,  \lambda_k\neq \lambda_j\big\}>0\end{equation}
the gap between the $k$-th eigenvalue and any other distinct one. Note that by definition this quantity is nonzero even in the presence of multiplicity for $\lambda_k$. Furthermore denote \[\Gamma_{K}:= \min_{1\leq k\leq K}\gamma_{k}.\]

For two sequences $\{a_n\}_n, \{b_n\}_n$ we denote by $0<a_n\ll b_n$ that $a_n/b_n\xrightarrow{n\to\infty}0$. 
\par Then let $\{\widetilde{K}(n)\}_{n\in\N}$ be a sequence of $\N\ni\widetilde{K}(n)\leq n$ which, for $\{h_n\}_{n\in\N}$ subject to \eqref{assumption3}, satisfies
    \begin{equation}\label{sequence-conditions}0<\eps_n\ll h_n\ll \min\left\{\Gamma_{\widetilde{K}(n)}\widetilde{K}(n)^{-\frac{m+2}{2m}}, \widetilde{K}(n)^{-1}\right\}\ll 1, \quad\text{ and }\quad \eps_n h_n^{-1} \ll 
    \Gamma_{\widetilde{K}(n)}\widetilde{K}(n)^{-\frac{1}{2}}.
    \end{equation}
For any choice of  neighborhood radii $\{h_n\}_{n\in\N}\subset(0,1)$ fulfilling \eqref{assumption3}, such a sequence  $\{\widetilde{K}(n)\}_{n\in\N}$ always exist, which can be observed as follows:
For $n\in\N$ define the sets 
\begin{align*}
    \mathcal{J}_n&:=\{1\leq K\leq n \mid \forall l\geq n: h_l^{1/2}<K^{-1}\},\\
    \mathcal{K}_n&:=\{1\leq K\leq n \mid \forall l\geq n: h_l^{1/2}<\Gamma_K K^{-\frac{m+2}{2m}}\},\\
    \mathcal{L}_n&:=\{1\leq K\leq n \mid \forall l\geq n: \eps_l^{1/2}h_l^{-1/2}<\Gamma_K K^{-\frac{1}{2}}\}.
\end{align*}
Now, define \begin{align*}
    \widetilde{K}(n):=\begin{cases}
        1, & \text{ if } \mathcal{J}_n\cap\mathcal{K}_n\cap\mathcal{L}_n=\emptyset,\\
        \max (\mathcal{J}_n\cap\mathcal{K}_n\cap \mathcal{L}_n) & \text{ otherwise.}
    \end{cases}
\end{align*}
Since $h_n\xrightarrow{n\to\infty}0$, for all $K\in\N$ there is an $N_K\in\N$ with $N_K\geq K$, such that for all $n\geq N_K$ it holds that $h_n^{1/2}<K^{-1}$, $h_n^{1/2}<\Gamma_K K^{-\frac{m+2}{2m}}$, and $\eps_n^{1/2}h_n^{-1/2}<\Gamma_K K^{-\frac{1}{2}}$, which means $K\in\mathcal{J}_{N_K}\cap\mathcal{K}_{N_K}\cap\mathcal{L}_{N_K}$, hence for all $n\geq N_1$ it holds that $\mathcal{J}_n\cap\mathcal{K}_n\cap \mathcal{L}_n$ is nonempty and $\widetilde{K}(n)\xrightarrow{n\to\infty}+\infty$. Obviously then $\widetilde{K}(n)$ satisfies \eqref{sequence-conditions}.
\par Given one such choice of $\{\widetilde{K}(n)\}_{n\in\N}$, we can then fix a \emph{sequence of spectral cutoffs} \[\{K(n)\}_{n\in\N}\subseteq\N,\] by prescribing that $K(n)$ is the last index of the multiplicity block of the continuum eigenvalue $\lambda_{\widetilde{K}(n)}$ for every $n\in\N$. 
These conditions on the construction of the sequence of spectral cutoffs are sufficient to ensure that the continuum eigenvalues and eigenspaces are well-approximated by the discrete ones, as we will see in \Cref{spectral-approximation}.
Furthermore, defining the cutoffs in this way allows for consistent definition of low-pass signals based on the eigenvalue magnitude, which in turn allows for changes of bases in the respective eigenspaces if needed.
Admittedly, the conditions \eqref{sequence-conditions} are rather opaque, in the subsequent \Cref{discussion-growth-spectral-window}, we therefore examine, under an additional assumption on the manifold, what these conditions imply for the admissible growth rate of the spectral cutoffs.
\subsection{Shallow graph convolutional neural networks}
This subsection is dedicated to formulating convolution on graphs and manifolds, specifying the assumptions on the activation function, and defining shallow GCNNs as functionals on the space of graph signals. We will review how such functionals can either be nonlinearly parametrized by the parameters themselves or linearly by a measure in the form of a parameter distribution. The latter approach not only renders the parametrization linear, but is also flexible in the width of the neural network. This point of view is made compatible with finite realizable networks through the use of convex representer theorems guaranteeing the existence of minimizers supported on finitely many parameters, when considering a regularized empirical risk minimization cost involving finitely many training pairs. In particular, these apply to the discrete and continuum training problems we consider in Section \ref{gamma}, see \Cref{representer-theorem}.

To define a continuum problem on manifolds that reflects the limit behavior of GCNNs, we need to define an appropriate set of continuum signals. A natural choice in light of the spectral definition of the convolution, but which also works well with spatial discretizations, is to define a \textit{manifold signal} to be a function $u\in L^2(\M, \mu)$. A \textit{graph signal} assigns every node of the graph of size $n$ a value in $\R$, so can be represented as a vector in $\R^n$. If equipped with the Euclidean inner product, rescaled by the factor $1/n$ it is isometrically isomorphic to $L^2(\M_n, \mu_n)$. Therefore we consider a \textit{graph signal} to be a function $u\in L^2(\M_n, \mu_n)$. For readability, we will occasionally write $L^2$ instead of $L^2(\M,\mu)$, especially in subscripts of norms and inner products.

Generally, a trained single-layer (shallow) graph convolutional neural network is a functional \[f^\bn_\theta:L^2(\M_n, \mu_n)\to\R, \quad f^\bn_{\theta}(u)=\frac{1}{M}\sum\limits_{j=1}^M \langle a_j, \sigma(b_j*_nu+c_j)\rangle,\] where $\theta= \{(a_j, b_j, c_j)\}_{j=1}^M\in\Theta_n^M\subset (\R^n)^M$ are minimizers of a regularized Empirical Risk Minimization objective  with loss function $\ell:\R^2\to\R$ and some regularizer $R:\Theta_n^M\to\R$ and regularization parameter $\zeta\geq 0$:
\[\argmin_{\theta} \left[\frac{1}{l}\sum\limits_{k=1}^l \ell\left(f^\bn_{\theta}(u_k^{\bn}), y_k\right) + \zeta R(\theta)\right]. \]
Clearly, those $f^\bn_\theta$ are nonlinearly parametrized by their parameters $\theta$, however parametrizing by a measure over the parameter space instead leads to defining \emph{shallow graph convolutional neural networks} of potentially infinite width as functionals evaluating graph signals as follows
\begin{align}\label{def:nn-parametrization-measure-discrete}
    f^\bn_{\rho_n}: L^2(\M_n, \mu_n)\to\R, \quad f^\bn_{\rho_n}(u) = 
    \int_{\Theta_n} a^T\sigma(b*_n u + c)d\rho_n(a,b,c),
\end{align} 
and accordingly \emph{shallow manifold convolutional neural networks} as functionals evaluating manifold signals as follows
\begin{equation}\label{def:nn-parametrization-measure-continuum}
    f_\rho: L^2(\M, \mu)\to\R, \quad f_\rho(u) = \int_{\Theta} \langle a, \sigma(b* u + c)\rangle_{L^2}d\rho(a,b,c).
\end{equation}
 Note, that a measure $\rho=\sum_{j=1}^M w_j\delta_{\theta_j}$ supported on $M$ points of $\Theta$ (or $\Theta_n$ respectively) corresponds to a network of width $M$. In \Cref{representer-theorem} we give a representer theorem, stating under which conditions  the empirical risk minimization functional indeed admits a minimizer of finite width.

The remainder of this subsection will be dedicated to collecting the setup and definitions of the convolution operations $\ast_n$ and $\ast$, the parameter spaces $\Theta_n$ and $\Theta$, the activation function $\sigma$ and the regularizer $R$. In order to study convergence of such shallow GCNNs towards their manifold counterpart, it is useful to employ tools from functional analysis. To this end, in light of the Riesz-Markov duality, we want to interpret $f_\rho(u)$ as a dual pairing between the neural response map $\psi(u, \cdot)$ and the parametrizing measure $\rho$, and aim to apply results on eigenvalue and eigenfunction approximation from \cite{GarGerHeiSle20}, which will guide the setup. Some further elaboration and discussion on construction choices will appear in \Cref{sec:main-results-discussion}.

\paragraph*{Convolution operations.} Since in general there is no group structure on graphs or manifolds, motivated by the Convolution Theorem for Euclidean domains, the convolution operations on graphs and manifolds are defined as multiplication in the respective spectral domain. 
Let $\{\phi^\bn_k\}_{k=1}^n$ be an $L^2(\M_n, \mu_n)$-orthonormal basis of eigenvectors corresponding to the ordered graph Laplacian eigenvalues $\lambda^\bn_1\leq \ldots \leq \lambda^\bn_n$.
Accordingly, denote by $\{\phi_k\}_k$ a choice of  $L^2(\M, \mu)$-orthonormal eigenfunctions corresponding to the Laplace-Beltrami eigenvalues $\lambda_1\leq \lambda_2\leq\ldots$. Then, the convolution operations on graphs and manifolds are given as the bilinear operations 
\begin{align*}
&\ast_n: L^2(\M_n, \mu_n)\times L^2(\M_n, \mu_n)\to L^2(\M_n, \mu_n),
\\ &\quad(u, v)\mapsto u*_n v := \sum\limits_{k=1}^n \langle u, \phi^\bn_k\rangle_{L^2(\M_n)} \langle v, \phi^\bn_k\rangle_{L^2(\M_n)} \phi^\bn_k,
\end{align*}
and \begin{align*}
&\ast: L^2(\M, \mu)\times L^2(\M, \mu)\to L^2(\M, \mu),\\
&\quad (u,v)\mapsto \,u*v\,:=\sum\limits_{k=1}^\infty \langle u, \phi_k\rangle_{L^2(\M)} \langle v, \phi_k\rangle_{L^2(\M)}\,\phi_k.
\end{align*}

\begin{rem}\label{compactness-convolution-operator}
    It can be seen easily that for all $u\in L^2(\M, \mu)$ the convolution operator \[(\cdot* u)\in L(L^2(\M, \mu),L^2(\M, \mu))\] is compact:
    Let $\{v_n\}_n\subset L^2(\M, \mu)$ be a bounded sequence with $\norm{v_n}_{L^2}\leq M<\infty$, which therefore admits a weakly convergent subsequence $\{v_{n_l}\}_l$ with limit $v\in L^2(\M, \mu)$. 
    Denote by $\{\varphi_k\}_{k\in\N}$ an orthonormal basis of $L^2(\M, \mu)$. Then Parseval's identity and the definition of the convolution operation yield
    \begin{align*}
        0&\leq \norm{v_{n_l}*u-v*u}_{L^2}^2= \sum\limits_{k=1}^\infty \langle (v_{n_l}-v)*u, \varphi_k\rangle_{L^2}^2
        \\
        &= \sum\limits_{k=1}^\infty \big\langle\sum_{j=1}^\infty \langle v_{n_l}-v, \varphi_j\rangle_{L^2}\langle u, \varphi_j\rangle_{L^2}\varphi_j, \varphi_k\big\rangle_{L^2}^2
        = \sum\limits_{k=1}^\infty \langle v_{n_l}-v, \varphi_k\rangle_{L^2}^2\langle u, \varphi_k\rangle_{L^2}^2.
    \end{align*}
    Now, due to the weak convergence of  $\{v_{n_l}\}_l$ the sequence $\{v_{n_l}-v\}_l$ is bounded, so there is a $C> 0$ such that \begin{align*}
        |\langle v_{n_l}-v, \varphi_k\rangle_{L^2}^2\langle u, \varphi_k\rangle_{L^2}^2|\leq C^2 \langle u, \varphi_k\rangle_{L^2}^2
    \end{align*} and furthermore it holds that
    \[\sum\limits_{k=1}^\infty C^2\langle u, \varphi_k\rangle_{L^2}^2= C^2 \sum\limits_{k=1}^\infty \langle u, \varphi_k\rangle_{L^2}^2 = C^2\norm{u}_{L^2}^2,\]
    so $g_k:=C^2 \langle u, \varphi_k\rangle_{L^2}^2$ is a dominating summable sequence of $\{\langle v_{n_l}-v, \varphi_k\rangle_{L^2}^2\langle u, \varphi_k\rangle_{L^2}^2\}_k$ and hence using the weak convergence of $\{v_{n_l}\}_l$ we can conclude  from the dominated convergence theorem that
    \begin{align*}
    0&\leq \lim_{l\to\infty} \norm{v_{n_l}*u-v*u}_{L^2}^2= \lim_{l\to\infty}\sum\limits_{k=1}^\infty \langle v_{n_l}-v, \varphi_k\rangle_{L^2}^2\langle u, \varphi_k\rangle_{L^2}^2 = \sum\limits_{k=1}^\infty \lim_{l\to\infty} \langle v_{n_l}-v, \varphi_k\rangle_{L^2}^2\langle u, \varphi_k\rangle_{L^2}^2 =0,
    \end{align*}
    thus the image of every bounded sequence has a strongly convergent subsequence.
\end{rem}
\begin{rem}\label{weak-strong-convergence-convolution}
The compactness of the convolution operator, established in the previous \Cref{compactness-convolution-operator}, immediately yields for all $u\in L^2(\M, \mu)$ that if $b_k\xrightharpoonup{k\to\infty}b$ in $L^2(\M, \mu)$, then it is $$\norm{b_k*u -b*u}_{L^2} = \bigg\|\sum\limits_{j=1}^\infty \langle b_k-b, \phi_j\rangle_{L^2} \langle u, \phi_j\rangle_{L^2}\,\phi_j\bigg\|_{L^2}\xrightarrow{j\to\infty} 0.$$
In other words, for every signal $u\in L^2(\M, \mu)$ the mapping $(\cdot* u)\in L(L^2(\M, \mu),L^2(\M, \mu))$ is weak-strong continuous.
\end{rem}

\paragraph*{The activation function.} Let the activation function $\sigma:\R\to\R$ be a nonlinear function that is Lipschitz-continuous with Lipschitz constant $L_\sigma$.
We understand its application $\sigma(v)$ to a signal $v\in L^2(\M, \mu)$ as a map \(\sigma(v): \M\to\R,\: x\mapsto \sigma(v(x)).\) Accordingly its application to a $\ell^2$-sequence $\{v_i\}_i\in \ell^2$ is component-wise, i.e. $\sigma(\{v_i\}_i) = \{\sigma(v_i)\}_i$. Lipschitz-continuity with Lipschitz constant $L_\sigma$ of the activation function is equivalent to Lipschitz-continuity with Lipschitz constant $L_\sigma$ of the operator $\sigma\circ: L^2(\M, \mu)\to L^2(\M, \mu), v\mapsto \sigma(v)$, i.e. \begin{equation}\label{lipschitz-property-sigma}\norm{\sigma(v)-\sigma(\tilde{v})}_{L^2(\M, \mu)}\leq L_\sigma\norm{v-\tilde{v}}_{L^2(\M, \mu)} \text{ for all } v, \tilde{v}\in L^2(\M, \mu).\end{equation} 
Note, that this also yields a constant $C_\sigma >0$, such that for all $u\in L^2(\M,\mu)$ it is
\begin{equation}\label{linear-growth-sigma}\norm{\sigma(u)}_{L^2} \leq \norm{\sigma(u)-\sigma(0)}_{L^2}+\norm{\sigma(0)}_{L^2}\leq L_\sigma\norm{u}_{L^2} + \norm{\sigma(0)}_{L^2} \leq C_\sigma(\norm{u}_{L^2}+1).\end{equation}

Here $\Theta_n\subset \big(L^2(\M_n, \mu_n)\big)^3$ and $\Theta\subset \big(L^2(\M, \mu)\big)^3$ denote the respective parameter spaces and $\rho_n$ and $\rho$ are measures on $\Theta_n$ and $\Theta$, respectively.
We call the maps $\psi_n: L^2(\M_n, \mu_n) \times \Theta_n\to\R$, and $\psi: L^2(\M, \mu) \times \Theta\to\R$ arising here, given as \begin{equation}\label{def-discrete-neural-response-maps}
    \psi_n(u,\theta)=\psi_n(u,(a,b,c)):=\langle a, \sigma(b*_n u+c)\rangle_{L^2(\M_n, \mu_n)},
\end{equation} and
\begin{equation}\label{def-continuum-neural-response-maps}
    \psi(u,\theta)=\psi(u,(a,b,c)):=\langle a, \sigma(b* u+c)\rangle_{L^2(\M, \mu)}
\end{equation} in the following \emph{neural response maps}.

\paragraph*{Parameter spaces.}\label{parametrization}
We now define the discrete parameter spaces $\Theta_n$ and their continuum counterpart $\Theta$. As will be explicated in \Cref{regularity}, we will require $\Theta$, endowed with its topology, to be compact in $(L^2(\M, \mu))^3$ with the strong product topology. The discrete parameter spaces inherit this restriction, but are additionally subject to a spectral restriction motivated by the spectral convergence results in the subsequent \Cref{spectral-approximation}. 

For now, we restrict ourselves to introduce only the ingredients needed for their definition.
Following Proposition 3.2 in \citep{Caselli:2024} we can write the $H^\alpha$ semi-norm also in the following spectral manner: 
$$[v]^2_{H^\alpha(\M, \mu)}:= \sum\limits_{k=1}^\infty \lambda_k^\alpha \langle v, \phi_k\rangle^2_{L^2(\M, \mu)},$$
and $H^\alpha(\M,\mu)$ forms a Hilbert space with the inner product $\langle\cdot, \cdot\rangle_{H^\alpha(\M, \mu)}$ given as 
\begin{equation}\label{continuum-h-alpha-topology}
    \langle u, v\rangle_{H^\alpha(\M, \mu)}:=\sum\limits_{k=1}^\infty \left(1+\sqrt{\lambda_k}\right)^{2\alpha}\langle u, \phi_k\rangle_{L^2}\langle v, \phi_k\rangle_{L^2}.
\end{equation} 
Of course there is not naturally a notion of $H^\alpha$ regularity on the finite dimensional parameter spaces, since all norms are equivalent. Still, we want to impose a decay of the Fourier coefficients with increasing frequency, analogously to the continuum case and in a way which is convenient to pass to the limit. So we employ this spectral definition of the semi-norm to define the following norm on $\R^n$:
\begin{equation}\label{discrete-h-alpha-topology}
    \norm{v}_{\alpha, n}: \R^n\to\R_{\geq 0}, \norm{v}_{\alpha, n}^2:=\sum\limits_{k=1}^{n} \left(1+ \sqrt{\lambda^\bn_k}\right)^{2\alpha}\langle v, \phi^\bn_k\rangle^2,
\end{equation}
induced by the inner product $\langle\cdot, \cdot\rangle_{H^\alpha(\M_n, \mu_n)}$ on $L^2(\M_n, \mu_n)$ by 
\[\langle u, v\rangle_{H^\alpha(\M_n, \mu_n)}:=\sum\limits_{k=1}^n \left(1+\sqrt{\lambda^\bn_k}\right)^{2\alpha}\langle u, \phi^\bn_k\rangle_{L^2}\langle v, \phi^\bn_k\rangle_{L^2}.\]
Furthermore, we impose the spectral cutoff-condition $\langle v, \phi^\bn_k\rangle=0$ for all $k>K(n)$, i.e. we restrict ourselves to the subspaces spanned by any first $K(n)$ eigenvectors $\{\varphi^\bn_k\}_{k=1}^{K(n)}$, denoted as $$V_{n}:= \Span\big\{\varphi^\bn_1, \ldots \varphi^\bn_{K(n)}\big\}.$$
Then, we define 
\begin{equation}\label{def-continuum-parameter-space}
    \Theta:= B_{H^\alpha}\times B_{{L^2}} \times B_{H^\alpha},
\end{equation} and \begin{equation}\label{def-discrete-parameter-space}
    \Theta_n:= \left(B_{\norm{\cdot}_{\alpha, n}}\cap V_n\right)\times \left(B_{\norm{\cdot}_{0,n}}\cap V_n\right) \times \left(B_{\norm{\cdot}_{\alpha, n}}\cap V_n\right),
\end{equation}
where $B$ denotes the closed unit ball in the indicated space with the given norm, and $\norm{\cdot}_{0,n}$ just denotes the Euclidean norm. 
We note that it is only essential for the definition that the respective parameter sets are convex, closed, and bounded in the corresponding norms, the choice of closed unit balls is only for simplicity of the exposition.
\paragraph*{Shallow convolutional neural networks as dual pairings.}
In the following we denote by $C(X,\tau)$ the space of functions mapping from $X$ to $\R$ that are continuous with respect to the topology $\tau$ on $X$. We will denote the weak topology by $w$. Naturally $C(X,\tau)$ is equipped with the supremum norm $\norm{\cdot}_\infty$.
\begin{lemma} \label{weak-continuity-neural-response-map}
     For all $u\in L^2(\M, \mu)$ it holds that $\psi(u, \cdot)\in C(\Theta, w)$ and furthermore $\psi(u, \cdot)$ is continuous w.r.t. the strong topology of $(L^2(\M, \mu))^3$.
\end{lemma}
\begin{proof}
    Let $\{\theta\}_k \subset \Theta$ be a sequence, that converges weakly in $\Theta$ to some $\theta\in\Theta$, i.e. if we write $\theta_k=(a_k, b_k, c_k)$ for $k\in\N$ and $\theta=(a,b,c)$, then $a_k,c_k\xrightharpoonup{k\to\infty}a,c$ in $H^\alpha(\M, \mu)$ and $b_k\xrightharpoonup{k\to\infty}b$ in $L^2(\M, \mu)$. Then the Rellich-Kondrachov for fractional Sobolev spaces \cite[Prop.~4.3.4]{Tay23} yields strong $L^2$-convergence up to a subsequence in the first and last parameter $a_{k_j},c_{k_j}\xrightarrow{j\to\infty}a,c$ in $L^2(\M, \mu)$.

        Let $u\in L^2(\M, \mu)$ arbitrary but fixed. Then the sequence $\{\psi(u, \theta_k\}\}_k\subset \R$ is bounded, namely there is a $R\geq 0$, such that 
\begin{align*}|\psi(u, \theta_k)| &=  |\langle a_k, \sigma(b_k * u + c_k)\rangle_{L^2}| \leq \norm{a_k}_{L^2}C_\sigma (\norm{b_k*u}_{L^2} + \norm{c_k}_{L^2}) \leq R,
\end{align*}
where the last inequality follows from $\{\theta_k\}_k=\{(a_k,b_k,c_k)\}_k$ being bounded in $L^2$ due to the definition of $\Theta$ as the unit ball of the product space (also just because of the weak convergence in $L^2$).

Furthermore $\psi(u, \cdot)$ is continuous w.r.t. the strong topology of $(L^2(\M, \mu)^3$, since for every sequence $\{\tilde{\theta}_k\}_k=\{(\tilde{a}_k,\tilde{b}_k, \tilde{c}_k)\}_k\subset L^2(\M, \mu)^3$ converging w.r.t. strong topology on $L^2$ to some $\tilde{\theta}=(\tilde{a}, \tilde{b}, \tilde{c})\in L^2(\M, \mu)^3$, we have
\begin{align*}
|\psi(u, \tilde{\theta}_k) - \psi(u, \tilde{\theta})| &=  |\langle \tilde{a}_k, \sigma(\tilde{b}_k * u + \tilde{c}_k)\rangle_{L^2} - \langle \tilde{a}, \sigma(\tilde{b} * u + \tilde{c})\rangle_{L^2}| \\&\leq \norm{\tilde{a}_k-\tilde{a}}_{L^2}C_\sigma (\norm{\tilde{b}_k}_{L^2}\norm{u}_{L^2} + \norm{\tilde{c}_k}_{L^2}+1) \\&\qquad+ \norm{\tilde{a}}_{L^2}L_\sigma (\norm{(\tilde{b}-\tilde{b}_k)*u}_{L^2} + \norm{\tilde{c}_k - \tilde{c}}_{L^2})\xrightarrow{k\to\infty} 0.
\end{align*}

Now from the boundedness of $\{\psi(u, \theta_k\}\}_k\subset \R$ and Rellich-Kondrachov we can conclude that $\psi(u, \cdot)\in C(\Theta, w)$, i.e. $|\psi(u, \theta_k)-\psi(u, \theta)|\xrightarrow{k\to\infty}0$:
Assume $|\psi(u, \theta_k)-\psi(u, \theta)|\not\to0$ as $k\to\infty$. Thus there is an $\eps>0$ and a subsequence $\{\theta_{k_l}\}_l\subset \{\theta_k\}_k$ s.t. $|\psi(u, \theta_{k_l})-\psi(u, \theta)|\geq \eps$ for all $l\in\N$.
Now $\{a_{k_l}\}_l, \{c_{k_l}\}_l\subset B_{H^\alpha}$, so by Rellich-Kondrachov there is another subsequence $\{l_o\}_o\subset \N$, s.t. $a_{k_{l_o}}, c_{k_{l_o}}\xrightarrow{o\to\infty} a,c$ strongly in $L^2$. But then, from the continuity of $\psi(u, \cdot)$ w.r.t. the strong topology on $L^2$, we know that there is an $O\in\N$ s.t. for all $o\geq O$ it holds that $|\psi(u, \theta_{k_{l_o}})-\psi(u, \theta)|< \eps$, which contradicts the assumption and proves the claim.
\end{proof}

Now, because $\Theta$ is compact in the weak topology, it holds that $C_0(\Theta, w) = C(\Theta, w)$, and therefore due to Riesz-Markov theorem  the dual of the space of weakly continuous functions on $\Theta$ can be identified with the space of signed Radon measures with finite total variation on $\Theta$ with respect to the Borel $\sigma$-algebra generated by its weak topology, i.e. $$\big(C(\Theta, w), \norm{\cdot}_\infty\big)^* \cong \big(\cM(\Theta, w), \norm{\cdot}_{TV}\big). $$
In the following, we will simply write $\cM(X)$ instead of $\cM(X,w)$, since throughout this paper spaces of signed Radon measures on subsets $X\subset H^\alpha(\M, \mu)\times L^2(\M, \mu)\times H^\alpha(\M, \mu)$ are always understood with respect to the Borel $\sigma$-algebra generated by the weak topology induced on $X$. Furthermore, the spaces of signed Radon measures with finite total variation are naturally equipped with their respective total variation norm $\norm{\cdot}_{TV}$, which will therefore often be left implicit for brevity throughout this paper. In the finite dimensional case the weak topology and all norm topologies coincide, so the dual of the space of continuous functions on $\Theta_n$ can be identified with the space of signed Radon measures with finite total variation on $\Theta_n$, i.e. $$\left(C(\Theta_n)\right)^* = \left(C(\Theta_n, w)\right)^* \cong \cM(\Theta_n).$$

Therefore we can now understand neural networks $f_{\rho_n}:L^2(\M_n, \mu_n)\to \R$ and $f_\rho:L^2(\M, \mu)\to \R$ as dual pairings, i.e. 
\begin{align*}
    f^\bn_{\rho_n}(v) &= \hspace{-0.25em}
    \int_{\Theta_n} \langle a, \sigma(b*_n v + c)\rangle_{L^2(\M_n)}d\rho_n(a,b,c) = \hspace{-0.25em}\int_{\Theta_n} \psi_n(v,\theta) d\rho_n(\theta) = \langle\rho_n, \psi_n(v, \cdot)\rangle_{\cM(\Theta_n), C(\Theta_n)}
\end{align*}and
\begin{align*}
    f_{\rho}(u) &= 
    \int_{\Theta} \langle a, \sigma(b* u + c)\rangle_{L^2(\M)}d\rho_n(a,b,c) = \int_{\Theta} \psi(u,\theta) d\rho(\theta) = \langle\rho, \psi(u, \cdot)\rangle_{\cM(\Theta),C(\Theta,w)}.
\end{align*}
This duality enables us to investigate limiting behaviour of neural networks in terms of weak limits of neuron evaluation functionals or weak-* limits of parameter distributions.

\subsection{Consistency of graph signals by comparison through transport maps}\label{signal-consistency}
The graph and continuum GCNNs defined in the previous subsections are defined on signals on different domains, so to formulate convergence results relating them we need a notion of consistency of their inputs.

First we prove the following lemma stating the existence of semi-discrete optimal transportation cells. Such a result of existence of optimal maps can be found in \cite[Thm.~5.5]{ChaDepJuu08} for the case of Euclidean spaces, but we were not able to find an explicit proof elsewhere for the more general manifold case, so we include a proof for completeness. It will be a consequence of the Marriage lemma for measures \cite[Lem.~8.1]{Burago:2015}, which, for a given radius, gives a necessary and sufficient condition for the existence of a partition of $\M$ into transportation cells with $L^\infty$-cost at most that radius. Naturally such a partition is in general far from being unique.
\begin{lemma}\label{existence-ot-maps-cells-semidiscrete}
    For all $n\in\N$, the minimum in $\eps_n$ is attained, i.e. it holds that 
    \[\eps_n=d_\infty(\mu, \mu_n) = \min\limits_{T:T_\# \mu=\mu_n} \esssup_{x\in\M} d_\M(x, T(x)).\]
\end{lemma}
\begin{proof}
    By the definition of $\eps_n$ we can find a monotonically decreasing sequence $r_k\downarrow\eps_n$ as $k\to\infty$, and transport maps $T_k$, $k\in\N$, with $(T_k)_\#\mu=\mu_n$ so that $\esssup_{x\in\M} d(x,T_k(x))<r_k$. Then, there also exist measurable maps $\widetilde{T}_k:\M\to\M_n$ such that $\widetilde{T}_k=T_k$ $\mu$-almost everywhere and $d(x,\widetilde{T}_k(x))<r_k$ for all $x\in\M$. The maps $\widetilde{T}_k$ induce partitions $\{U^{r_k}_i\}_{i=1}^n$ of $\M$ into measurable subsets, such that
        \[U^{r_k}_i:= \widetilde{T}_k^{-1}(\{x_i\}), \qquad \mu(U^{r_k}_i)=\mu(T_k^{-1}(\{x_i\}))=\mu_n(\{x_i\})=\frac{1}{n}, \quad \text{ and } \quad U^{r_k}_i\subseteq B_\M(x_i, r_k).
        \] Thus $\mu_n$ is an $r_k$-approximation of $\mu$ (in the sense of \cite[Def.~1.1]{Burago:2015}) for all $k\in\N$. Note that this definition applies here, because we could also instead use $\tilde{\mu}_n:=\Vol(\M)\mu_n$ and $\Vol_\M=\Vol(\M)\mu$. We want to show that then $\mu_n$ is also an $\eps_n$-approximation of $\mu$ by using the Marriage lemma for measures \cite[Lem.~8.1]{Burago:2015}.
        Clearly, it holds for every $I\subset\{1,\ldots, n\}$ that 
        \begin{align}
            \mu(\bigcup_{i\in I} B_\M(x_i, r_k))\geq \mu(\bigcup_{i\in I} U_i^{r_k}) = \sum_{i\in I} \mu(U_i^{r_k}) = \frac{|I|}{n}.
        \end{align}
        Next, we notice that for all $x\in\M$ it holds that 
        \begin{align}\label{increasing-seq-balls-limit}
        \bigcap_{k\in\N} B_\M(x,r_k)=\overline{B_\M(x,\eps_n)}\quad 
        \text{ and }\quad 
        \bigcup_{i\in I} B_\M(x_i, r_k)\supseteq \bigcup_{i\in I} B_\M(x_i, r_{k+1})\;\text{ for all } k\in\N.
        \end{align}
        Now, \eqref{increasing-seq-balls-limit} implies that 
        \begin{align}
            \bigcup_{i\in I} \overline{B_\M(x,\eps_n)} = \bigcup_{i\in I} \bigcap_{k\in\N} B_\M(x_i, r_k) = \bigcap_{k\in\N} \bigcup_{i\in I} B_\M(x_i, r_k),
        \end{align} and thus since for any $x\in\M$, $r\geq 0$ sets of the form $\partial B_\M(x,r)$ have $\Vol_\M$-measure zero, since the volume of a geodesic ball depends continuously on its radius, we get
        \begin{align}
            \mu\left(\bigcup_{i\in I} B_\M(x_i, \eps_n)\right) = \mu\left(\bigcup_{i\in I} \overline{B_\M(x_i, \eps_n)}\right) = \mu\left(\bigcap_{k\in\N} \bigcup_{i\in I} B_\M(x_i, r_k)\right) = \lim_{k\to\infty} \mu\left(\bigcup_{i\in I} B_\M(x_i, r_k)\right)\geq \frac{|I|}{n}.
        \end{align}
        This is precisely the condition in the Marriage lemma for measures \cite[Lem.~8.1]{Burago:2015} and thus $\mu_n$ is indeed an $\eps_n$-approximation of $\mu$. Therefore, there exists a partition $\{U^\bn_i\}_{i=1}^n$ of $\M$ into measurable subsets, such that 
        \[U^\bn_i\subset B_\M(x_i, \eps_n) \quad \text{and} \quad \mu(U^\bn_i)=\frac{1}{n} \text{ for all }i\in\{1,\ldots, n\}.\]
        Then $T_n:\M\to\M_n$, where $T_n(x)=x_i$ for $x\in U^\bn_i$ are transportation maps, i.e. it holds that $(T_n)_\#\mu = \mu_n$ and by construction they attain the minimum in the definition of $\eps_n$, and thus are optimal.
\end{proof}
In the following, for $n\in\N$ let $T_n:\M\to\M_n$ be an optimal transportation map with transportation cells $U^\bn_i:=T_n^{-1}(\{x_i\})$, $i\in\{1, \ldots n\}$ that satisfy
    \begin{align}\label{optimal-transportation-cells}U^\bn_i\subset B_\M(x_i, \eps_n) \quad \text{and} \quad \mu(U^\bn_i)=\frac{1}{n} \text{ for all }i\in\{1,\ldots, n\}.\end{align}
\begin{defn}\label{consistent-signals}
    We say that a collection of graph signals $u_n\in L^2(\M_n, \mu_n)$, $n\in\N$ is \emph{consistent across resolutions} if they are discretizations of an underlying manifold signal, namely if there exists a $u\in L^2(\M, \mu)$ s.t. $d_{TL^2}((\mu, u), (\mu_n, u_n))\to 0$.
\end{defn}
For a definition of $TL^2$-convergence see Section 3 in \cite{GarciaTrillos:2015}. 
 \begin{rem}\label{tl2-convergence-as-convergence-by-transportation-maps}
     From the previous \cref{existence-ot-maps-cells-semidiscrete} we know that there exist a sequence $\{T_n\}_{n\in\N}$ of optimal transportation maps such that $(T_n)_\#\mu = \mu_n$. Since the transportation maps are optimal with respect to the $L^\infty$-cost and $\mu$ is finite on $\M$, they are stagnating \cite[p.~11]{GarciaTrillos:2015}, since \[\int_\M d_\M(x, T_n(x))\dd\mu(x)\leq \mu(\M)\esssup_{x\in\M} d_\M(x, T_n(x)) = 1\cdot d_\infty(\mu, \mu_n) =\eps_n \xrightarrow{n\to\infty} 0.\]
Furthermore, by construction the empirical measure $\mu_n$ converge weakly to $\mu$, therefore, by \citep[Prop.~3.12]{GarciaTrillos:2015}, the previous \Cref{consistent-signals} is equivalent to the requirement that there is a $u\in L^2(\M, \mu)$, such that $\norm{u-u_n\circ T_n}_{L^2}\xrightarrow{n\to\infty} 0$.
 \end{rem}

We will see later that graph signals arising from a spatial discretization by averaging over transport cells or a spectral discretization of a manifold signal can be examples of such consistency across resolutions, see \Cref{tl2-compatibility-spatial-spectral-discretization-maps}. 

For every $n\in\N$ fix an optimal transportation map $T_n:\M\to\M_n$ and the induced partition of optimal transportation cells $U^\bn_i:=T_n^{-1}(\{x_i\})$, $\mu(U^\bn_i)=1/n$, which exist due to \Cref{existence-ot-maps-cells-semidiscrete}.
Then, define the \emph{spatial discretization maps} $\mP_n: L^2(\M, \mu)\to L^2(\M_n, \mu_n)$, $n\in \N$ which discretize by averaging over the partition $\{U^\bn_i\}_{i=1}^n$ (see also \citep[p. 689]{Burago:2015} and \citep[p. 840]{GarGerHeiSle20})
\begin{equation}\label{def-spatial-discretization}
    (\mP_n v)(x_i):= n\cdot \int_{U^\bn_i} v(x)d\mu(x) \quad \text{ for } v\in L^2(\M, \mu),
\end{equation} and the adjoint extension maps $\mP^*_n: L^2(\M_n, \mu_n)\to L^2(\M, \mu)$, $n\in\N$, that extend discrete functions to piecewise constant step functions over the tesselation
\begin{equation}\label{def-spatial-extension}(\mP^*_n v)(x):= v(T_n(x)) = \sum_{i=1}^n v(x_i)\mathds{1}_{U^\bn_i}(x) \quad\text{ for } v\in L^2(\M_n, \mu_n).\end{equation}

\begin{rem}
    With the considerations of \Cref{tl2-convergence-as-convergence-by-transportation-maps} and the definition of $\mP_n^*$ we see that the notions of $TL^2$-convergence and convergence under spatial extension $\mP_n^*$  for any collection of graph signals $u_n\in L^2(\M_n, \mu_n)$ towards a manifold signal $u\in L^2(\M, \mu)$ coincide in our setting, i.e. 
    \begin{equation}\label{tl2-spatial-extension-convergence}
        d_{TL^2}((\mu, u), (\mu_n, u_n))\xrightarrow{n\to\infty}0 \qquad \Longleftrightarrow\qquad \norm{u-\mP_n^* u_n}_{L^2}\xrightarrow{n\to\infty}0.
    \end{equation}
\end{rem}

For completeness, we also briefly introduce the interpolation operators used in the spectral approximation results of \cite{GarGerHeiSle20} that we rely on later.
For $n\in \N$, let ${I_n: L^2(\M_n, \mu_n)\to \text{Lip}(\M)}$ be the \emph{interpolation operator} from \citep[p. 840]{GarGerHeiSle20}, defined as  \begin{equation}\label{def-interpolation-op}
    I_n=\Lambda_{h-2\varepsilon_{n}} \mP_n^*, 
\end{equation} where $\Lambda_r:L^2(\M, \mu)\to \text{Lip}(\M)$, $r>0$ are smoothing operators, constructed by pointwisely normalizing the integral kernel operators $\Lambda^0_r:L^2(\M, \Vol) \to \text{Lip}(\M))$, that convolve an $L^2(\M, \Vol)$-function with a radial kernel $k_r(x,y)$ in the geodesic distance with compact support on the geodesic ball $B_\M(x,r)$, see \cite[p.~862]{GarGerHeiSle20}. These interpolation operators find application in comparing eigenfunctions of the graph Laplacian and the Laplace-Beltrami operator.

\vskip 1em 
\def\arraystretch{1.25}
\begin{small}
\begin{longtable}{  r |  p{0.8\textwidth}  }
  \caption{Notation}\\ 
  $d,m$ & Dimensions of the ambient data space $\R^d$ and of the manifold $\M$. \\
  \hline
  $\M$ & Compact connected smooth manifold without boundary of dimension $m$ embedded in $\R^d$, whose Riemannian metric $g$ is the one inherited from the embedding.\\
  \hline
  $d_{\M}(x,y)$ & The geodesic distance between points $x,y\in \M$.\\
  \hline
  $\mu$  & Probability measure supported on $\M$ describing the data distribution. \\ 
  \hline
  $\Delta$  & The Laplace-Beltrami operator on $\M$. \\ 
  \hline
  $\lambda_k$ & Eigenvalues of the Laplace-Beltrami operator, arranged such that \newline$\lambda_1\leq \lambda_2 \leq \ldots \uparrow +\infty$.\\
  \hline
  $E(\lambda)$ & Eigenspace of the Laplace-Beltrami operator associated to the eigenvalue $\lambda$.\\
  \hline
  $\gamma_{k}$ & The spectral gap of the eigenvalue $\lambda_k$, i.e. the minimal distance of the eigenvalue $\lambda_k$ to its distinct neighboring eigenvalues.\\
  \hline
  $\varphi_k$ & Arbitrary choice of an orthonormal basis of  $L^2(\M, \mu)$ consisting of eigenfunctions of the Laplace-Beltrami operator in order of increasing eigenvalues.\\
  \hline
  $N(\lambda)$ & Eigenvalue counting function (see \eqref{eigencount}).\\
  \hline
  $\M_n$ & Point cloud $\M_n=\{x_1, \ldots x_n\} \subset \R^d$ drawn from i.i.d samples from $\mu$, the nodes of the geometric proximity Graph $G_n$.\\
  \hline
  $n$ & Resolution or size of the graph $G_n$, i.e. number of nodes $\M_n:= \{x_1, \ldots x_n\}$ of $G_n$.\\
  \hline
  $\mu_n$ & Empirical measure of the sample $\M_n$.\\
  \hline
  $\eps_n$ & Upper bound on the geodesic distance $d_{\M}$ between any point in $\M$ and a point of $\M_n$ (see \eqref{infty-wasserstein}). \\
  \hline
  $h_n$ & Radius, in terms of geodesic distance, of interaction between points so that the edge between them receives a nonzero weight (see \eqref{graph-weights-eta}).\\
  \hline
  $\Delta_n$ & Unnormalized graph Laplacian corresponding to the $n$-th geometric graph $G_n$, as defined in \eqref{def-unnormalized-graph-laplacian}.\\
  \hline
  \rule{0pt}{2.6ex}
  $\lambda_k^\bn$ & Eigenvalues of the graph Laplacian, arranged such that $\lambda_1^\bn\leq \ldots \leq \lambda_n^\bn$.\\
  \hline
  $K(n)$ & A frequency cutoff compatible with multiplicity blocks for the $n$-th geometric proximity graph, such that up to this index the spectrum of the $n$-th graph Laplacian is informative about that of the Laplace-Beltrami operator (see \eqref{sequence-conditions}).\\
  \hline
  \rule{0pt}{2.6ex}
  $\varphi_k^\bn$, $\varphi_k$ & An arbitrary choice of eigenfunctions $\varphi_k^\bn$ of the graph Laplacian corresponding to the eigenvalues $\lambda^{\bn}_k$ such that they form an orthonormal basis of $L^2(\M_n, \mu_n)$ and an arbitrary choice of orthonormal eigenfunctions $\varphi_k$ of the Laplace-Beltrami operator corresponding to the eigenvalues $\lambda_k$, such that they form an orthonormal basis of $L^2(\M, \mu)$.\\
  \hline
  \rule{0pt}{2.6ex}
  $\phi_k^\bn$, $\phi_k^\cn$ & A fixed choice of orthonormal eigenfunctions of the graph Laplacian corresponding to the eigenvalues $\lambda^{\bn}_k$ and a fixed resolution-dependent choice of orthonormal eigenfunctions of the Laplace-Beltrami operator corresponding to the eigenvalues $\lambda_k$, such that the first $K(n)$ of them align well with the eigenfunctions $\phi^\bn_k$ of the graph Laplacian for the graph of size $n$.\\
  \hline
  $\delta(\eps, h, \lambda)$ & Auxiliary function for bounding approximation errors between discrete and continuum eigenvalues and eigenfunctions, as defined in \eqref{delta-definition}. \\ 
  \hline
  $\delta^\phi_n$ & Upper bound on the norm of the difference of interpolated discrete eigenfunctions and continuous ones, as defined in \eqref{uniform-error-stepfct-extension}.\\
  \hline
  $\Theta_n$, $\Theta$ & Discrete and continuum parameter space, see \eqref{def-discrete-parameter-space}, and \eqref{def-continuum-parameter-space}. \\ 
  \hline
  $\alpha$ & The order of fractional Sobolev regularity imposed on the parameter spaces. \\ 
  \hline
  $\ast_n$, $\ast$ & Convolution operations, defined by multiplication in the Fourier domain of the graph Laplacian corresponding to the graph of size $n$ or the Laplace-Beltrami operator, respectively.\\
  \hline
  $\sigma$, $L_\sigma$, $C_\sigma$ & The activation function, its Lipschitz constant, and a linear growth constant for it (see \eqref{linear-growth-sigma}).\\
  \hline
  $\psi_n$, $\psi$ & Neural response maps as defined in \eqref{def-discrete-neural-response-maps} and \eqref{def-continuum-neural-response-maps}.\\
  \hline
  \rule{0pt}{2.6ex}
  $f^\bn_{\rho}$ & Infinitely wide shallow graph convolutional neural network parametrized by the measure $\rho$ on the discrete parameter space $\Theta_n$, as defined in \eqref{def:nn-parametrization-measure-discrete}.\\
  \hline
  $f_{\rho}$ & Infinitely wide shallow manifold convolutional neural network parametrized by the measure $\rho$ on the continuum parameter space $\Theta$, as defined in \eqref{def:nn-parametrization-measure-continuum}.\\
  \hline
  $\mP_n, \mP_n^*$ & Spatial discretization and extension maps defined in \eqref{def-spatial-discretization} and \eqref{def-spatial-extension}.\\
  \hline
  $I_n$ & Interpolation operator from \citep[p. 840]{GarGerHeiSle20}, see \eqref{def-interpolation-op}.\\
  \hline
  $\mR_n$, $C_\mR$ & Linear discretization operators,  uniformly bounded in $n$ by $C_\mR$, and compatible with $TL^2$-convergence (see \Cref{consistent-discretization-operators}).\\
  \hline
  $\mS_{n,\alpha}, \mS_{n,\alpha}^*$ &  Spectral discretization and extension maps that respect the $H^\alpha$-structure, defined in \eqref{def-spectral-discretization} and \eqref{def-adjoint-spectral-extension}.\\
  \hline
  $\mQ_{n,\alpha}, \mQ_{n,\alpha}^*$ &  The parameter discretization operation and its adjoint, constructed from the spatial and spectral discretization and extension operators \eqref{def-parameter-discretization} and \eqref{def-parameter-extension}.\\
  \hline
  $\ell$ & Loss function $\R \times \R \to \R$, generally assumed to be continuous (in the first argument).\\
  \hline
  $J_{\alpha}$ & The empirical risk minimization functional corresponding to the training of the shallow manifold convolutional neural network, as defined in \eqref{def-J-alpha}.\\
  \hline
  \rule{0pt}{2.2ex}  $\tilde{J}_{n,\alpha}$ & The empirical risk minimization functional corresponding to the training of the shallow graph convolutional neural network, where the graph is of resolution $n$ and $\alpha$ is the degree of regularity, as defined in \eqref{def-tilde-J-alpha-n}.\\
  \hline
  $J_{n,\alpha}$ & The lifted empirical risk minimization functional corresponding to the training of the graph of resolution $n$, as defined in \eqref{def-J-alpha-n}.\\
  \hline
  $C(X,\tau)$ & The space of functions mapping from $X$ to $\R$ which are continuous with respect to the topology $\tau$ on $X$ equipped with the supremum norm $\norm{\cdot}_{\infty}$. \\
  \hline
  $\cM(X)$ & The space of finite signed Radon measures on $X$ equipped with the total variation norm $\norm{\cdot}_{TV}$. \\
  \hline
  $\cP(X)$ & The space of Radon probability measures on $X$. \\
  \hline
  $w, w^*$ & Weak, weak-* topology. \\
  \hline
  $C$ & General notation for no-further specified constants, whose value may vary across occurrences, even within one formula.\\
  \hline
  $\omega_m$ & The volume of the unit ball in $\R^m$.\\
  \hline
  $B_V$,  $B_V(v, r)$ & The unit ball in the normed vector space $(V, \norm{\cdot}_V)$; a closed ball of radius $r$ in the normed vector space $(V, \norm{\cdot}_V)$ around $v\in V$.\\
  \hline
  $B_{\norm{\cdot}_{V_w}}$ & The closed unit ball in the normed vector space $(V, \norm{\cdot}_{V_w})$, used if in particular a different norm than the standard norm on $V$ is used.\\
  \hline
  $B_{d_V}(v, r)$ & A closed ball of radius $r$ in the metric space $(V,d_V)$ around $v\in V$.
\label{tab:TableOfNotation}
\end{longtable}
\end{small}

\section{Main results and discussion}\label{sec:main-results-discussion}
In this section, we summarize our main theoretical results, provide a simple but illustrative numerical experiment reflecting the effects of the choices made in the framework including the specific form of the discrete parameter spaces, and discuss the meaning of the spectral cutoff choices in terms of spectral gaps of the underlying manifolds.

Let $\{(u_k, y_k)\}_{k=1}^l\in (L^2(\M)\times\R)^l$ be continuum training data. From these, using a sequence of uniformly bounded discretization operators $\mR_n: L^2(\M, \mu)\to L^2(\M_n, \mu_n)$ such that for every manifold signal the sequence of its discretizations is consistent across resolutions in the sense of \Cref{consistent-signals} (see also \cref{{consistent-discretization-operators}}), we obtain discrete training data $\{(\mR_nu_k, y_k)\}_{k=1}^l\in (L^2(\M_n, \mu_n)\times \R)^l$.

Then, given a non-negative real-valued continuous loss function $\ell:\R\times\R\to\R^+$, and a regularization parameter $\zeta\geq 0$ for a resolution $n\in\N$ and discrete training data as defined above, we define the (reguralized) ERM functionals $\tilde{J}_{n,\alpha}:\cM(\Theta_n)\to\R^+\cup\{+\infty\}$ for the discrete representation of \eqref{def:nn-parametrization-measure-discrete} as follows
\begin{equation}\label{def-tilde-J-alpha-n}\tilde{J}_{n,\alpha}(\tilde{\nu}) := \frac{1}{l}\sum\limits_{k=1}^l \ell(f^\bn_{\tilde{\nu}}(\mR_n u_k), y_k) + \zeta\norm{\tilde{\nu}}_{TV}\quad\text{ for }\tilde{\nu}\in\cM(\Theta_n).\end{equation}
Analogously, we define the (regularized) continuum ERM functional $J_\alpha:\cM(\Theta)\to \R^+\cup\{+\infty\}$ for the representation defined in \eqref{def:nn-parametrization-measure-continuum} as:
\begin{equation}\label{def-J-alpha} J_\alpha(\nu) := \frac{1}{l}\sum\limits_{k=1}^l \ell(f_\nu(u_k), y_k) + \zeta \norm{\nu}_{TV}\quad\text{ for }\nu\in\cM(\Theta).\end{equation}
Our main results concern the convergence of minimizers of the ERM functionals and of realizations of graph networks with parametrizations that are consistent across resolutions in the following sense. To relate parametrizations across the discrete and continuum settings, we introduce adjoint parameter projection maps \[\mQ_{n,\alpha}:\Theta \to\Theta_n\quad\text{ and }\quad\mQ_{n,\alpha}^*:\Theta_n\to\Theta,\] whose precise construction and properties are developed in the subsequent \Cref{discrete-continuum-erm-convergence}. These maps allow us to represent parametrizations on a discrete parameter space on the common continuum parameter space and thereby formulate their convergence across resolutions. For the purposes of the following summary, it suffices to know that these maps preserve the respective parameter norms on the admissible spectral subspaces associated with the spectral windows $K(n)$, in particular they satisfy $\mQ_{n,\alpha}\mQ_{n,\alpha}^*(\Theta_n)=\Theta_n$.
Our main results can then be summarized as follows:
\begin{thm}[Summary of main results]\label{summary-main-results}\leavevmode
Let $M>0$ arbitrary but fixed. 
\begin{enumerate}
    \item \textbf{Uniform convergence of the neural response for consistent data.}
    For the sequence of maps $\{\psi_n(\mR_n u, \cdot)\circ \mQ_{n,\alpha} - \psi(u, \cdot)\}_n \subset C(\Theta, w)$ it holds that 
    \begin{equation*}
    \norm{\psi_n(\mR_n u, \cdot)\circ \mQ_{n,\alpha} - \psi(u, \cdot)}_{\infty} \xrightarrow{n\to\infty}0 \qquad \text{ for all } u\in L^2(M, \mu).
    \end{equation*}
    \item \textbf{Uniform convergence of networks on compact sets up to subsequences for consistent inputs and consistent parametrization.}
    Let $\tilde{\rho}_n\in B_{\cM(\Theta_n)}(0,M)$, where $n\in\N$, such that there is a $\rho\in \cM(\Theta)$, such that $(\mQ_{n, \alpha}^*)_\#\tilde{\rho}_n=:\rho_n \xrightharpoonup[n\to\infty]{*} \rho$.     Then for every compact $K\subset L^2(\M, \mu)$ the sequence of networks \[\{f_{\tilde{\rho}_n}\circ \mR_n  - f_\rho\}_n \subset C(K,\norm{\cdot}_{L^2})\] is relatively compact in the topology induced by the uniform norm.
    \item \textbf{Convergence of trained networks.}
    Let $\{\tilde{\rho}_n\}_n\subset B_{\cM(\Theta)}(0,M)$ be a sequence consisting of respective minimizers $\tilde{\rho}_n$ of $\tilde{J}_{n,\alpha}$. Then there exists a subsequence $\{\tilde{\rho}_{n_k}\}_k$ such that $\{(\mQ_{n,\alpha}^*)_\#\tilde{\rho}_{n_k}\}_k$ converges weak-* to a $\rho\in\cM(\Theta)$ that is a minimizer of $J_\alpha$. Furthermore, it holds for all $u\in L^2(\M, \mu)$ that
    \begin{equation*}
    f_{\tilde{\rho}_{n_k}}(\mR_{n_k} u)\xrightarrow{k\to\infty} f_\rho(u).
    \end{equation*}
\end{enumerate}
\end{thm}

The proof of all of these results is completed in \Cref{gamma}, after developing the necessary tools and properties in  \Cref{spectral-approximation} and \Cref{discrete-continuum-erm-convergence}.

\subsection{A synthetic classification example}\label{subsec:toy-example}
Consider $\M=\S^1\times\S^1=\mathbb{T}^2$, which is not only a compact smooth Riemannian manifold, but also a compact abelian group with the group operation being either componentwise multiplication in the embedding under the identification $\R^2\times\R^2\cong \mathbb{C}\times\mathbb{C}$, or componentwise addition in the angular parametrization.
The dual group of $\mathbb{T}^2$ is $\widehat{\mathbb{T}^2}=\operatorname{Hom}(\mathbb{T}^2, \S^1)$, which can be identified with $\Z^2$, since $\operatorname{Hom}(\S^1, \S^1)\cong \Z$.

Now Bochner's theorem states that for every continuous positive definite function $u:\M\to\mathbb{C}$ there exists a unique finite positive measure $\nu$ on the dual group $\widehat{\mathbb{T}^2}$, such that \[u(x)=\int_{\widehat{\mathbb{T}^2}}\xi(x)\dd\nu(\xi).\] Since $\widehat{\mathbb{T}^2}\cong\Z^2$ this means we can write the dual group as a collection $\widehat{\mathbb{T}^2}=\{\xi_{j,k}\}_{(j,k)\in\Z^2}$ and therefore Bochner's theorem yields the following representation of continuous positive definite functions $u:\M\to\mathbb{C}$:  
\[u(x)=\sum_{(j,k)\in\Z^2} \hat{u}(j,k)\xi_{j,k}(x), \qquad \text{where}\qquad \R\ni\hat{u}(j,k)\geq 0 \text{ for all }j,k\in\Z.
\]
Thus a continuous function $u:\M\to\R$ is positive definite iff the Fourier coefficients satisfy \[\widehat{u}(-j,-k)=\widehat{u}(j,k) \geq 0 \text{ for all }j,k\in\Z.
\]
Due to the compactness of $\M=\mathbb{T}^2$ it holds that $C(\mathbb{\M})\subset L^2(\M)$. Furthermore, for this flat torus it also holds that the gaps between distinct eigenvalues are bounded away from zero, since all eigenvalues of the flat torus lie in $\Z$, which can easily be seen as follows.
For product manifolds $\M_1\times\M_2$ it holds that eigenvalues $\lambda$ of the Laplacian on $\M_1\times \M_2$ can be written in the form $\lambda=\lambda^{(1)}+\lambda^{(2)}$, where $\lambda^{(o)}$ is an eigenvalue of the Laplacian on $\M_o$, $o=1,2$ \cite[pp.~26-27]{Chavel:1984}, and all eigenvalues of the circle $\S^1$ are non-negative integers \cite[p.~35]{Chavel:1984}. 
Therefore, we can choose $\widetilde{K}(n)=\min\{h_n^{-\frac{1}{2}}, n\}$.
In particular, using the fact that $x=(x_1,x_2)\in\S^1\times\S^1$ can be represented in normalized angular coordinates $\theta=(\theta_1,\theta_2)\in [0,1)^2$ through $x_i=(\cos(2\pi\theta_i), \sin(2\pi\theta_i))$, the $\xi_{j,k}$ take the form \[\xi_{j,k}(x)=\exp(2\pi i (j\theta_1+k\theta_2)).\] 

\paragraph*{Setup of the transferability experiments.}
We generate 50 continuous positive definite functions using Bochner's theorem, as well as 50 continuum functions with at least one negative Fourier coefficient and label them accordingly with $1$ or $-1$, respectively. Furthermore, 50 discontinuous continuum step functions are generated with label $-1$.
The neighborhood radius for the graphs is set to be \[h_n:= \frac{\log(n)^{3/8}}{n^{1/4}}.\] The nodes of the graph are $n$ random points $(x,y)\in[0,1)^2$. We assign edge weights \[w_{i,j}:=\frac{1}{nh_n^2}\1_{[0, 1]}\left(\frac{d((x_i,y_i),(x_j,y_j))}{h_n}\right),\] where $d$ denotes the Euclidean distance between the corresponding embedded points, where the embedding is the natural embedding of $\S^1\times\S^1$ into $\R^2\times\R^2$, and two nodes are connected by an edge with weight $w_{i,j}$ if $w_{i,j}>0$. From this the Graph Laplacians are constructed. An eigensolver gives the ordered eigenvalues with corresponding eigenvectors up to a given mode $K$.
The discrete training and test data are obtained as point evaluations of the respective continuum data at the sampled graph nodes. 
We then train shallow GCNNs of width 50 on the discrete training signals and their labels, where we vary different hyperparameters. 
\par First, we fix the resolution $n$ to $2^{13}=8192$ and generate seven graphs of this size. For varying spectral windows $K$, we train networks with parameters in the product of the respective balls restricted to the span of the first $K$ eigenfunctions of the corresponding graph Laplacian. Then, for each spectral window $K$, we compute and plot in \Cref{fig:transferability-K}, per signal, the difference between the highest and lowest output of the corresponding seven trained networks.
\par Next, we vary the graph resolution $n$ and create seven graphs for each resolution and train the corresponding networks, where now the admissible spectral window is as in \eqref{sequence-conditions}. We would like to point out that for all resolutions $n$ in this experiment, namely $2^{13}, 2^{14}, \ldots, 2^{17}$ the admissible spectral window $K(n)$ is $5$ according to \eqref{sequence-conditions}. Again, for each resolution $n$, we compute and plot in \Cref{fig:transferability-n}, per signal, the difference between the highest and lowest output of the corresponding seven trained networks.

In both plots (\Cref{fig:transferability}), we distinguish between signals from the training and test data sets.
\begin{figure}
    \centering
    \begin{subfigure}{0.48\linewidth}
        \centering
        \includegraphics[width=\linewidth]{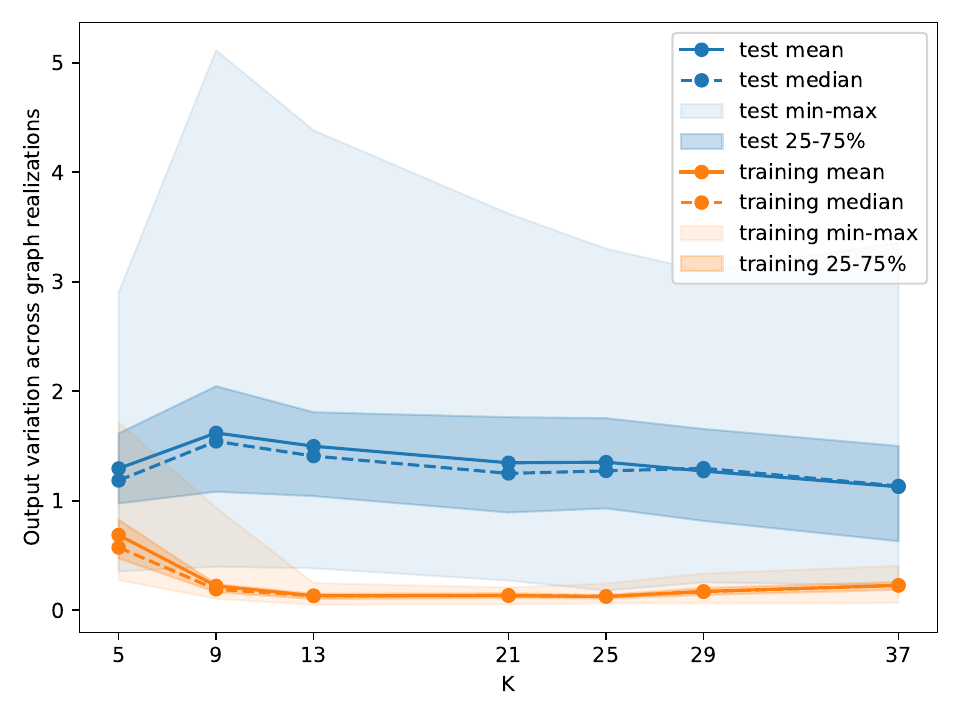}
        \caption{Fixed $n = 2^{13}$, varying $K$}
        \label{fig:transferability-K}
    \end{subfigure}
    \hfill
    \begin{subfigure}{0.48\linewidth}
        \centering
        \includegraphics[width=\linewidth]{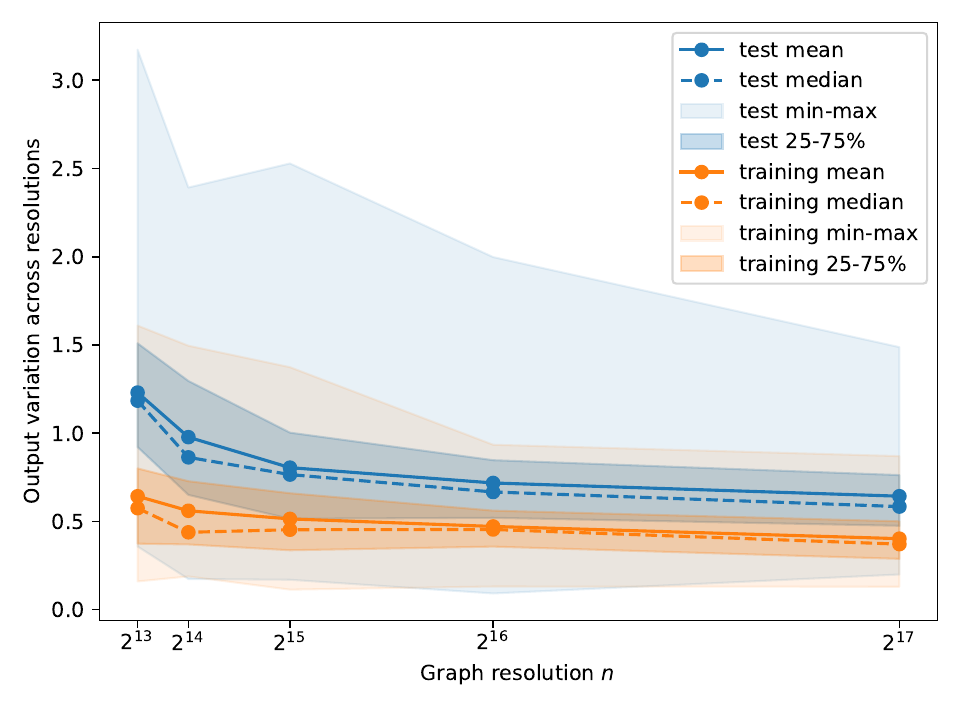}
        \caption{Varying $n$, $K=K(n)$ according to \eqref{sequence-conditions}}
        \label{fig:transferability-n}
    \end{subfigure}
    \caption{\textbf{Transferability.} In (a), we depict the effect of increasing the spectral window $K$ above the theoretical cutoff $K(n)=5$ prescribed by condition \eqref{sequence-conditions}. In this case the training accuracy naturally increases, leading to a decrease in the training output variation. On the test data, however, we observe that the transferability initially decreases and does not significantly increase further with even larger values. With increasing $K$ the test accuracy still increases slightly, which might explain, why the output variation at some point decreases, and thus the transferability increases again with increasing $K$. 
    In (b), we observe that with increasing resolution the output variation decreases both on training and test data, showing that the trained networks become increasingly more transferable.
    }
    \label{fig:transferability}
\end{figure}

\subsection{Growth of the spectral window}\label{discussion-growth-spectral-window}
\par The conditions \eqref{sequence-conditions} are rather opaque, and we would like to illustrate how slowly $\{\widetilde{K}(n)\}_n$ grows. To this end, assume on the manifold $\M$ that there is a $\beta\geq 0$, such that $\gamma_{k}^{-1}=O(k^\beta)$ and define 
 \[\beta^*:= \inf\big\{\beta \geq 0\mid \gamma_{k}^{-1}=O(k^\beta)\big\}.\]
Note that this is assumption is not necessarily verifiable in practice,  but is introduced only to make the admissible asymptotic growth of the spectral window more transparent.
    Now let $\{\widetilde{K}(n)\}_{n\in\N}$ be a sequence of $\N\ni\widetilde{K}(n)\leq n$, such that 
    \begin{equation} \label{sequence-conditions-polynomial} 0<\eps_n\ll h_n\ll \widetilde{K}(n)^{-\max\{\beta^* + \frac{m+2}{2m}, 1\}} \ll 1, \quad\text{ and }\quad \eps_n h_n^{-1} \ll 
    \widetilde{K}(n)^{-\left(\beta^*+\frac{1}{2}\right)}.
    \end{equation}
    Clearly $\{\widetilde{K}(n)\}_n$ then satisfies \eqref{sequence-conditions}. \vskip -1em
\begin{rem}\label{big-exponent-existence}
    Suppose $\beta^*+\frac{m+2}{2m}\geq 1$, i.e. $\beta^*\geq \frac{2m-m-2}{2m} = \frac{m-2}{2m}$. Hence $2m\beta^*\geq m-2$ and thus
    \[\frac{m}{2m\beta^*+m+2} \geq \frac{m}{4m\beta^*+4} \geq \frac{1}{4\beta^*+4}.\]
    Therefore, for $0<h_n<1$ it holds that \[h_n^{\frac{m}{2m\beta^* + m+2}} \leq h_n^{\frac{1}{4\beta^*+4}}.\]\vskip -1em 
    Then for a sequence of neighborhood sizes $\{h_n\}_n\subset(0,1)$, satisfying \eqref{assumption3}, one could choose $\widetilde{K}(n)$ as \vspace{-1em}\[\widetilde{K}(n):= \min\{\lfloor h_n^{-\frac{m}{2m\beta^* + m+2}} \rfloor, n\}\]
    since \[h_n^{-\frac{m}{2m\beta^* + m+2}} \geq h_n^{-\frac{1}{4\beta^*+4}} \xrightarrow{n\to\infty} +\infty,\qquad \widetilde{K}(n)^{-\left(\beta^*+\frac{m+2}{2m}\right)}  \geq h_n^{1/2}> h_n,\]
    and
    \[
    \widetilde{K}(n)^{\left(\beta^*+\frac{m+2}{2m}\right)} h_n \leq (h_n^{-\frac{m}{2m\beta^*+m+2}} )^{\left(\beta^*+\frac{m+2}{2m}\right)} h_n \leq h_n^{1/2}  \xrightarrow{n\to\infty}0.
    \]
    Thus it holds indeed that $h_n \ll \widetilde{K}(n)^{-\left(\beta^*+\frac{m+2}{2m}\right)}.$
     If one takes $h_n=\sqrt{\eps_n}$, one obtains \[\eps_n h_n^{-1}= h_n \ll \widetilde{K}(n)^{-\left(\beta^*+\frac{m+2}{2m}\right)} \leq \widetilde{K}(n)^{-\left(\beta^*+\frac{1}{2}\right)},\] so $\{\widetilde{K}(n)\}_n$ satisfies \eqref{sequence-conditions}.
    Using \eqref{probabilistic-bounds-eps} \cite[Thm.~2]{GarGerHeiSle20} and the above, it follows that with high probability for $m\geq 3$,\vspace{-0.5em} \[\widetilde{K}(n)\geq C \left(\frac{n}{\log(n)}\right)^{\frac{1}{2(2m\beta^*+m+2)}} \geq C \left(\frac{n}{\log(n)}\right)^{\frac{1}{8(m\beta^*+1)}},
    \] and for $m=2$,
    \[\widetilde{K}(n)\geq C\left(\frac{n}{\log(n)^{3/2}}\right)^{\frac{1}{8(\beta^*+1)}}.\]
\end{rem}

\begin{rem}\label{small-exponent-existence}
    Now, suppose $\beta^*+\frac{m+2}{2m}< 1$.  Then for a sequence of neighborhood sizes $\{h_n\}_n$, satisfying \eqref{assumption3}, one could choose $\widetilde{K}(n)$ as \[\widetilde{K}(n):= \min\{\lfloor h_n^{-1/2}\rfloor,n\} \xrightarrow{n\to\infty} +\infty,\] which also yields  
    $\widetilde{K}(n)^{-\left(\beta^*+\frac{m+2}{2m}\right)} \,>\,\widetilde{K}(n)^{-1}\geq h_n^{1/2}> h_n,$ and using that $0\leq \,\beta^*+\frac{m+2}{2m}< 1$ further we have
    \[\widetilde{K}(n)^{\left(\beta^*+\frac{m+2}{2m}\right)}h_n \leq \widetilde{K}(n)h_n \leq h_n^{-\frac{1}{2}}h_n= h_n^{\frac{1}{2}}\xrightarrow{n\to\infty}0.\]
    With the same argument as in the other case, taking $h_n=\sqrt{\eps_n}$ yields
    \[\eps_n h_n^{-1}= h_n \ll \widetilde{K}(n)^{-\left(\beta^*+\frac{m+2}{2m}\right)} \leq \widetilde{K}(n)^{-\left(\beta^*+\frac{1}{2}\right)},\] 
    so again $\{\widetilde{K}(n)\}_n$ satisfies \eqref{sequence-conditions}. In this case it holds that $2m\beta^*<m-2$, and thus using \eqref{probabilistic-bounds-eps} \cite[Thm.~2]{GarGerHeiSle20} it follows that with high probability for $m\geq 3$, \[\widetilde{K}(n)\geq C \left(\frac{n}{\log(n)}\right)^{\frac{1}{2(2m\beta^*+m+2)}} \geq C \left(\frac{n}{\log(n)}\right)^{\frac{1}{4m}}.\]
\end{rem} 

So with \Cref{big-exponent-existence} and \Cref{small-exponent-existence} we get an idea of how slowly the admissible spectral window grows with the graph resolution in relation to the dimension of the manifold and the rate of decay of distinct spectral gaps.
\subsection{On the construction of the parameter spaces}\label{regularity}
To establish a formal notion of training convergence for shallow GCNNs, we adopt an infinite-width formulation where the parametrization is represented by measures over the parameter space. To understand such measures as elements of the dual of the space of neural response maps, we require these maps to be continuous with respect to a topology that renders their domain, the parameter space, compact, as was explicated in \Cref{parametrization}. 
While this is easily satisfied in the finite-dimensional setting by restricting the parameter spaces to the products of closed balls, the situation is more involved in the infinite-dimensional setting, since clearly compactness in the norm topology is not implied by boundedness.

One option would be to try and recover compactness using Banach-Alaoglu by considering products of closed unit balls in $L^2(\M, \mu)$ equipped with their respective weak topology.
This would then require the neural response maps $\psi(u, \cdot)$ for fixed signals $u\in L^2(\M, \mu)$ to be sequentially continuous with respect to the weak topology of $\big(L^2(\M, \mu)\big)^3$.
We already observed in \Cref{weak-strong-convergence-convolution} that the convolution term behaves nicely with the weak $L^2$-topology in the sense that for every signal $u\in L^2(\M, \mu)$ the mapping $(\cdot * u): L^2(\M, \mu)\to L^2(\M,\mu), \;b\mapsto b* u$ is weak-strong continuous. Regardless, an infinite-dimensional parameter space equipped with the product of weak $L^2$-topologies is not compatible with the neural network setting, since given a weakly converging sequence $(a_k,b_k,c_k)$, in the ensuing sequence of responses $\langle a_k, \sigma(b_k * u + c_k)\rangle_{L^2}$ the argument $b_k * u + c_k$ would only converge weakly, and the pointwise nonlinearity $\sigma$ is never weakly continuous. Moreover, even if such an application of $\sigma$ would be sequentially weakly continuous (which would force it to be linear), we would still encounter an $L^2$-inner product of two weakly converging sequences.

However, we show in \Cref{weak-continuity-neural-response-map} that the neural response maps are always continuous in the parameter variables with respect to the norm topology on $L^2$. Thus, a way to achieve the competing requirements of continuity and compactness in the continuum setting for the first and last parameter  $a$ and $c$ in $(a,b,c)\in\Theta$ is to let them be in the unit ball of a space that is compactly embedded in $L^2(\M, \mu)$, and equip the parameter space with the product of the respective weak topologies. Using this strategy, a natural choice of $\Theta$ is therefore the product of the balls $B_{H^\alpha}\times B_{{L^2}}\times B_{H^\alpha}$, each equipped with its respective weak topology, and with $\alpha>0$. 

Naturally, such a choice induces a restriction on the allowed parameters and learned functions already at the discrete level. To give some interpretation of such restrictions, one can consider the parameters $(a_n, b_n, c_n)\in (\R^n)^3$ for a single discrete neuron, and notice that in each of them the $i$-th component represents the contribution of node $x_i$ to the parameter’s associated signal operation. Since $b_n$ acts through a convolution, it is most naturally seen as a spectral filter on the graph signal $u$, which can heighten or dampen certain frequency components of the input signal $u$. Then $c_n$ can be seen as a pointwise shift of the values of the filtered signal, emphasizing or suppressing the contributions of the filtered input signal per node. After that, the activation function is applied node-wise and outputs how much every node of the graph is activated using the spectrally filtered and pointwise shifted signal. Finally $a_n$, corresponding to a linear outer layer, acts as a spatial readout that aggregates how much of every node should contribute to the final scalar output. With this interpretation, we see that the $H^\alpha$-regularity is imposed on the two parameters $a_n$ and $c_n$ directly involved in pointwise operations.
Moreover, we point to the condition $\eps_n \ll h_n$ which is imposed even in a quantitative form in \eqref{sequence-conditions}, and also implies that sharp differences of values on very close vertices cannot be relevant for any operation built using the edge weights of \eqref{graph-weights-eta}.

Ultimately, this $H^\alpha$-regularity assumption was introduced on our part as a way to enforce compactness assumptions. Different subsets which induce compactness in the weak-* topology may also be used, and it is thinkable that one could reconstruct the results under such assumptions. Moving away from weak-* topology of measure and compact or locally compact parameter spaces (which are needed for the standard Riesz-Markov duality) is potentially much more challenging, however. This would mean stepping out of the relatively standard paradigm of integral representations corresponding to infinite-width limits onto a potentially very different functional-analytic framework. We are not aware of any other feasible options for this.

\paragraph*{An overfitting example.}
In the regime of the synthetic classification example of \Cref{subsec:toy-example}, we construct training data in which the continuous positive definite functions consist of functions with only Fourier coefficients corresponding to at most the first three eigenspaces, or equivalently, the first nine Fourier modes, where only one of the two symmetric pairs of Fourier coefficients associated to the third eigenspace is nonzero, i.e. 
\begin{gather*}
    u(x)=\sum\limits_{\substack{(j,k)\in\{0,\pm 1\}^2\\|j|+|k|\leq 1}} \hat{u}(j,k)\xi_{j,k}(x) + \hat{u}(\bar{j},\bar{k})\big(\xi_{\bar{j},\bar{k}}(x) + \xi_{-\bar{j},-\bar{k}}(x)\big), \\ \text{where }
    \;\hat{u}(j,k)\geq 0 \text{ for all }(j,k)\in\{0, \pm 1\}^2, |j|+|k|\leq 1,\;\text{ and } \; \hat{u}(\bar{j},\bar{k})>0,\;(\bar{j}, \bar{k})\in\{(1,1), (1,-1)\}.
\end{gather*}
 The non-positive-definite functions are constructed to have non-negative Fourier coefficients associated with the first two eigenspaces and a symmetric pair of negative Fourier coefficients associated with the third eigenspace, i.e. \begin{gather*}
v(x)=\sum\limits_{\substack{(j,k)\in\{0,\pm 1\}^2\\|j|+|k|\leq 1}} \hat{v}(j,k)\xi_{j,k}(x) + \hat{v}(\bar{j},\bar{k})\big(\xi_{\bar{j},\bar{k}}(x) + \xi_{-\bar{j},-\bar{k}}(x)\big), \\ \text{where }
\;\hat{v}(j,k)\geq 0 \text{ for all }(j,k)\in\{0, \pm 1\}^2, |j|+|k|\leq 1,\;\text{ and } \; \hat{v}(\bar{j},\bar{k})<0,\;(\bar{j}, \bar{k})\in\{(1,1), (1,-1)\}.
\end{gather*} 
Furthermore, we construct test data where the continuous positive-definite functions are generated like the positive-definite training functions, while the non-positive-definite functions have at least one symmetric pair of negative Fourier coefficients associated to the first two eigenspaces and only non-negative Fourier coefficients associated to the third eigenspace. 

We then train models on seven generated graphs of resolution $n$, where we fix the admissible spectral window to the first nine modes, using different choices of norms $\norm{\cdot}_{\alpha, n}$, $\alpha\geq 0$, and radii $R>0$ for the parameter space, i.e.
\[\Theta_n = \left(B_{\norm{\cdot}_{\alpha, n}}(0, R)\times B_{\norm{\cdot}_{0, n}}(0, R) \times B_{\norm{\cdot}_{\alpha, n}}(0, R)\right)\cap \left(\Span\{\varphi^\bn_i\}_{i=1}^9\right)^3.\]

For each choice of radius and resolution we plot in \Cref{fig:parameter-space-overunderfitting} the evolution of the mean output over the seven trained networks per test signal across the different choices of norms. Hereby we group the test signals by their prescribed labels, with continuous positive-definite test signals shown in blue and the remaining signals in orange. The lines indicate the median over the respective label group, while the transparent bands indicate the 25th to 75th percentile range of the mean network outputs.

As expected, the models trained on these fabricated training data generally fail to classify the non-positive-definite test functions correctly and tend to overfit, exhibiting behaviour which indicates that the symmetric pair of Fourier coefficients associated to the third eigenspace is decisive.

\begin{figure}[H]
    \centering
    \begin{subfigure}{0.32\linewidth}
        \centering
        \includegraphics[width=\linewidth]{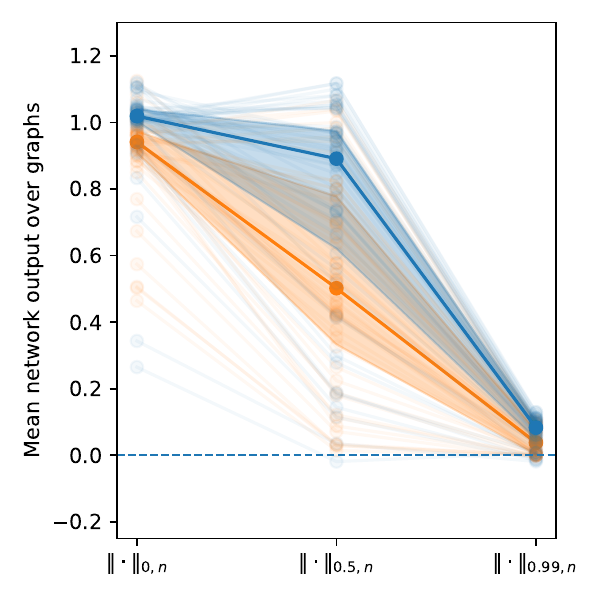}
        \caption{n = 8192, R = 8}
        \label{fig:r8}
    \end{subfigure}
    \hfill
    \begin{subfigure}{0.32\linewidth}
        \centering
        \includegraphics[width=\linewidth]{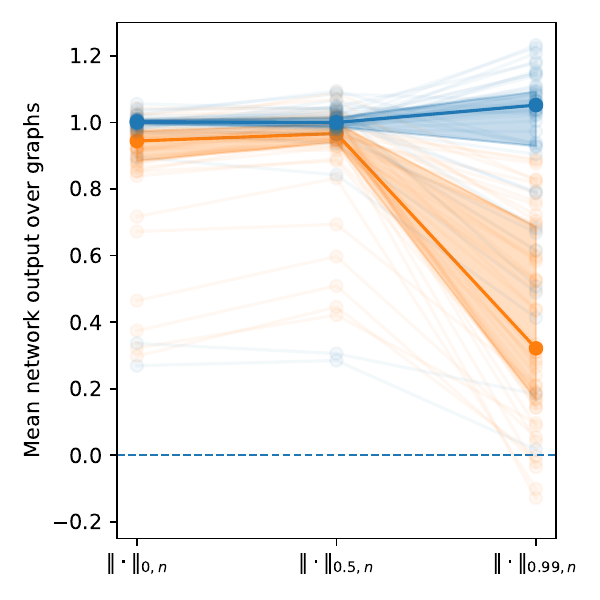}
        \caption{n = 8192, R = 40}
        \label{fig:r10}
    \end{subfigure}
    \hfill
    \begin{subfigure}{0.32\linewidth}
        \centering
        \includegraphics[width=\linewidth]{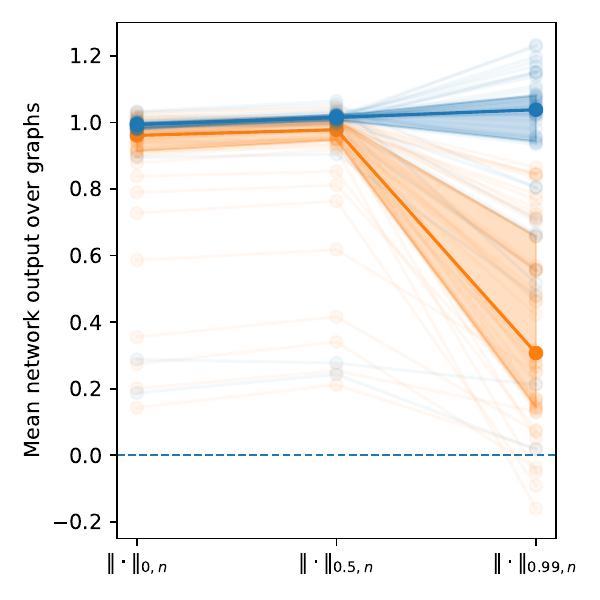}
        \caption{n = 16384, R = 40}
        \label{fig:r2}
    \end{subfigure}
    \caption{\textbf{Effect of the choice of parameter norm and ball radius on over- and underfitting.} 
    Comparison of $\norm{\cdot}_{0,n}$- and $\norm{\cdot}_{\alpha,n}$-balls for different radii $R$, Sobolev exponents $\alpha$, and resolutions $n$.
    We observe in all considered configurations of radii and resolutions (a) – (c) that the trained models with $\alpha=0$ exhibit overfitting behaviour, for both groups of test signals the models largely outputs values around $+1$. This behaviour appears to be the least pronounced with the smallest radius $R=8$ in (a), where the model outputs are overall more broadly distributed, but the two groups nevertheless remain strongly overlapping. Increasing the Sobolev regularity in (a) eventually leads to underfitting, also reflected by increased training loss and reduced test accuracy. At the larger radius $R=40$ in (b), after training with the weaker constraints where $\alpha=0$ or $\alpha = 0.5$ the manufactured overfitting persists, whereas the stronger constraint with $\alpha=0.99$ visibly reduces it. Panel (c) shows the same qualitative behavior at the higher resolution $n=16384$.}
    \label{fig:parameter-space-overunderfitting}
    
\end{figure}

\subsection{On the expressivity of continuum GCNNs}\label{expressivity}
Here we would like to point out that even in the continuum setting and without smoothness restrictions in the parameters, a single neuron cannot detect components of arbitrarily high frequency in its inputs. Thus, some of the limitations that may at first sight be attributed to our specific regularity assumptions on the parameters are already inherent in the structure of spectral GCNNs.
To see this, note that for all $v\in L^2(\M, \mu)$ it holds that its Fourier coefficients $\langle v, \varphi_k\rangle$ w.r.t. any orthonormal eigenbasis $\{\varphi_k\}_k\subset L^2(\M, \mu)$ of the Laplace-Beltrami operator converge to $0$ as $k\to\infty$, since \[\norm{v}_{L^2}^2=\sum_{k=1}^\infty \langle v, \varphi_k\rangle^2<\infty.\]
Hence, for all $\theta=(a,b,c)\in L^2(\M, \mu)^3$ \begin{align*}
|\psi(\varphi_k, \theta)-\psi(0,\theta)|&=|\psi(\varphi_k, \theta)-\langle a, \sigma(c)\rangle| = |\langle a, \sigma(b*\varphi_k+c) -\sigma(c)\rangle| \\&\leq \norm{a}_{L^2}L_\sigma \norm{b*\varphi_k +c-c}_{L^2} = \norm{a}_{L^2}L_\sigma \norm{b*\varphi_k}_{L^2}
\\&= \norm{a}_{L^2}L_\sigma |\langle b, \varphi_k\rangle| \xrightarrow{k\to\infty} 0.\end{align*}
Let $g:L^2(\M, \mu)^3 \to\R$ be given by $(a,b,c)=\theta\mapsto g(\theta)= L_\sigma \norm{a}_{L^2}\norm{b}_{L^2}$. Then \[|\psi(\varphi_k, \theta)-\langle a, \sigma(c)\rangle| \leq L_\sigma \norm{a}_{L^2}\norm{b*\varphi_k}_{L^2}\leq g(\theta).\]
Thus by the dominated convergence theorem for any subset $\mathcal{V}\subseteq L^2(\M, \mu)^3$ and $\nu\in \cM(\mathcal{V})$ s.t. $\int_\mathcal{V}g(\theta)d\nu(\theta)<\infty$, $\int_\mathcal{V} \psi(\varphi_k, \theta)d\nu(\theta) <\infty$ for all eigenfunctions $\varphi_k$, $k\in \N$, and $\int_\mathcal{V} \psi(0, \theta)d\nu(\theta) <\infty$ it is \begin{align*}\left|\int_\mathcal{V} \psi(\varphi_k, \theta)d\nu(\theta) - \int_\mathcal{V}\langle a, \sigma(c)\rangle d\nu(\theta)\right| &= \left|\int_\mathcal{V} \psi(\varphi_k, \theta) - \langle a, \sigma(c)\rangle d\nu(\theta)\right|\\&\leq \int_\mathcal{V} \left|\psi(\varphi_k, \theta) - \langle a, \sigma(c)\rangle \right| d\nu(\theta)\xrightarrow{k\to\infty}0.\end{align*}
Thus $$f_\nu(\varphi_k)\xrightarrow{k\to\infty}\int_{\mathcal{V}}\langle a, \sigma(c)\rangle d\nu(\theta) \,=:\, C_\nu.$$

Due to the Lipschitz property of $\sigma$ (see \eqref{lipschitz-property-sigma}) it holds for all $u\in L^2(\M, \mu)$, $\theta\in\Theta$ that
\begin{align*}
    \left|\psi\left(\sum_{k=K}^\infty \langle u, \varphi_k\rangle \varphi_k, 
\theta\right) - \psi(0, \theta)\right| &
    = |\langle a, \sigma(b*\sum_{k=K}^\infty \langle u, \varphi_k\rangle \varphi_k+c)-\sigma(c)\rangle_{L^2}|\\
    &\leq L_\sigma \norm{a}_{L^2}\norm{\sum_{k=K}^\infty \langle u, \varphi_k\rangle \langle b, \varphi_k\rangle \varphi_k}_{L^2} \\
    &\leq L_\sigma \norm{a}_{L^2}\sup_{k\geq K}|\langle b, \varphi_k\rangle|\norm{u}_{L^2}\xrightarrow{K\to\infty} 0.
\end{align*}
So again dominated convergence yields $|f_\nu\left(\sum_{k=K}^\infty \langle u, \varphi_k\rangle \varphi_k\right) - C_\nu| \xrightarrow{K\to\infty} 0,$ so for input signals $u$ that consist only of high frequency components, any convolutional network $f_\nu$ parametrized by some $\nu\in\cM(\mathcal{V})$ will output approximately the same value $C_\nu$ for such signals $u$.

\section{Spectral Approximations}\label{spectral-approximation}

In this section we collect the error bounds for approximation of the eigenvalues and eigenspaces required for the convergence results in \Cref{discrete-continuum-erm-convergence}.
Only the low-frequency part of the spectrum of the graph Laplacian approximates that of the Laplace-Beltrami well, and we will see that the spectral cutoff sequence $\{K(n)\}_n$ introduced in \Cref{point-cloud-graphs-manifold-ass}, constructed from a sequence $\{\widetilde{K}(n)\}_n$ satisfying \eqref{sequence-conditions} ensures that the assumptions from \cite[Thm.~4,5, and Cor.~1]{GarGerHeiSle20} are satisfied. In particular, $K(n)$ is indeed a reasonable choice for an upper bound of the informative spectral window, and the error bounds stated here are direct consequences of the results from \cite{GarGerHeiSle20}.

The situation for the approximation of eigenspaces is more subtle, due to possible high multiplicities, and the fact that multiplicity blocks of the graph Laplacian and the Laplace–Beltrami operator need not coincide.
Nevertheless, for a fixed resolution and a fixed choice of orthonormal graph eigenfunctions associated with a given Laplace-Beltrami multiplicity block, one can select an orthonormal basis of the corresponding continuum eigenspace such that the discrete eigenfunctions approximate this continuum-basis well, which is a consequence of \cite[Thm.~4]{Burago:2015} and \cite[Thm.~5]{GarGerHeiSle20}.

This approximation error depends on the spectral gap. In this work, we are interested in the discrete-to-continuum convergence, where there are no bandwidth restrictions on the manifold signal space, hence the spectral window must ultimately diverge. 
Thus, achieving uniform control of eigenspace approximation errors requires assumptions on the asymptotic behavior of spectral gaps of the Laplace–Beltrami operator. 
\begin{thm}[Error bounds for the relative eigenvalue and eigenspace approximation]\label{uniform-bounds-spectra}
    For $k\in\N$, denote by $\underline{k}$ and $\bar{k}= \underline{k}+m(\lambda_k)-1$ the first and last index of the multiplicity block of the Laplace-Beltrami eigenvalue $\lambda_k$, i.e. $\lambda_{\underline{k}-1}<\lambda_{\underline{k}} =\ldots \lambda_k =\ldots = \lambda_{\bar{k}} <\lambda_{\bar{k}+ 1} $. Let $N\in\N$ such that  $k\leq \bar{k}\leq K(l)$ for all $l\geq N$.
    \\Fix $n\geq N$ and denote by $\{\varphi^\bn_j\}_{j=\underline{k}}^{\bar{k}}$ a choice of orthonormal eigenfunctions of the $n$-th  graph Laplacian $\Delta_n$ corresponding to the eigenvalues 
    $\{\lambda_i^\bn\}_{i=\underline{k}}^{\bar{k}}$. \\
    Then it holds that
    \begin{equation}\label{eigenvalues-ratio-uniform-bound}
              |\lambda_k^\bn - \lambda_k|\leq C\lambda_k\delta(\eps_n, h_n, \lambda_k),
          \end{equation}
          where 
           \begin{equation}\label{delta-definition}\delta(\eps, h, \lambda):= \frac{\eps}{h} + \left(1+\sqrt{\lambda}\right)h + \left(K+\frac{1}{R^2}\right)h^2.\end{equation}
    Furthermore, there exists an orthonormal basis $\{\varphi^{\cn}_j\}_{j=\underline{k}}^{\bar{k}}$ of the Laplace-Beltrami eigenspace $E(\lambda_k)$ such that for all $j\in \{\underline{k},\ldots, \bar{k}\}$
   \begin{equation}\label{uniform-error-interpolation}
       \norm{I_n \varphi^\bn_j-\varphi^{\cn}_j}_{L^2} \leq \frac{C}{\Gamma_{K(n)}} \delta(\eps_n, h_n, \lambda_{K(n)})
   \end{equation}and 
   \begin{equation}\label{uniform-error-stepfct-extension}
       \norm{\mP^*_n \varphi^\bn_j-\varphi^{\cn}_j}_{L^2} \leq C\delta^\phi_n,
    \end{equation} where
   $I_n:L^2(\M_n, \mu_n)\to \text{Lip}\,(\M)$, $n\in\N$ are the interpolation operators from \citep[p.~840]{GarGerHeiSle20}, and \begin{equation}\label{def-delta-phi-n}
       \delta^\phi_n:=h_n\sqrt{\lambda_{K(n)}} + \frac{1}{\Gamma_{{K(n)}}}\delta(\eps_n, h_n, \lambda_{K(n)}).
   \end{equation}
\end{thm}
\begin{proof}
     To prove the asserted error bounds we aim to apply Theorem 4 (or in particular Corollary 1) and Theorem 5 in \citep[pp. 837-840]{GarGerHeiSle20}. However, \cite[Thm.~5]{GarGerHeiSle20} states only the existence of an eigenfunction of the Laplace-Beltrami operator $\Delta$ corresponding to a given eigenfunction of the graph Laplacian $\Delta_n$, such that the latter approximates the former well. The statement we want to show, however, is not only about the approximation of single eigenfunctions, but about  approximation of eigenspaces of $\Delta$, which is covered by \cite[Thm.~4]{Burago:2015}. Therefore, this proof is structured as follows, first  we will show that indeed the assumptions for Theorem 4 (or in particular Corollary 1) and Theorem 5 in \citep[pp. 837-840]{GarGerHeiSle20} are satisfied in our setting, then we reason that \cite[Thm.~5]{GarGerHeiSle20} indeed yields the desired eigenspace approximation properties. Finally, we will prove that this then yields the stated uniform-per-resolution bounds.

    First, we will show that there is a constant $c_\M>0$ such that for all $n\in\N$ and $k\leq K(n)$ \[\left(\sqrt{\lambda_k}+1\right)h_n \leq \left(\sqrt{\lambda_{K(n)}}+1\right)h_n \leq c_\M,\] which is precisely the assumption required for the results of \cite{GarGerHeiSle20} mentioned above.
    Weyl's law states that the number of eigenvalues (including multiplicity) smaller than a $\lambda>0$ \begin{equation}\label{eigencount}N(\lambda):=\left|\{k\in\N \mid \lambda_k\leq \lambda, \lambda_k \text{ eigenvalue of }\Delta\}\right|\end{equation} asymptotically behave like
        \[\lim_{\lambda\to\infty} N(\lambda)\lambda^{-m/2} = \frac{\omega_m}{(2\pi)^m}\Vol_\M(\M)=:C_{\M},\] where $\omega_m$ denotes the volume of the $m$-ball $B(0)\subset \R^m$. Specifically, we will use the more refined asymptotic law due to Hörmander \cite[Thm.~173]{Berger:2003}
        \begin{equation}\label{hormander-rate}
        N(\lambda) = \frac{\omega_m}{(2\pi)^m}\Vol_\M(\M)\lambda^{m/2} + O(\lambda^{\frac{m-1}{2}}).
        \end{equation}        
        Let $\{k_j\}_j\subset \N$ enumerate the first index in each multiplicity block, i.e. $\lambda_{k_j}\neq \lambda_{k_j-1}$, and denote by $m(\lambda_k)$ the multiplicity of $\lambda_k$. Then
        \[N(\lambda_k) = \sum\limits_{j\in\N,k_j\leq k}m(\lambda_{k_j})\geq k.\]
        Let $\delta:=\lambda_2/2$ and observe that \begin{align}\label{multiplicity-ineq}
            m(\lambda_{k_j})\leq N(\lambda_{k_j}+\delta)-N(\lambda_{k_j}-\delta)
        \end{align} for all $j\in\N$ , $j\geq 2$. 
        Furthermore applying the mean-value theorem to the  differentiable function $g:[\lambda_{k_j}-\delta, \lambda_{k_j}+\delta]\to\R, x\mapsto x^{m/2}$   yields for $m\geq 2$ that 
        \begin{align}\label{mean-value-thm-ineq}
            (\lambda_{k_j}+\delta)^{m/2} - (\lambda_{k_j}-\delta)^{m/2} \leq m \delta(\lambda_{k_j}+\delta)^{\frac{m-2}{2}}.
        \end{align}
        By \eqref{hormander-rate} there is $M'>0$ such that for all $j\in\N$, $j\geq 2$ it holds that
        \begin{gather}
        \begin{aligned}\label{hormander-rate-o-not-ineq}
            |N(\lambda_{k_j}+\delta)-C_\M(\lambda_{k_j}+\delta)^{m/2}|\leq M' (\lambda_{k_j}+\delta)^{\frac{m-1}{2}},\\
            |C_\M(\lambda_{k_j}-\delta)^{m/2} - N(\lambda_{k_j}-\delta)| \leq M' (\lambda_{k_j}-\delta)^{\frac{m-1}{2}}.
        \end{aligned}
        \end{gather}
        Combining all these considerations \eqref{multiplicity-ineq}, \eqref{mean-value-thm-ineq}, \eqref{hormander-rate-o-not-ineq} we obtain that there is an $M>0$ such that for all $j\in\N$, $j\geq 2$ it is for $m\geq 2$
        \begin{align*}
            m(\lambda_{k_j})&\leq N(\lambda_{k_j}+\delta)-N(\lambda_{k_j}-\delta)\\
            &\leq  |N(\lambda_{k_j}+\delta)-C_\M(\lambda_{k_j}+\delta)^{m/2}| + C_\M(\lambda_{k_j}+\delta)^{m/2}\\
            &\hspace{10pt}+ |C_\M(\lambda_{k_j}-\delta)^{m/2} - N(\lambda_{k_j}-\delta)| - C_\M(\lambda_{k_j}-\delta)^{m/2}\\
            &\leq C_\M\left((\lambda_{k_j}+\delta)^{m/2} - (\lambda_{k_j}-\delta)^{m/2}\right) + M' (\lambda_{k_j}+\delta)^{\frac{m-1}{2}} + M' (\lambda_{k_j}-\delta)^{\frac{m-1}{2}}\\
            &\leq C_\M m \delta(\lambda_{k_j}+\delta)^{\frac{m-2}{2}} + 2M'(\lambda_{k_j}+\delta)^{\frac{m-1}{2}}\\
            &\leq M(\lambda_{k_j}+\delta)^{\frac{m-1}{2}}\leq M(3/2)^{\frac{m-1}{2}}\lambda_{k_j}^{\frac{m-1}{2}}.
        \end{align*}
        Thus we may conclude for all $m\in\N$ that 
         \begin{equation}\label{multiplicity-max-growth}
            m(\lambda_k)=O\left(\lambda_k^{\frac{m-1}{2}}\right)\qquad \text{ for }  k\in\N, k\geq 2.
        \end{equation}
        Therefore $\lim\limits_{k\to\infty}m(\lambda_k)\lambda_k^{-m/2} = 0$, since \vskip -1em
        \[m(\lambda_k)\lambda_k^{-m/2} = O\left(\lambda_k^{\frac{m-1}{2}}\right)\lambda_k^{-m/2} = O\left(\lambda_k^{-1/2}\right),\]\vskip -1em 
        which yields 
         \[\lim_{k\to\infty} k\lambda_k^{-m/2} \leq \lim_{k\to\infty} N(\lambda_k)\lambda_k^{-m/2} \leq \lim_{k\to\infty} (k+m(\lambda_k)) \lambda_{k}^{-m/2} = \lim_{k\to\infty} k\lambda_k^{-m/2}, \] and hence 
        \begin{equation}\label{weyls-law-multiplicity}
            C_\M = \lim_{k\to\infty} N(\lambda_k)\lambda_k^{-m/2} = \lim_{k\to\infty} k\lambda_k^{-m/2}.
        \end{equation}
        So with this and since by construction 
        $\lambda_{K(n)}=\lambda_{\widetilde{K}(n)}$ as well as $\widetilde{K}(n)^{\frac{1}{m}}\leq \widetilde{K}(n)$, condition \eqref{sequence-conditions} yields that for every $n\in\N$
        \begin{equation}\label{convergence-h_n-eigenvalue}
        \begin{aligned}
        0\leq \lim_{n\to\infty} h_n\sqrt{\lambda_{K(n)}} &= \lim_{n\to\infty} h_n\widetilde{K}(n)^{1/m}\widetilde{K}(n)^{-1/m}\sqrt{\lambda_{K(n)}} 
        \\&= \lim_{n\to\infty} h_n\widetilde{K}(n)^{1/m}\lim_{n\to\infty} \widetilde{K}(n)^{-1/m}\sqrt{\lambda_{K(n)}}
        \\&\leq \lim_{n\to\infty} h_n\widetilde{K}(n)\lim_{n\to\infty} \widetilde{K}(n)^{-1/m}\sqrt{\lambda_{\widetilde{K}(n)}}
        \\&= 0\cdot \lim_{k\to\infty} k^{-1/m}\sqrt{\lambda_{k}}
        = 0\cdot \lim_{k\to\infty} (k\lambda_{k}^{-m/2})^{-\frac{1}{m}}= 0\cdot C_\M^{-1/m} =0,
        \end{aligned}
        \end{equation}
        so $\{\sqrt{\lambda_{K(n)}} + 1)h_n\}_n$ is a bounded sequence converging to $0$, and therefore there is a constant $c_\M>0$ such that for all $n\in\N$ and $k\leq K(n)$ \[\left(\sqrt{\lambda_k}+1\right)h_n \leq \left(\sqrt{\lambda_{K(n)}}+1\right)h_n \leq c_\M.\]
        Hence, Theorem 4 (or in particular Corollary 1) in \citep[pp. 837-840]{GarGerHeiSle20} yield \eqref{eigenvalues-ratio-uniform-bound}.
    
    Notice that following \cite[p.~878]{GarGerHeiSle20} the Lemmata \cite[Lem. 7.2,~7.3,~7.4]{Burago:2015} carry over to the setting of \cite{GarGerHeiSle20}, by replacing the application of \cite[Thm.~1,~Lem.~6.2]{Burago:2015} with \cite[Thm.~4,~Lem.~14]{GarGerHeiSle20}. Consequently, the proof of \cite[Thm.~4]{Burago:2015} applies in the setting of \cite{GarGerHeiSle20} (see  \cite[p.~878]{GarGerHeiSle20})

    Now, since the assumptions for \cite[Thm.~5]{GarGerHeiSle20} are satisfied, for the given choice of $\{\varphi^\bn_j\}_{j=\underline{k}}^{\bar{k}}$ of orthonormal eigenfunctions of the $n$-th graph Laplacian there exist orthonormal eigenfunctions $\{\varphi^{\cn}_j\}_{j=\underline{k}}^{\bar{k}}$ satisfying 
    \begin{equation}\label{eigenfunctions-interpolation-uniform-bound}
              \norm{I_n \varphi^\bn_j - \varphi_j}_{L^2}\leq \frac{C}{\gamma_{k}} \delta(\eps_n, h_n, \lambda_j) \qquad\text{ for all } j\in\{\underline{k}, \ldots, \bar{k}\}.
          \end{equation}
      Clearly, the eigenfunctions $\{\varphi^{\cn}_j\}_{j=\underline{k}}^{\bar{k}}$ then form an orthonormal basis of $E(\lambda_k)$.
      Since for all $j\in\{\underline{k}, \ldots, \bar{k}\}$ it holds by construction that $\lambda_j\leq\lambda_{K(n)}$ and that $\gamma_{j}\geq \Gamma_{K(n)}$, \eqref{eigenfunctions-interpolation-uniform-bound} implies \eqref{uniform-error-interpolation}.

    Then following \cite[p. 879]{GarGerHeiSle20} and using \eqref{eigenfunctions-interpolation-uniform-bound}, we get for all $j\in\{\underline{k}, \ldots, \bar{k}\}$
    \begin{align*}
        \norm{\mP^*_n\varphi^\bn_j - \varphi^{\cn}_j}_{L^2} &\leq \left(\norm{I_n\varphi^\bn_j - \mP^*_n\varphi^\bn_j}_{L^2} + \norm{I_n \varphi^\bn_j - \varphi^{\cn}_j}_{L^2}\right)\\
        &\leq Ch_n\sqrt{\lambda_j} + \frac{C}{\gamma_{j}} \delta(\eps_n, h_n, \lambda_j)\\
         &\leq C\left(h_n\sqrt{\lambda_{K(n)}} + \frac{1}{\Gamma_{K(n)}} \delta(\eps_n, h_n, \lambda_{K(n)})\right).\qedhere
    \end{align*}
\end{proof}

Note that the choice of eigenfunctions of the Laplace-Beltrami operator in \Cref{uniform-bounds-spectra} depend on the choice of eigenfunctions of the graph Laplacian for a fixed resolution $n$. 
\begin{rem}\label{non-uniform-data-distribution}
The bounds on $\eps_n$ from \eqref{probabilistic-bounds-eps} as well as the above-cited results on convergence rates for eigenvalues and eigenfunctions, also extend to the case of weighted Laplacians, i.e. to nonuniform sampling measures $\rho\mu$ with Lipschitz continuous density $\rho$ bounded uniformly away from zero and accordingly weighted discrete measures $\bm{m}_n\mu_n$ approximating $\rho\mu$, which is the framework used in \cite{GarGerHeiSle20}. We believe our analysis can be generalized to this case, but for brevity we restrict ourselves to the case of constant density, since otherwise a number of arguments in \Cref{discrete-continuum-erm-convergence} would become more cumbersome.
\end{rem}

     \begin{cor}\label{convergence-relative-error-bound-delta-cor}
     For the error bounds from \Cref{uniform-bounds-spectra}, it holds that 
     \begin{equation}\label{convergence-relative-error-bound-delta}
         \lim_{n\to\infty}\delta(\eps_n, h_n, \lambda_{K(n)}) = 0,
      \end{equation} or in particular for every $k\in\N$ and $N_k$ such that $k\leq K(n)$ we have
     \[\lim_{n\to\infty} \delta(\eps_n, h_n, \lambda_{k}) = 0.\]
     \end{cor}
    \begin{proof}
         This follows immediately from the observation \eqref{convergence-h_n-eigenvalue} and the assumption \eqref{assumption3}.
    \end{proof}
    \begin{lemma}
    It holds that 
    \begin{equation}\label{convergence-relative-error-bound-delta-max-spectral-gap-kn}
        \lim_{n\to\infty} \frac{1}{\Gamma_{K(n)}}\delta(\eps_n, h_n, \lambda_{K(n)})K(n)^{\frac{1}{2}} = 0.
    \end{equation}
    \end{lemma}
    \begin{proof}
    We need to show that
        \begin{align*}
            \lim_{n\to\infty}\frac{1}{\Gamma_{K(n)}} \widetilde{K}(n)^{\frac{1}{2}} \left(\frac{\eps_n}{h_n} + \left(1+\sqrt{\lambda_{K(n)}}\right)h_n + \left(K+\frac{1}{R^2}\right)h_n^2\right) = 0.
        \end{align*}
        Now the conditions \eqref{sequence-conditions} yield
        \begin{align*}\lim_{n\to\infty} h_n^2 \frac{1}{\Gamma_{K(n)}} \widetilde{K}(n)^{\frac{1}{2}} &= \lim_{n\to\infty} h_n \frac{1}{\Gamma_{K(n)}} \widetilde{K}(n)^{\frac{m+2}{2m}} \widetilde{K}(n)^{-1/m} h_n \\
        &= \lim_{n\to\infty} h_n \frac{1}{\Gamma_{K(n)}} \widetilde{K}(n)^{\frac{m+2}{2m}} \lim_{n\to\infty} h_n \lim_{n\to\infty} \widetilde{K}(n)^{-1/m} = 0,\end{align*}
        as well as 
        \begin{align*}
        \lim_{n\to\infty} h_n \frac{1}{\Gamma_{K(n)}} \sqrt{\lambda_{K(n)}} \widetilde{K}(n)^{\frac{1}{2}}  &= \lim_{n\to\infty} h_n \frac{1}{\Gamma_{K(n)}} \widetilde{K}(n)^{\frac{1}{2}} \widetilde{K}(n)^{1/m} \widetilde{K}(n)^{-1/m} \sqrt{\lambda_{K(n)}} \\
        &= \lim_{n\to\infty} h_n \frac{1}{\Gamma_{K(n)}} \widetilde{K}(n)^{\frac{m+2}{2m}} \lim_{n\to\infty} \widetilde{K}(n)^{-1/m} \sqrt{\lambda_{\widetilde{K}(n)}} \\
        &= 0\cdot \lim_{k\to\infty} k^{-1/m}\sqrt{\lambda_k} =  0\cdot \lim_{k\to\infty} (k\lambda_k^{-m/2})^{-1/m}= 0\cdot C_\M^{-1/m} = 0,\end{align*}
        and finally \[\lim_{n\to\infty}\frac{\eps_n}{h_n}\frac{1}{\Gamma_{K(n)}} \widetilde{K}(n)^{\frac{1}{2}} = 0.\qedhere\]
     \end{proof}
        As an immediate consequence of \eqref{convergence-relative-error-bound-delta-max-spectral-gap-kn} we obtain:
     \begin{cor}
     It holds that 
     \begin{equation}\label{convergence-relative-error-bound-delta-max-spectral-gap}
         \lim_{n\to\infty} \frac{1}{\Gamma_{K(n)}} \delta(\eps_n, h_n, \lambda_{K(n)}) = 0.
     \end{equation}
     \end{cor}

\begin{lemma}\label{kn-delta-phi-n-to-0}
    It holds that $K(n)^{1/2}\delta^\phi_n\xrightarrow{n\to\infty}0$.
\end{lemma}
\begin{proof}
We have to show that 
\begin{align*}
    \lim_{n\to\infty} K(n)^{\frac{1}{2}}h_n\sqrt{\lambda_{K(n)}} + \frac{1}{\Gamma_{K(n)}}\delta(\eps_n, h_n, \lambda_{K(n)} = 0.
\end{align*}
Due to \eqref{convergence-relative-error-bound-delta-max-spectral-gap-kn} it remains to show that
\begin{align*}
    \lim_{n\to\infty}K(n)^{\frac{1}{2}}h_n\sqrt{\lambda_{K(n)}} = 0.
\end{align*}

Using \eqref{weyls-law-multiplicity} we first observe that
\begin{align*}
    \lim_{n\to\infty} K(n)\widetilde{K}(n)^{-1} & = \lim_{n\to\infty} K(n)\lambda_{K(n)}^{-\frac{m}{2}}\lambda_{\widetilde{K}(n)}^{\frac{m}{2}}\widetilde{K}(n)^{-1} \\
    &= \lim_{k\to\infty} k\lambda_k^{-\frac{m}{2}} \lim_{k\to\infty} \lambda_k^{\frac{m}{2}}k^{-1} = C_\M\cdot C_\M^{-1} = 1.
\end{align*}
Furthermore, since $\Gamma_{K}\leq \gamma_1$ by definition and using again \eqref{weyls-law-multiplicity}, as well as \eqref{sequence-conditions} we obtain
\begin{align*}
    0&\leq \lim_{n\to\infty} K(n)^{\frac{1}{2}}h_n \sqrt{\lambda_{K(n)}} \\
    &= \lim_{n\to\infty} K(n)^{\frac{1}{2}}\widetilde{K}(n)^{\frac{1}{2}}\Gamma_{\widetilde{K}(n)}\Gamma_{\widetilde{K}(n)}^{-1}\widetilde{K}(n)^{-\frac{1}{2}}h_n \widetilde{K}(n)^{\frac{1}{m}}\widetilde{K}(n)^{-\frac{1}{m}}\sqrt{\lambda_{\widetilde{K}(n)}}\\
    &\leq \gamma_1\lim_{n\to\infty}K(n)^{\frac{1}{2}}\widetilde{K}(n)^{-\frac{1}{2}}\lim_{n\to\infty} \Gamma_{\widetilde{K}(n)}^{-1}\widetilde{K}(n)^{\frac{m+2}{2m}}h_n \lim_{n\to\infty} \widetilde{K}(n)^{-\frac{1}{m}}\sqrt{\lambda_{\widetilde{K}(n)}}
    \\&= \gamma_1 \cdot 1\cdot 0 \cdot C_\M^{-\frac{1}{m}} = 0.
    \qedhere
\end{align*}
    \end{proof}

An immediate corollary of the previous \Cref{kn-delta-phi-n-to-0} is:
\begin{cor}\label{delta-phi-n-to-0}
    It holds that $\delta^\phi_n\xrightarrow{n\to\infty}0$.
\end{cor}

\section{Continuum-to-discrete projections}\label{discrete-continuum-erm-convergence}
In order to show consistency of the graph networks across resolutions, we need some ways to meaningfully compare objects in the discrete and continuum spaces. A difficulty in comparing discrete neural response maps $\psi_n$ to the continuum map $\psi$ is that the activation function acts pointwise, while the convolution operation is defined spectrally.

In the previous \Cref{signal-consistency} we already defined a spatial discretization $\mP_n$ of $L^2$-functions, through averaging over optimal transportation cells. 
 While this yields a natural approximation compatible with both the sampling and the $L^2$-structure, quantifying the loss of information by such a discretization is problematic when dealing with spectral operations. In contrast, the normalized eigenfunctions of the Laplace-Beltrami operator and the orthonormal eigenvectors of the graph Laplacians provide bases for the respective $L^2$-spaces, which allows for the definition of spectral projection maps that preserve the Hilbert structure of $L^2$ and its Sobolev subspaces. 

 This section first introduces such spectral projections, and subsequently concerns itself with the properties of both spatial and spectral projections, and their interplay. Neither projection is generally better suited for the discrete-to-continuum analysis of nonlinear convolutional neural response maps. Rather, they serve complementary roles. On the one hand, convolutions are expressed directly as multiplications in the spectral domain. And on the other hand, the spatial extension maps, being constant on each transport cell, commute with the pointwise application of the activation function, which is crucial for handling the nonlinearity.

Equipped with these projection mappings, at the end of this section we obtain pointwise convergence of the neural response maps $\psi_n$ to $\psi$ under appropriate discretization of the parameters and the signal in \Cref{convergence-neural-response-discrete-continuum}. 
As a consequence of the tension between their spatial and spectral aspects, and the fact that graph convolution is based on the respective graph eigenfunctions, the discrete expression has to be lifted spatially, whereas asymptotic control of convolution operations requires the introduction of spectral projections.
This leads to controlling the asymptotic behavior of expressions such as
\[\mP_n^*(\mS_n b \ast_n \mR_n u) - \mS_n^*\mS_n b \ast \mP_n^*\mR_n u,\]
appearing in \eqref{big-response-commutation}. Here, $\mP_n^*$ denotes the spatial extension map from \eqref{def-spatial-extension}, $\mS_n$ a spectral discretization map with adjoint $\mS_n^*$ and $\mR_n$ generic discretization maps, that are compatible with $TL^2$-convergence and uniformly bounded in $n$. The latter are not made specific because they are reflected in conditions for the training data, which we want to restrict as little as possible. As can be seen from this, there appears a melange of discretizations and extension operators, so in \Cref{prop-discrete-continuum-proj} we will give all the convergence results required to treat such interactions.

\subsection{Spectral projections and their adjoints}
We first introduce the spectral projections from continuum to discrete signals used in the sequel, their adjoints which serve as extension operators. 

For any $\alpha\in[0,1]$ we introduce the \emph{$H^\alpha$-spectral projection maps} $\mS_{n, \alpha}: L^2(\M, \mu)\to L^2(\M_n, \mu_n)$, $n\in\N$, given as follows:
For every $n\in\N$ fix orthonormal eigenfunctions of the $n$-th graph Laplacian $\{\phi^\bn_k\}_{k=1}^{K(n)}$ corresponding to the first $K(n)$ discrete eigenvalues. Then we can find for every $n\in\N$ by \Cref{uniform-bounds-spectra} continuum eigenfunctions $\{\phi_{k}^{\cn}\}_{k=1}^{K(n)}$ such that \[\norm{I_n \phi^\bn_k-\phi_k^{\cn}}_{L^2} \leq \frac{C}{\Gamma_{K(n)}} \delta(\eps_n, h_n, \lambda_{K(n)})\]and 
   \[\norm{\mP^*_n \phi^\bn_k-\phi_k^{\cn}}_{L^2} \leq Ch_n\sqrt{\lambda_{K(n)}} + \frac{C}{\Gamma_{K(n)}} \delta(\eps_n, h_n, \lambda_{K(n)})=C \delta^\phi_n.\]
Note that we first fix a choice of graph eigenfunctions $\{\phi^\bn_k\}_{k=1}^{K(n)}$ for every resolution, to then be able to choose for every resolution a collection of continuum eigenfunctions $\{\phi^\cn_k\}_{k=1}^{K(n)}$ that aligns well with the choice for the graph of that resolution.

From this fixed choice of collections of eigenfunctions of the graph Laplacian and the Laplace-Beltrami operator, we then define $S_{n,\alpha}$ such that for $v\in L^2(\M, \mu)$ and $n\in\N$
\begin{equation}\label{def-spectral-discretization}
    \mS_{n,\alpha} v := \sum\limits_{k=1}^{K(n)} \left(\frac{1+\sqrt{\lambda_k}}{1+\sqrt{\lambda_k^\bn}}\right)^{\alpha}\langle v, \phi_k^{\cn}\rangle_{L^2(\M, \mu)}\phi^\bn_k,
\end{equation}
so $S_{n,\alpha}$ discretizes continuum signals by retaining only the first $K(n)$ Laplace-Beltrami modes in the resolution-dependent orthonormal basis $\{\phi_k^{\cn}\}$ induced by the chosen orthonormal graph eigenbasis, scaling the associated spectral coefficients according to the regularity parameter $\alpha$ and projecting back to the discrete spatial domain via said eigenvectors of the $n$-th graph Laplacian.

Their adjoint extension map $\mS^*_{n,\alpha}: L^2(\M_n, \mu_n)\to L^2(\M, \mu)$, $n\in\N$ are given by 
\begin{equation}\label{def-adjoint-spectral-extension}
    \mS^*_{n,\alpha} w = \sum\limits_{k=1}^{K(n)}  \left(\frac{1+\sqrt{\lambda_k^\bn}}{1+\sqrt{\lambda_k}}\right)^{\alpha} \langle w, \phi^\bn_k\rangle_{L^2(\M_n, \mu_n)}\phi_k^{\cn}.
\end{equation}
We write shorthand $\mS_n: L^2(\M, \mu)\to L^2(\M_n, \mu_n)$, and $\mS^*_n: L^2(\M_n, \mu_n)\to L^2(\M, \mu)$ for $\mS_{n, 0}$ and $\mS^*_{n, 0}$, $n\in\N$, respectively.

Denote by $\mathbb{P}(V,w)$ the $L^2(\M, \mu)$-orthogonal projection of $w\in L^2(\M, \mu)$ onto the closed subspace $V\subset L^2(\M, \mu)$ and by $E(\lambda)$ the eigenspace of the Laplace-Beltrami operator associated to the eigenvalue $\lambda\geq 0$.
In the following, we will denote by $\{\varphi_k\}_k$ an arbitrary orthonormal basis of $L^2(\M,\mu)$ consisting of eigenfunctions of the Laplace-Beltrami operator, such that for any $k\in\N$ the eigenfunction $\varphi_k$ corresponds to the eigenvalue $\lambda_k$.  Note that for any such choice and any $w\in L^2(\M, \mu)$ it holds that \begin{equation}\label{eigenspace-projection-basis-free}
    \mathbb{P}\left(\bigoplus_{k=1}^{K(n)} E(\lambda_k), w\right)=\sum_{k=1}^{K(n)}\langle w,\varphi_k\rangle_{L^2(\M, \mu)}\varphi_k = \sum_{k=1}^{K(n)}\langle w,\phi_k^\cn\rangle_{L^2(\M, \mu)}\phi_k^\cn,\end{equation}
due to the orthogonality of distinct eigenspaces, and since the orthogonal projection onto each eigenspace is independent of the choice of orthonormal basis.

\subsection{Properties of the continuum-to-discrete projections}\label{prop-discrete-continuum-proj}

Here we will prove properties for each of the continuum-to-discrete projections separately, but also jointly and in relation to convolutions and graph convolutions, all of which are required to show \Cref{convergence-discrete-continuum-convolution-operation}.

\begin{lemma}[Properties of the averaging discretization and stepwise extension]\label{properties-averaging-discretization}
    The following properties hold for all $n\in \N$ for the maps $\mP_n$ and $\mP^*_n$:
    \begin{lemlist}
    \item The maps $\mP_n$, 
    $\mP_n^*$ are linear.
    \item \label{spatial-adjointness} It holds that $\mP_n$ and $\mP_n^*$ are $L^2$-adjoint, that is 
    \[\langle \mP_n u, v\rangle_{L^2(\M_n, \mu_n)} = \langle  u,\mP_n^* v\rangle_{L^2(\M, \mu)} \quad\text{ for all } u\in L^2(\M, \mu), v\in L^2(\M_n, \mu_n).\]
    \item\label{adjoint-l2-innerproduct-preserving}  The extension map $\mP^*_n$ preserves the $L^2$-inner product, i.e. \[\langle u, v\rangle_{L^2(\M_n, \mu_n)} = \langle \mP^*_n u, \mP^*_n v \rangle_{L^2(\M, \mu)} \quad\text{ for all } u, v\in L^2(\M_n, \mu_n).\]
    \item\label{adjoint-l2-norm-preserving} The extension map $\mP^*_n$ preserves the $L^2$-norm, i.e. \[\norm{\mP^*_n v}_{L^2(\M, \mu)} = \norm{v}_{L^2(\M_n, \mu_n)}\quad\text{ for all } v\in L^2(\M_n, \mu_n).\]
    \item\label{step-fct-discretization-l2-norm-preserving} The operator $\mP^*_n\mP_n:L^2(\M, \mu)\to L^2(\M, \mu)$ is linear and bounded with operator norm $\norm{\mP^*_n\mP_n} = 1$. In particular, it holds that $\mP^*_n\mP_n$ is contracting in the $L^2$-norm, i.e. \[\norm{\mP^*_n\mP_n v}_{L^2(\M, \mu)} \leq \norm{v}_{L^2(\M, \mu)} \quad\text{for all }v\in L^2(\M, \mu).\]
    \item\label{commute-adjoint-activation} It holds that $\big(\sigma(\mP_n^*v)\big)(x) = \mP^*_n\big(\sigma(v)\big)(x)$ for any $v\in L^2(\M_n, \mu_n)$ and $x\in M$.
\end{lemlist}
\end{lemma}

\begin{proof}
\textup{(i)}, \textup{(ii)} and \textup{(iii)} follow immediately from the definitions and imply \textup{(iv)}.
    \begin{enumerate}[label = {(\roman*)}]
    \setcounter{enumi}{4}
        \item Linearity follows immediately from the definition of $\mP_n$ and $\mP^*_n$. Fix an arbitrary $v\in L^2(\M, \mu)$. Then
            \begin{align*}
                \norm{\mP_n^*\mP_n v}_{L^2(\M, \mu)} &= \int_\M \left|(\mP_n^*\mP_n v)(x)\right|^2 d\mu(x) = \sum\limits_{i=1}^n\int_{U_i^\bn} \left|\mP_n v(x_i)\right|^2 d\mu(x) \\
                &=\sum\limits_{i=1}^n\int_{U_i^\bn} \left|\frac{1}{\mu(U_i^\bn)} \int_{U_i^\bn}v(y)d\mu(y)\right|^2 d\mu(x) \\&=\sum\limits_{i=1}^n\mu(U_i^\bn) \,\left|\frac{1}{\mu(U_i^\bn)} \int_{U_i^\bn}v(y)d\mu(y)\right|^2  \\
                \text{(Jensen's ineq)}&\leq \sum\limits_{i=1}^n\frac{\mu(U_i^\bn)}{\mu(U_i^\bn)}\int_{U_i^\bn}\left|v(y)\right|^2d\mu(y) = \int_\M \left|v(y)\right|^2 d\mu(y) = \norm{v}^2_{L^2(\M, \mu)}.
            \end{align*}
            But on the other hand, let $\{c_i\}_{i=1}^n\subset \R$, and define $v\in L^2(\M, \mu)$ as $v(x)=\sum\limits_{i=1}^n c_i\mathds{1}_{U_i^\bn}(x)$. Then it is easy to check that $\mP_n^*\mP_n v(x) = v(x)$ for all $x\in \M$, thus \[\norm{(\mP_n^*\mP_n )v}_{L^2(\M, \mu)} = \norm{v}_{L^2(\M, \mu)},\] yielding operator norm $\norm{\mP_n^*\mP_n} = 1$.
            \end{enumerate}
    \begin{enumerate}[label = {(\roman*)}]
    \setcounter{enumi}{5}
        \item Writing out the definitions yields \[\big(\sigma(\mP_n^*v)\big)(x) = \sigma\big((\mP_n^*v)(x)\big) = \sigma\big(\sum\limits_{i=1}^n v(x_i)\mathds{1}_{U_i^\bn}(x)\big) = \sum\limits_{i=1}^n \sigma(v(x_i))\mathds{1}_{U_i^\bn}(x) = \mP^*_n\big(\sigma(v)\big)(x).\] The second to last equality holds, despite the non-linearity of $\sigma$, because of the disjoint tesselation by the $U_i^\bn, i\in\{1, \ldots n\}$.\qedhere
    \end{enumerate}
\end{proof}

\begin{lemma}\label{convergence-extended-discretization-pn}
    For any $v\in L^2(\M, \mu)$ it holds that $\norm{\mP_n^*\mP_n v - v}_{L^2(\M, \mu)}\xrightarrow{n\to\infty} 0.$
\end{lemma}

\begin{proof}
    Fix an arbitrary $v\in L^2(\M, \mu)$ and an $\eps>0$. 
        By density of $C_c(\M)$ in $L^2(\M, \mu)$, there exists $v_\eps\in C_c(\M)$ with $\norm{v-v_\eps}_{L^2(\M, \mu)}<\eps$. 
        From \Cref{step-fct-discretization-l2-norm-preserving} it follows that 
        \[\norm{\mP_n^*\mP_n v_\eps-\mP_n^*\mP_n v}_{L^2(\M, \mu)} = \norm{\mP_n^*\mP_n (v_\eps-v)}_{L^2(\M, \mu)}\leq \norm{v_\eps - v}_{L^2(\M, \mu)}<\eps.\]
        Furthermore, we have that
        \begin{align*}
            \norm{v_\eps-\mP_n^*\mP_n v_\eps}_{L^2(\M, \mu)}^2&= \int_\M \left|(v_\eps-\mP_n^*\mP_n v_\eps)(x)\right|^2 d\mu(x) \\&= \sum\limits_{i=1}^n\int_{U_i^\bn} \left|v_\eps(x)-\mP_n v_\eps(x_i)\right|^2 d\mu(x) \\
                &=\sum\limits_{i=1}^n\int_{U_i^\bn} \left|v_\eps(x)-\frac{1}{\mu(U_i^\bn)} \int_{U_i^\bn}v_\eps(y)d\mu(y)\right|^2 d\mu(x) \\
                &=\sum\limits_{i=1}^n\int_{U_i^\bn} \left|\frac{1}{\mu(U_i^\bn)}\int_{U_i^\bn}v_\eps(x) - v_\eps(y)d\mu(y)\right|^2 d\mu(x)\\
                \text{(Jensen's ineq)}&\leq \sum\limits_{i=1}^n\frac{1}{\mu(U_i^\bn)} \int_{U_i^\bn} \int_{U_i^\bn}\left|v_\eps(x) - v_\eps(y)\right|^2d\mu(y) d\mu(x)
        \end{align*}
        Since the support of $v_\eps$ is compact, it is actually uniformly continuous.
By construction (see \eqref{optimal-transportation-cells}), it holds that 
    \begin{equation}\label{diam_transport_cells}\diam_\M(U_i^\bn)\leq 2\eps_n,\end{equation}
 hence $\diam_\M(U_i^\bn)\to 0$ for all $i=1, \ldots, n$ as $n\to\infty$.       
        Thus in particular, there is an $N\in\N$, such that for all $n\geq N$ and for all $x,y\in U_i^\bn$, $i=1, \ldots, n$ it holds that $|v_\eps(x)-v_\eps(y)|<\eps$.

        Hence for all $n\geq N$
        \begin{align*}
            \norm{v_\eps-\mP_n^*\mP_n v_\eps}_{L^2(\M, \mu)}^2&\leq \sum\limits_{i=1}^n\frac{1}{\mu(U_i^\bn)} \int_{U_i^\bn} \int_{U_i^\bn}\left|v_\eps(x) - v_\eps(y)\right|^2d\mu(y) d\mu(x)\\
            &< \sum\limits_{i=1}^n\frac{1}{\mu(U_i^\bn)} \int_{U_i^\bn} \int_{U_i^\bn}\eps^2 d\mu(y) d\mu(x)\\
            &= \sum\limits_{i=1}^n\frac{\mu(U_i^\bn)^2}{\mu(U_i^\bn)} \eps^2 = \sum\limits_{i=1}^n \mu(U_i^\bn) \eps^{2} = \frac{1}{n} n \eps^2 = \eps^2
        \end{align*}
        Combining the above yields
        \begin{align*}
            \norm{v-\mP_n^*\mP_n v}_{L^2} \leq \norm{v-v_\eps}_{L^2} + \norm{v_\eps -\mP_n^*\mP_n v_\eps}_{L^2} + \norm{\mP_n^*\mP_n v_\eps - \mP_n^*\mP_n v}_{L^2}<3\eps.
        \end{align*}
        Since $\eps>0$ and $v\in L^2(\M, \mu)$ were arbitrary, this proves the claim.
\end{proof}

\begin{lemma}[Properties of the $H^\alpha$ spectral projections]\label{properties-spectral-h-alpha-discretization}
    The following properties hold for all $n\in \N$, $\alpha\in [0,1]$ for the maps $\mS_{n,\alpha}$ and $\mS^*_{n,\alpha}$:
    \begin{lemlist}
    \item \label{spectral-h-alpha-adjointness} The maps $\mS_{n,\alpha}$ and $\mS_{n,\alpha}^*$ are $H^\alpha$-adjoint, that is \[\langle \mS_{n,\alpha} v, w\rangle_{H^\alpha(\M_n, \mu_n)} = \langle  v,\mS_{n,\alpha}^* w\rangle_{H^\alpha(\M, \mu)}\quad \text{ for all } v\in H^\alpha(\M, \mu), w\in H^\alpha(\M_n, \mu_n).\]
    \item \label{spectral-h-alpha-preservation-innerproduct}For all $u, v\in L^2(\M_n, \mu_n)$ with $\langle u, \phi^\bn_k\rangle_{L^2(\M_n, \mu_n)}=\langle v, \phi^\bn_k\rangle_{L^2(\M_n, \mu_n)}=0$ for all $k>K(n)$ it holds that \[\langle u, v\rangle_{H^\alpha(\M_n, \mu_n)} = \langle \mS^*_{n,\alpha} u, \mS^*_{n,\alpha}  v \rangle_{H^\alpha(\M, \mu)},\] and  for all $u, v\in L^2(\M, \mu)$ with $\langle u, \phi_k\rangle_{L^2}=\langle v, \phi_k\rangle_{L^2}=0$ for all $k>K(n)$ we have \[\langle u, v\rangle_{H^\alpha(\M, \mu)} = \langle \mS_{n,\alpha} u, \mS_{n,\alpha}  v \rangle_{H^\alpha(\M_n, \mu_n)}.\]
    \item \label{spectral-h-alpha-preservation-norm} It holds for all $u\in H^\alpha(\M, \mu)$ that $\norm{\mS_{n,\alpha} u}_{\alpha,n}\leq \norm{u}_{H^\alpha(\M, \mu)}$. In particular, the extension map $\mS^*_{n,\alpha}$ preserves the $H^\alpha$-norm for all $u\in L^2(\M_n, \mu_n)$ with $\langle u, \phi^\bn_k\rangle_{L^2(\M_n, \mu_n)}=0$ for all  $k>K(n)$, and also $\mS_{n,\alpha}$ preserves the $H^\alpha$-norm for all $u\in L^2(\M, \mu)$ with $\langle u, \phi_k\rangle_{L^2}=0$  for all $k>K(n)$.
    \item \label{subspace-projection-h-alpha} It holds that $\mS^*_{n,\alpha}\mS_{n,\alpha} v = \sum\limits_{k=1}^{K(n)} \langle v, \varphi_k\rangle_{L^2(\M, \mu)} \varphi_k$  for every $v\in L^2(\M, \mu)$, where $\{\varphi_k\}_{k\leq K(n)}$ is any orthonormal basis of $\bigoplus_{k\leq K(n)} E(\lambda_k)$. In particular $\mS_{n,\alpha}^*\mS_{n,\alpha} v = \mS_n^*\mS_n v$ for every $v\in L^2(\M, \mu)$.
    \item \label{spectral-h-alpha-contraction} For any $v\in L^2(\M, \mu)$ we have $\norm{\mS_{n,\alpha}^*\mS_{n,\alpha} v}_{L^2(\M, \mu)}\leq\norm{v}_{L^2(\M, \mu)}$.
\end{lemlist}
\end{lemma}

\begin{proof}
Again \textup{(i)}, \textup{(ii)}, and the first part of \textup{(iii)} follow immediately from the definitions and imply the second part of \textup{(iii)}. Furthermore \textup{(v)} follows from \textup{(iv)}.

For \textup{(iv)}, writing out the definitions reads
    \begin{align*}
    \mS^*_{n,\alpha}\mS_{n,\alpha} v &= \sum\limits_{k=1}^{ K(n)} \left(\frac{1+(\lambda_k^\bn)^{1/2}}{1+\lambda_k^{1/2}}\right)^{\alpha}\langle \mS_{n,\alpha} v, \phi^\bn_k\rangle_{L^2(\M_n, \mu_n)} \phi_k^{\cn} 
    \\&= \sum\limits_{k=1}^{ K(n)} \left(\frac{1+(\lambda_k^\bn)^{1/2}}{1+\sqrt{\lambda_k}}\right)^{\alpha} \Big\langle \sum\limits_{j=1}^{ K(n)} \left(\frac{1+\lambda_j^{1/2}}{1+(\lambda_j^\bn)^{1/2}}\right)^{\alpha}\!\!\langle v, \phi_j^{\cn}\rangle_{L^2(\M, \mu)}\phi_j^\bn, \phi^\bn_k\Big\rangle_{L^2(\M_n, \mu_n)} \phi_k^{\cn} \\
    &= \sum\limits_{k=1}^{ K(n)} \left(\frac{1+(\lambda_k^\bn)^{1/2}}{1+\lambda_k^{1/2}}\right)^{\alpha} \sum\limits_{j=1}^{ K(n)}  \left(\frac{1+\lambda_k^{1/2}}{1+(\lambda_k^\bn)^{1/2}}\right)^{\alpha} \!\!\langle v, \phi_j^{\cn}\rangle_{L^2(\M, \mu)} \langle \phi_j^\bn, \phi^\bn_k\rangle_{L^2(\M_n, \mu_n)} \phi_k^{\cn} 
    \\&= \sum\limits_{k=1}^{ K(n)} \left(\frac{1+(\lambda_k^\bn)^{1/2}}{1+\lambda_k^{1/2}}\right)^{\alpha}\left(\frac{1+(\lambda_k^\bn)^{1/2}}{1+\lambda_k^{1/2}}\right)^{-\alpha} \!\!\!\langle v, \phi_k^{\cn}\rangle_{L^2(\M, \mu)} \phi_k^{\cn} \\
    &= \sum\limits_{k=1}^{ K(n)} \langle v, \phi_k^{\cn}\rangle_{L^2(\M, \mu)} \phi_k^{\cn}= \Pr\left(\bigoplus_{k=1}^{ K(n)} E(\lambda_k), v\right),
    \end{align*}
    by \eqref{eigenspace-projection-basis-free}, which yields the claim.
\end{proof}

Thus, the spectral projection $S^{n,\alpha}$ projects the $ K(n)$-dimensional subspace $\bigcup_{k=1}^{K(n)} E(\lambda_k)$ of $H^\alpha(\M, \mu)$ on the $ K(n)$-dimensional subspace $\bigcup_{k=1}^{K(n)} E(\lambda_k^\bn)$ of  $H^\alpha(\M_n, \mu_n)$ and vice versa for the spectral extension operator $S^*_{n, \alpha}$.

\begin{lemma}\label{convergence-extended-discretization-sn}
    For any $v\in L^2(\M, \mu)$ it holds that $\norm{v - \mS_n^*\mS_n v}_{L^2(\M, \mu)}\xrightarrow{n\to\infty} 0,$ as well as $\norm{v - \mS_{n,\alpha}^*\mS_{n,\alpha} v}_{L^2(\M, \mu)}\xrightarrow{n\to\infty} 0.$
\end{lemma}

\begin{proof}
    Using \Cref{subspace-projection-h-alpha} we observe that 
    \[v- \mS_{n,\alpha}^*\mS_{n,\alpha} v = v- \mS_n^*\mS_n v = \sum\limits_{k=1}^\infty \langle v, \varphi_k\rangle_{L^2(\M, \mu)} \varphi_k - \sum\limits_{k=1}^{ K(n)} \langle v, \varphi_k\rangle_{L^2(\M, \mu)} \varphi_k = \sum\limits_{k=K(n)+1}^\infty \langle v, \varphi_k\rangle_{L^2(\M, \mu)} \varphi_k.\] 
    Thus, since $v\in L^2(\M, \mu)$ and $K(n)\xrightarrow{n\to\infty}\infty$ we have 
    \[\norm{\mS_{n,\alpha}^*\mS_{n,\alpha} v - v}_{L^2(\M, \mu)} = \norm{\mS_n^*\mS_n v - v}_{L^2(\M, \mu)} = \norm{\sum\limits_{k=K(n)+1}^\infty \langle v, \varphi_k\rangle_{L^2(\M, \mu)}\varphi_k}_{L^2(\M, \mu)}\xrightarrow{n\to\infty } 0.\qedhere\]
\end{proof}

\begin{defn}\label{consistent-discretization-operators}
    We say a sequence of (linear) discretization operators $\mR_n: L^2(\M, \mu)\to L^2(\M_n, \mu_n)$, induces graph signals consistent across resolutions (see \cref{consistent-signals}), if  there is a $C_{\mR}\geq 0$ such that $\norm{\mR_n u}_{L^2}\leq C_{\mR}\norm{u}_{L^2}$ for all $n\in\N$  and for all $u\in L^2(\M, \mu)$ it holds that
        \[d_{TL^2}((\mu, u), (\mu_n, \mR_n u)) \xrightarrow{n\to\infty} 0,\] or equivalently by \eqref{tl2-spatial-extension-convergence}
        \begin{equation}\label{consistent-u}
            \norm{u-\mP_n^*u_n}_{L^2} = \norm{u-\mP_n^*\mR_n u}_{L^2}\xrightarrow{n\to\infty}0.
        \end{equation}
\end{defn}
In the following let $\mR_n:L^2(\M, \mu)\to L^2(\M_n, \mu_n)$, $n\in\N$ be such a sequence of discretization operators inducing graph signals consistent across resolutions.

\begin{rem}\label{tl2-compatibility-spatial-spectral-discretization-maps}
Since \[\norm{\mP_n u}_{L^2(\M_n)}^2 = \langle \mP_n u, \mP_n u\rangle_{L^2(\M_n)} = \langle u, \mP_n^*\mP_n u\rangle_{L^2(\M)} \leq \norm{u}_{L^2(\M)}\norm{\mP_n^*\mP_n u}_{L^2(\M)} \leq \norm{u}_{L^2(\M)}^2,\] the family $\{\mP_{n}\}_{n\in\N}$ is uniformly bounded and due to 
\Cref{convergence-extended-discretization-pn} the condition \eqref{consistent-u} is satisfied, hence it holds that the spatial discretization maps $\mP_n$ are examples of such discretization operators that induce graph signals which are consistent across resolutions. Uniform boundedness of the family $\{\mS_{n,\alpha}\}_{n\in\N}$ (\Cref{spectral-h-alpha-preservation-norm}) and the following \Cref{convergence_spatial_spectral_discretization}, respective \Cref{convergence_spatial_h_alpha_spectral_discretization} in combination with \Cref{convergence-extended-discretization-sn} show that also the spectral discretization maps satisfy \Cref{consistent-discretization-operators}.
\end{rem}

\begin{lemma}\label{convergence-extended-discretization-sn-pn-convolution}
    Let $\mR_n: L^2(\M, \mu)\to L^2(\M_n, \mu_n)$ satisfy \cref{{consistent-discretization-operators}}. Then, for any $b,v\in L^2(\M, \mu)$ it holds that 
    \[\norm{b*v - (\mS^*_n\mS_n b)*(\mP^*_n\mR_n v)}_{L^2(\M, \mu)}\leq  \norm{b}_{L^2}\big(\norm{v-\mS^*_n\mS_n v}_{L^2} + \norm{v-\mP_n^*\mR_n v}_{L^2}\big)\xrightarrow{n\to\infty} 0.\]
\end{lemma}

\begin{proof}
Denote by $\{\varphi_k\}_k$ an arbitrary choice of orthonormal eigenfunctions of the Laplace-Beltrami operator.
Then, bilinearity of the convolution operation and orthogonality of the eigenfunctions yield \begin{align*}b*v - (\mS^*_n\mS_n b)*(\mP^*_n\mR_n v)&= \sum\limits_{k=1}^\infty \langle b, \varphi_k\rangle \langle v, \varphi_k\rangle \varphi_k - \sum\limits_{k=1}^\infty \langle \mS_n^*\mS_n b, \varphi_k\rangle \langle \mP_n^*\mR_n v, \varphi_k \rangle \varphi_k \\
& =  \sum\limits_{k=1}^\infty \langle b-\mS_n^*\mS_n b, \varphi_k \rangle \langle v, \varphi_k\rangle \varphi_k + \sum\limits_{k=1}^\infty \langle \mS_n^*\mS_n b, \varphi_k \rangle \langle v-\mP_n^*\mR_n v , \varphi_k \rangle \varphi_k  \\
& =  \sum\limits_{k=K(n)+1}^\infty \langle b, \varphi_k \rangle \langle v, \varphi_k\rangle \varphi_k  + \sum\limits_{k=1}^{K(n)} \langle b, \varphi_k\rangle \langle v-\mP_n^*\mR_n v, \varphi_k  \rangle \varphi_k.  \end{align*}
Thus 
\begin{align*}&\norm{b*v - (\mS^*_n\mS_n b)*(\mP^*_n\mR_n v)}_{L^2}\\
&\hspace{48pt} \leq \norm{\sum\limits_{k=K(n)+1}^\infty \langle b, \varphi_k\rangle \langle v, \varphi_k\rangle \varphi_k}_{L^2} + \norm{\sum\limits_{k=1}^{K(n)} \langle b, \varphi_k\rangle \langle v-\mP_n^*\mR_n v, \varphi_k \rangle \varphi_k}_{L^2}.
\end{align*}
Now the first term is precisely \[\norm{b*v-\mS^*_n\mS_n(b*v)}_{L^2} = \norm{b*(v-\mS_n^*\mS_n v)}_{L^2}\leq \norm{b}_{L^2}\norm{v- \mS_n^*\mS_n v}_{L^2},\] due to the orthonormality of the eigenfunctions and the Cauchy-Schwarz inequality and goes to $0$ by \Cref{convergence-extended-discretization-sn}, and for the second term it holds that
\begin{align*}\norm{\sum\limits_{k=1}^{K(n)} \langle b, \varphi_k\rangle \langle v-\mP_n^*\mR_n v, \varphi_k \rangle \varphi_k}_{L^2}^2 &= \sum\limits_{k=1}^{K(n)} \langle b, \varphi_k\rangle^2\langle v-\mP_n^*\mR_n v, \varphi_k \rangle^2 \\& \leq \sum\limits_{k=1}^{K(n)} \langle b, \varphi_k\rangle^2 \sum\limits_{k=1}^{K(n)}\langle v-\mP_n^*\mR_n v, \varphi_k \rangle^2
\\&\leq \norm{\mS_n^*\mS_n b}_{L^2}^2\norm{v-\mP_n^*\mR_n v }_{L^2}^2 \leq \norm{b}_{L^2}^2\norm{v-\mP_n^*\mR_n v }_{L^2}^2\xrightarrow{n\to\infty}0,
\end{align*}
by definition of $\mR_n$.
\end{proof}

\begin{lemma}\label{convergence_spatial_spectral_discretization}
For all $v\in L^2(\M, \mu)$ we have that \[\norm{v-\mP_n^*\mS_n v}_{L^2(\M, \mu)}\leq \sqrt{K(n)}\norm{v}_{L^2}\delta^\phi_n \xrightarrow{n\to\infty} 0,\]where $\delta^\phi_n$ was defined in \eqref{def-delta-phi-n}.
\end{lemma}
\begin{proof}
    We can expand in any $L^2(\M,\mu)$-orthonormal eigenbasis $\{\varphi_k\}_k$, indexed corresponding to the eigenvalues. Then, using  \eqref{eigenspace-projection-basis-free} we can write
    \begin{align*}\norm{v-\mP_n^*\mS_n v}_{L^2(\M, \mu)}
    &=\norm{v- \sum\limits_{k=1}^{ K(n)}\langle v, \phi_k^{\cn}\rangle \mP_n^*(\phi^\bn_k)}_{L^2(\M, \mu)}
    \\&= \norm{\sum_{k=1}^\infty \langle v, \varphi_k\rangle \varphi_k - \sum\limits_{k=1}^{ K(n)}\langle v, \phi_k^\cn\rangle \mP_n^*(\phi^\bn_k)}_{L^2(\M, \mu)}
    \\&=\norm{\sum_{k=1}^{ K(n)} \langle v, \phi_k^{\cn}\rangle (\phi_k^{\cn} - \mP_n^*(\phi^\bn_k)) + \sum_{k=K(n)+1}^\infty \langle v, \varphi_k\rangle \varphi_k}_{L^2(\M, \mu)}
    \\&\leq \norm{\sum\limits_{k=1}^{ K(n)}\langle v, \phi_k^{\cn}\rangle_{L^2}(\phi_k^{\cn}-\mP_n^*(\phi^\bn_k))}_{L^2(\M, \mu)} +\norm{v-\mS_n^*\mS_n v}_{L^2(\M, \mu)}.\end{align*}
    By \Cref{convergence-extended-discretization-sn} the last term goes to $0$ as $n\to\infty$.
    Furthermore the Cauchy-Schwarz inequality and \Cref{uniform-bounds-spectra} with \eqref{convergence-relative-error-bound-delta-max-spectral-gap-kn} yield
        \begin{align*}
        \norm{\sum\limits_{k=1}^{ K(n)}\langle v, \phi_k^{\cn}\rangle_{L^2}}(\phi_k^{\cn}-\mP_n^*(\phi^\bn_k))^2_{L^2(\M,\mu)}&\leq \sum_{k=1}^{ K(n)}\langle v, \phi_k^{\cn}\rangle_{L^2}^2\sum_{k=1}^{ K(n)}\norm{\phi_k^{\cn}-\mP_n^*(\phi^\bn_k)}_{L^2}^2 
       \\&\leq \norm{v}_{L^2}^2 \sum_{k=1}^{ K(n)} (\delta^\phi_n)^2  \leq \norm{v}^2_{L^2}K(n)(\delta^\phi_n)^2 \xrightarrow{n\to\infty} 0,
        \end{align*}
    which yields the claim.
\end{proof}

\begin{lemma}\label{convergence-L2-distance-L2-H-alpha-isomorphisms}
    Let $\alpha\in(0,1]$. Then, it holds for all $v\in H^\alpha(\M, \mu)$ that $\norm{\mS_n v - \mS_{n, \alpha} v}_{L^2}\xrightarrow{n\to\infty}0$.
 \[\norm{\mS_n v - \mS_{n, \alpha} v}_{L^2(\M_n, \mu_n)}\leq C \delta(\eps_n, h_n, \lambda_{ K(n)})^{\alpha/2} \norm{v}_{H^\alpha(\M)}\xrightarrow{n\to\infty}0.\]
 \end{lemma}

\begin{proof}
    We have
    \begin{align*}
        \norm{\mS_n v - \mS_{n, \alpha} v}_{L^2(\M_n)}^2 &= \norm{\sum_{k=1}^{ K(n)}\left(1-\left(\frac{1+\lambda_k^{1/2}}{1+(\lambda^\bn_k)^{1/2}}\right)^\alpha\right)\langle v, \phi_k^{\cn}\rangle_{L^2}\,\phi^\bn_k}^2_{L^2(\M_n)}  \\
        &= \sum_{k=1}^{ K(n)}\left|1-\left(\frac{1+\lambda_k^{1/2}}{1+(\lambda^\bn_k)^{1/2}}\right)^\alpha\right|^2\langle v, \phi_k^{\cn}\rangle_{L^2(\M)}^2.
    \end{align*}
    It holds for all $n\in\N$ that $\lambda^\bn_1=\lambda_1=0$, and since $\M$ is compact and connected, $\lambda_k> 0$ for all $k\geq 2$ (see e.g. \cite[p.~31]{Rosenberg:1997}). 
    Now we can further estimate using $|x+y|^\alpha \leq |x|^\alpha+|y|^\alpha$ (and as a consequence $\left| |x|^\alpha -|y|^\alpha\right|\leq |x-y|^\alpha$) for all $x,y \in \R$ and $\alpha\in (0,1]$ and \Cref{uniform-bounds-spectra}, in particular \eqref{eigenvalues-ratio-uniform-bound}, that for all $2\leq k\leq K(n)$ \begin{align*}
        \left|1^\alpha-\left(\frac{1+\lambda_k^{1/2}}{1+(\lambda^\bn_k)^{1/2}}\right)^\alpha\right|^2 
        &\leq \left|1-\frac{1+\lambda_k^{1/2}}{1+(\lambda^\bn_k)^{1/2}}\right|^{2\alpha} 
        = \left|\frac{(\lambda^\bn_k)^{1/2} -\lambda_k^{1/2}}{1+(\lambda^\bn_k)^{1/2}}\right|^{2\alpha} \\
        & \leq \left(\frac{\left|\lambda^\bn_k -\lambda_k\right|^{1/2}}{\left|1+(\lambda^\bn_k)\right|^{1/2}}\right)^{2\alpha}
        = \left|\frac{\lambda^\bn_k -\lambda_k}{1+(\lambda^\bn_k)}\right|^{\alpha} 
        \leq \left|\lambda^\bn_k -\lambda_k\right|^{\alpha}\\        
        &= \left(\frac{\left|\lambda^\bn_k -\lambda_k\right|}{1+\lambda_k}\right)^{\alpha}(1+\lambda_k)^\alpha \leq \left(\frac{\left|\lambda^\bn_k -\lambda_k\right|}{\lambda_k}\right)^{\alpha}(1+\lambda_k)^\alpha \\
        &\leq \left(C \delta(\eps_n, h_n, \lambda_k)\right)^\alpha (1+\lambda_k)^\alpha 
        \leq C\delta(\eps_n, h_n, \lambda_{K(n)})^\alpha (1+\lambda_k)^\alpha.
    \end{align*}
    Putting this together with the above we have for all $n\in\N$
    \begin{align*}
        \norm{\mS_n v - \mS_{n, \alpha} v}_{L^2(\M_n)}^2 &= \sum_{k=1}^{ K(n)}\left|1-\left(\frac{1+\lambda_k^{1/2}}{1+(\lambda^\bn_k)^{1/2}}\right)^\alpha\right|^2\langle v, \phi_k^{\cn}\rangle_{L^2(\M)}^2 \\
        &\leq \sum_{k=2}^{ K(n)}C\left( \delta(\eps_n, h_n, \lambda_{ K(n)})\right)^\alpha (1+\lambda_k)^\alpha\langle v, \phi_k^{\cn}\rangle_{L^2(\M)} ^2\\
        &\leq C\delta(\eps_n, h_n, \lambda_{ K(n)})^\alpha \sum_{k=1}^{ K(n)} (1+\lambda_k)^\alpha\langle v, \phi_k^{\cn}\rangle_{L^2(\M)}^2\\
        &\leq C \delta(\eps_n, h_n, \lambda_{ K(n)})^\alpha \norm{v}_{H^\alpha(\M)}^2.
    \end{align*}
    So \eqref{convergence-relative-error-bound-delta} yields that this converges to $0$ as $n\to\infty$.
\end{proof}

\begin{lemma}\label{convergence_spatial_h_alpha_spectral_discretization}
It holds for all $v\in H^\alpha(\M, \mu)$, that $\norm{\mS_{n,\alpha}^*\mS_{n,\alpha}v-\mP_n^*\mS_{n,\alpha} v}_{L^2(\M, \mu)}\to 0$ as $n\to\infty$.
\end{lemma}

\begin{proof}
    Using \Cref{subspace-projection-h-alpha} and \Cref{adjoint-l2-norm-preserving} we obtain
    \begin{align*}
        \norm{\mS_{n,\alpha}^*\mS_{n,\alpha}v-\mP_n^*\mS_{n,\alpha} v}_{L^2(\M, \mu)} 
        &= \norm{\mS_{n}^*\mS_{n}v-\mP_n^*\mS_{n,\alpha} v}_{L^2(\M, \mu)} 
        \\&\leq \norm{\mS_{n}^*\mS_{n}v-\mP_n^*\mS_{n} v}_{L^2(\M, \mu)} + \norm{\mP_n^*\mS_{n} v - \mP_n^*\mS_{n, \alpha} v}_{L^2(\M, \mu)}
        \\&= \norm{\mS_{n}^*\mS_{n}v-\mP_n^*\mS_{n} v}_{L^2(\M, \mu)} + \norm{\mS_{n} v - \mS_{n, \alpha} v}_{L^2(\M_n, \mu_n)} \\
        &\leq \norm{\mS_n^*\mS_n v - v}_{L^2} + \norm{v- \mP_n^*\mS_n v}_{L^2}+\norm{\mS_{n} v - \mS_{n, \alpha} v}_{L^2}.
    \end{align*}
    where all three terms go to $0$ as $n\to\infty$ by \Cref{convergence-extended-discretization-sn}, \Cref{convergence_spatial_spectral_discretization} and \Cref{convergence-L2-distance-L2-H-alpha-isomorphisms}.
\end{proof}

With the preceding convergence results for the discretization and extension maps, we can now compare a lifted discrete and a continuum convolution expression with appropriate discretizations of the convolution parameter $b$ and the signal. Using a spectral discretization on $b$ is essential, since it allows us to apply the uniform error bounds on the graph-eigenfunctions from \Cref{uniform-bounds-spectra}. This yields the following estimate:
\begin{lemma}\label{convergence-discrete-continuum-convolution-operation}
        Let $\mR_n: L^2(\M, \mu)\to L^2(\M_n, \mu_n)$ satisfy \cref{{consistent-discretization-operators}}. Then, for any $b,v\in L^2(\M, \mu)$ it holds that \[\norm{\mP_n^*(\mS_n b*_n\mR_n v) - (\mS^*_n\mS_n b)*(\mP^*_n\mR_n v)}_{L^2(\M, \mu)}\leq 2 C_{\mR_n}\delta^\phi_n \norm{b}_{L^2}\norm{u}_{L^2} \xrightarrow{n\to\infty}0.\]
\end{lemma}

\begin{proof}
    Again we expand in an arbitrary $L^2(\M,\mu)$ orthonormal eigenbasis $\{\varphi_k\}_k$, indexed corresponding to the eigenvalues. Using \Cref{adjoint-l2-innerproduct-preserving}, the extended discrete convolution expression reads
    \begin{align*}
    \mP_n^*(\mS_n b *_n \mR_n v) &= \sum\limits_{k=1}^{K(n)}\langle \mS_n b, \phi^\bn_k \rangle_{L^2(\M_n, \mu_n)} \langle \phi^\bn_k, \mR_n v \rangle_{L^2(\M_n, \mu_n)}\mP^*_n \phi^\bn_k \\
    &= \sum\limits_{k=1}^{K(n)}\langle b, \phi_k^{\cn} \rangle_{L^2(\M, \mu)} \langle \mP_n^* \mR_n v, \mP_n^* \phi^\bn_k \rangle_{L^2(\M, \mu)}\mP^*_n \phi^\bn_k,
    \end{align*}
    while using  \eqref{eigenspace-projection-basis-free} and \Cref{subspace-projection-h-alpha} and the orthonormality of $\{\varphi_k\}_k$, the continuum convolution expression reads
    \begin{align*}
    (\mS^*_n \mS_n b)*(\mP^*_n\mR_n v) &= \sum\limits_{k=1}^\infty\langle\mS^*_n\mS_n b, \varphi_k \rangle_{L^2(\M, \mu)} \langle \mP^*_n\mR_n v, \varphi_k \rangle_{L^2(\M, \mu)}\varphi_k \\
    &= \sum\limits_{k=1}^{\infty}\big\langle \sum\limits_{j=1}^{K(n)}\langle b, \varphi_j\rangle_{L^2(\M, \mu)} \varphi_j, \varphi_k \big\rangle_{L^2(\M, \mu)} \langle \mP^*_n\mR_n v, \varphi_k \rangle_{L^2(\M, \mu)}\varphi_k\\
    &= \sum\limits_{k=1}^{K(n)}\langle b, \varphi_k \rangle_{L^2(\M, \mu)} \langle \mP^*_n\mR_n v, \varphi_k \rangle_{L^2(\M, \mu)}\varphi_k\\
    &= \sum\limits_{k=1}^{K(n)}\langle b, \phi_k^{\cn} \rangle_{L^2(\M, \mu)} \langle \mP^*_n\mR_n v, \phi_k^{\cn} \rangle_{L^2(\M, \mu)}\phi_k^{\cn}.
    \end{align*}
    Therefore
    \begin{align*}&\norm{\mP_n^*(\mS_n b *_n \mR_n v) - (\mS^*_n \mS_n b)*(\mP^*_n\mR_n v)}_{L^2} \\
    &\hspace{24pt}= \norm{\sum\limits_{k=1}^{K(n)} \langle b, \phi_k^{\cn}\rangle_{L^2}\left(\langle \mP_n^* \mR_n v, \mP_n^* \phi^\bn_k \rangle_{L^2}\mP^*_n \phi^\bn_k- \langle \mP^*_n\mR_n v, \phi_k^{\cn} \rangle_{L^2}\,\phi_k^{\cn} \right)}_{L^2} \\
    &\hspace{24pt}= \norm{\sum\limits_{k=1}^{K(n)} \langle b, \phi_k^{\cn}\rangle_{L^2}\left(\langle \mP_n^* \mR_n v, \mP_n^* \phi^\bn_k -\phi_k^{\cn} \rangle_{L^2}\,\phi_k^{\cn}+ \langle \mP^*_n\mR_n v, \mP_n^*\phi^\bn_k \rangle_{L^2}(\mP_n^*\phi^\bn_k -\phi_k^{\cn}) \right)}_{L^2}\\
    &\hspace{24pt}\leq \norm{\sum\limits_{k=1}^{K(n)} \langle b, \phi_k^{\cn}\rangle_{L^2}\langle \mP_n^* \mR_n v, \mP_n^* \phi^\bn_k -\phi_k^{\cn} \rangle_{L^2(\M, \mu)} \phi_k^{\cn}}_{L^2} \\
    &\hspace{48pt}+ \norm{\sum\limits_{k=1}^{K(n)} \langle b, \phi_k^{\cn}\rangle_{L^2} \langle \mP^*_n\mR_n v, \mP_n^*\phi^\bn_k \rangle_{L^2}(\mP_n^*\phi^\bn_k -\phi_k^{\cn})}_{L^2}.
    \end{align*}

    Now for the first term using \Cref{uniform-bounds-spectra}, \Cref{subspace-projection-h-alpha}, \Cref{spectral-h-alpha-contraction}, \Cref{adjoint-l2-norm-preserving}, and \Cref{consistent-discretization-operators} we get that 
    \begin{align*}
        \norm{\sum\limits_{k=1}^{K(n)} \langle b, \phi_k^{\cn}\rangle_{L^2}\langle \mP_n^* \mR_n v, \mP_n^* \phi^\bn_k -\phi_k^{\cn} \rangle_{L^2}\,\phi_k}_{L^2}^2 
        &=\sum\limits_{k=1}^{K(n)} \langle b, \phi_k\rangle_{L^2}^2\langle \mP_n^* \mR_n v, \mP_n^* \phi^\bn_k -\phi_k^{\cn} \rangle_{L^2}^2 \\
        &\leq \sum\limits_{k=1}^{K(n)} \langle b, \phi_k^{\cn}\rangle_{L^2}^2\norm{\mP_n^* \mR_n v}_{L^2}^2 \norm{\mP_n^* \phi^\bn_k -\phi_k^{\cn}}_{L^2}^2\\
        &\leq \norm{\mP_n^*\mR_n v}_{L^2}^2(\delta^\phi_n)^2\norm{\mS_n^*\mS_n b}_{L^2}^2\\ 
        &\leq C_{\mR}^2 \norm{v}_{L^2}^2(\delta^\phi_n)^2\norm{b}^2_{L^2}\xrightarrow{n\to\infty}0,
    \end{align*}
    while for the second term we get using the same results that
    \begin{align*}
        &\norm{\sum\limits_{k=1}^{K(n)} \langle b, \phi_k^{\cn}\rangle_{L^2} \langle \mP^*_n\mR_n v, \mP_n^*\phi^\bn_k \rangle_{L^2}(\mP_n^*\phi^\bn_k -\phi_k^{\cn})}_{L^2} \\
        &\hspace{50pt}\leq \sum\limits_{k=1}^{K(n)} |\langle b, \phi_k^{\cn}\rangle_{L^2} |\,\big|\langle \mP^*_n\mR_n v, \mP_n^*\phi^\bn_k \rangle_{L^2}\big|\norm{\mP_n^*\phi^\bn_k -\phi_k^{\cn}}_{L^2}
        \\&\hspace{50pt}\leq \delta^\phi_n \sum\limits_{k=1}^{K(n)} |\langle b, \phi_k^{\cn}\rangle_{L^2}|\,\big|\langle \mR_n v, \phi^\bn_k \rangle_{L^2(\M_n, \mu_n)}\big|
        \\&\hspace{50pt}\leq \delta^\phi_n \norm{\mS_n^*\mS_n b}_{L^2(\M, \mu)}\norm{\mR_n v}_{L^2(\M_n, \mu_n)}
        \\&\hspace{50pt}\leq C_{\mR}\delta^\phi_n \norm{b}_{L^2}\norm{v}_{L^2}\xrightarrow{n\to\infty}0.
    \end{align*}

    So, with all the above, we obtain
     \[\norm{\mP_n^*(\mS_n b *_n \mR_n v) - (\mS^*_n \mS_n b)*(\mP^*_n\mR_n v)}_{L^2}\leq 2 C_{\mR}\delta^\phi_n \norm{b}_{L^2}\norm{v}_{L^2}\xrightarrow{n\to\infty}0.\qedhere\]
\end{proof}

\subsection{Interplay with neural response maps}\label{interplay-projection-neural-response}
These discretization maps can be naturally used in constructing discretization maps and extension maps between the (ambient) continuum parameter space $\Theta$ and the (ambient) discrete parameter spaces $\Theta_n$, $n\in \N$, respectively vice versa:
\begin{equation}\label{def-parameter-discretization}
    \mQ_{n, \alpha}: (L^2(\M, \mu))^3\to (L^2(\M_n, \mu_n))^3, \quad (a,b,c)\mapsto (\mS_{n, \alpha} a, \mS_n b, \mS_{n,\alpha} c),
\end{equation} and 
\begin{equation}\label{def-parameter-extension}
    \mQ^*_{n, \alpha}: (L^2(\M_n, \mu_n))^3\to (L^2(\M, \mu))^3, \quad (a,b,c)\mapsto (\mS^*_{n, \alpha} a, \mS^*_n b, \mS^*_{n,\alpha} c).
\end{equation}

\begin{rem}
    So we observe that the discretization and extension maps just map the continuum ambient space to the subspace spanned by the first $K(n)$ eigenvectors of the $n$-th graph Laplacian of the respective discrete space and vice versa by definition of the spectral projections and \Cref{spectral-h-alpha-preservation-norm} (which we denote by adding the subscript ${K(n)}$ to a space), i.e. we have that:
    \begin{equation}\label{q-nalpha-onto}
    \mQ_{n,\alpha}( H^\alpha(\M)\times L^2(\M)\times H^\alpha(\M))=H^\alpha_{K(n)}(\M_n)\times L^2_{K(n)}(\M_n)\times H^\alpha_{K(n)}(\M_n),\end{equation}
    \begin{equation}\label{q-nalpha-*-subspace}
        \mQ_{n,\alpha}^*(H^\alpha(\M_n)\times L^2(\M_n)\times H^\alpha(\M_n)) =  H^\alpha_{K(n)}(\M)\times L^2_{K(n)}(\M)\times H^\alpha_{K(n)}(\M),
    \end{equation}
    and putting them together yields the relations
    \begin{equation}\label{q-n-alpha-onto-subspace-restriction-discretize-extend}
    \mQ_{n,\alpha}^*\mQ_{n,\alpha}( H^\alpha(\M)\times L^2(\M)\times H^\alpha(\M))=H^\alpha_{K(n)}(\M)\times L^2_{K(n)}(\M)\times H^\alpha_{K(n)}(\M),\end{equation} 
    \begin{equation}\label{q-n-alpha-onto-subspace-restriction-extend-discretize}
    \mQ_{n,\alpha}( H^\alpha_{K(n)}(\M)\times L^2_{K(n)}(\M)\times H^\alpha_{K(n)}(\M))=H^\alpha_{K(n)}(\M_n)\times L^2_{K(n)}(\M_n)\times H^\alpha_{K(n)}(\M_n),\end{equation} 
    More precisely, taking norm preservation into account we get \begin{equation}\label{q-nalpha-ext-discr-onto-subspace-restriction}
    \mQ_{n,\alpha}^*\mQ_{n,\alpha}(\Theta)= \left(B_{H^\alpha}\cap H^\alpha_{K(n)}(\M)\right)\times \left(B_{{L^2}}\cap L^2_{K(n)}(\M)\right)\times \left( B_{H^\alpha}\cap H^\alpha_{K(n)}(\M)\right),\end{equation} 
    and\vskip -2em
    \begin{equation}\label{q-nalpha-discr-ext-onto-subspace-restriction}
    \mQ_{n,\alpha}\mQ_{n,\alpha}^*(\Theta_n)= \Theta_n.\end{equation} 
\end{rem}
So we have that $\mQ_{n,\alpha}^*$ and $\mQ_{n,\alpha}$ are homeomorphisms between the $K(n)$-dimensional subspaces of the parameter spaces equipped with their respective topology, obtained by the spectral cutoff of the respective ambient parameter spaces $H^\alpha\times L^2\times H^\alpha$:
\begin{lemma}\label{q-nalpha-homeomorphism}
    For the parameter discretization and extension maps it holds that $\mQ_{n, \alpha}(\Theta)=\Theta_n$ and \[\mQ_{n, \alpha}^*: \Theta_n \to (\mQ_{n, \alpha}^*(\Theta_n), w)\] is a homeomorphism with the inverse being $\mQ_{n, \alpha}: (\mQ_{n, \alpha}^*(\Theta_n), w) \to \Theta_n $.
\end{lemma}
\begin{proof}
    The first claim follows from the observation \eqref{q-n-alpha-onto-subspace-restriction-extend-discretize} that the subspaces spanned by the first $K(n)$ eigenvectors or eigenfunctions agree under the respective projection due to \Cref{subspace-projection-h-alpha}. 
    Regarding the second claim, first note that $\mQ_{n, \alpha}$ and its adjoint $\mQ_{n, \alpha}^*$ are linear bounded operators.
    Furthermore using \eqref{q-nalpha-discr-ext-onto-subspace-restriction}
    we see that $\mQ_{n, \alpha}^*: \Theta_n \to (\mQ_{n, \alpha}^*(\Theta_n), w)$ is bijective with inverse $\mQ_{n, \alpha}: (\mQ_{n, \alpha}^*(\Theta_n), w) \to \Theta_n $.
    Finally $\mQ_{n, \alpha}^*$ and $\mQ_{n, \alpha}$ are continuous w.r.t the appropriate topologies, because as bounded linear operators they are norm-norm-continuous, but also weak-weak-continuous, which yields the second claim. 
\end{proof}

\begin{rem}
    An immediate consequence of \Cref{q-nalpha-homeomorphism} is that \begin{equation}\label{q-nalpha-continuous-fcts-isomorphism}
        C(\mQ_{n, \alpha}^*\mQ_{n, \alpha}\Theta, w)= C(\mQ_{n, \alpha}^*\Theta_n, w)\cong C(\Theta_n), \end{equation} where the spaces of continuous functions are equipped with the respective $\sup$-norm and \begin{equation}\label{q-nalpha-radon-measures-isomorphism}
        \M(\mQ_{n, \alpha}^*\mQ_{n, \alpha}\Theta) = \M(\mQ_{n, \alpha}^*\Theta_n) =(\mQ_{n, \alpha}^*)_\#\cM(\Theta_n)\cong \cM(\Theta_n),\end{equation}
    where the spaces of finite signed Radon measures are equipped with the respective $TV$-norm, and the Radon measures on $\mQ_{n, \alpha}^*\mQ_{n, \alpha}\Theta$ and $\mQ_{n,\alpha}^*\Theta_n$ are understood with respect to its weak topology.
\end{rem}

    \begin{rem}\label{estimate-neuron-different-signal-and-parameter-inputs}
    For $u, u'\in L^2(\M, \mu)$ and $\theta, \theta'\in \big(L^2(\M, \mu)\big)^3$, where we write $\theta=(a,b,c)$ and $\theta' = (a', b',c')$, we get using Cauchy-Schwarz, \eqref{lipschitz-property-sigma} and \eqref{linear-growth-sigma}
    \begin{align*}
        |\psi(u, \theta) - \psi(u', \theta')| &= \big|\langle a,\sigma( b * u +  c)\rangle - \langle  a',\sigma( b' *  u' +  c')\rangle\big| \\
& \leq \norm{ (a-a')}_{L^2}C_\sigma\big(1+\norm{ b*  u}_{L^2} + \norm{c}_{L^2}\big) \\& \qquad + \norm{ a'}_{L^2}L_\sigma\big(\norm{ b*u - b'*u'}_{L^2} + \norm{ c-c'}_{L^2}\big).
    \end{align*}
    \end{rem}

\begin{lemma}\label{convergence-signal-param-spatial-h-alpha-discretization}
    Let $\mR_n: L^2(\M, \mu)\to L^2(\M_n, \mu_n)$ satisfy \cref{{consistent-discretization-operators}}, $u\in L^2(\M, \mu)$ and $\theta\in \Theta$. Then it holds that 
    \[|\psi(\mP^*_n\mR_n u, \mQ^*_{n,\alpha}\mQ_{n,\alpha}\theta) - \psi(u, \theta) |\xrightarrow{n\to\infty}0.\]
\end{lemma}

\begin{proof}
    \Cref{estimate-neuron-different-signal-and-parameter-inputs} yields 
    \begin{align*}|\psi(u, \theta) - \psi(\mP^*_n\mR_n u, \mQ^*_{n,\alpha}\mQ_{n,\alpha}\theta) 
    &\leq 
    \norm{ (a-\mS^*_{n, \alpha}\mS_{n, \alpha} a)}_{L^2}C_\sigma\big(1+\norm{ b*  u}_{L^2} + \norm{c}_{L^2}\big) \\& \quad + \norm{ \mS^*_{n, \alpha}\mS_{n, \alpha} a -a }_{L^2}L_\sigma\big(\norm{ b*u - \mS^*_n\mS_n b *\mP^*_n\mR_n u}_{L^2}  \\&\hspace{150pt}+\norm{ c-\mS^*_{n,\alpha}\mS_{n,\alpha} c}_{L^2}\big)\\
    & \quad + \norm{ a }_{L^2}L_\sigma\big(\norm{ b*u - \mS^*_n\mS_n b *\mP^*_n\mR_n u}_{L^2} +\norm{ c-\mS^*_{n,\alpha}\mS_{n,\alpha} c}_{L^2}\big),
    \end{align*}    
    for which \Cref{convergence-extended-discretization-sn} and \Cref{convergence-extended-discretization-sn-pn-convolution} yield that the right-hand-side above goes to $0$ as $n\to\infty$.
\end{proof}

\begin{lemma}\label{convergence-neuron-discretization-pn-h-alpha}
    Let $\mR_n: L^2(\M, \mu)\to L^2(\M_n, \mu_n)$ satisfy \cref{{consistent-discretization-operators}}, $u\in L^2(\M, \mu)$ and $\theta\in \Theta$. Then it holds that 
    \[|\psi_n(\mR_n u, \mQ_{n,\alpha}\theta) - \psi(\mP^*_n\mR_n u, \mQ^*_{n,\alpha}\mQ_{n,\alpha}\theta)|\xrightarrow{n\to\infty}0.\]
\end{lemma}

\begin{proof}
    Let $a,b,c\in L^2(\M, \mu)$, $\norm{a}_{L^2}\leq 1$, $\norm{b}_{L^2}\leq 1$, $\norm{c}_{L^2}\leq 1$ such that $\theta=(a,b,c)$.

    We observe using \Cref{adjoint-l2-innerproduct-preserving}, Cauchy-Schwarz and \Cref{linear-growth-sigma} that
    \begin{align}
        |\psi_n&(\mR_n u, \mQ_{n,\alpha}\theta) - \psi(\mP^*_n\mR_n u, \mQ^*_{n,\alpha}\mQ_{n,\alpha}\theta)| \label{big-response-commutation}\\
        &= |\langle \mS_{n, \alpha} a,\, \sigma(\mS_n b *_n \mR_n u + \mS_{n,\alpha} c)\rangle_{L^2(\M_n, \mu_n)} \nonumber
        \\&\hspace{25pt}- \langle \mS^*_{n, \alpha}\mS_{n, \alpha} a,\, \sigma((\mS^*_n \mS_n b)*(\mP^*_n\mR_n u) + \mS^*_{n,\alpha}\mS_{n,\alpha}c)\rangle_{L^2(\M, \mu)}| \nonumber\\
        &= |\langle \mP^*_n\mS_{n,\alpha} a,\, \mP_n^*\big(\sigma(\mS_n b *_n \mR_n u + \mS_{n,\alpha} c)\big) - \sigma((\mS^*_n \mS_n b)*(\mP^*_n\mR_n u) + \mS^*_{n,\alpha}\mS_{n,\alpha}c)\rangle_{L^2(\M, \mu)} \nonumber\\ &\hspace{25pt}+ \langle \mP^*_n\mS_{n, \alpha} a- \mS^*_{n, \alpha}\mS_{n, \alpha} a,\, \sigma((\mS^*_n \mS_n b)*(\mP^*_n\mR_n u) + \mS^*_{n,\alpha}\mS_{n,\alpha}c)\rangle_{L^2(\M, \mu)}|\nonumber\\
        &\leq 
            \norm{\mP^*_n\mS_{n, \alpha} a}_{L^2}\lVert\mP_n^*\big(\sigma(\mS_n b *_n \mR_n u + \mS_{n,\alpha} c)\big) - \sigma((\mS^*_n \mS_n b)*(\mP^*_n\mR_n u) + \mS^*_{n,\alpha}\mS_{n,\alpha}c)\rVert_{L^2}\nonumber\\
            &\hspace{25pt} + \norm{\mP^*_n\mS_{n, \alpha} a - \mS^*_{n, \alpha}\mS_{n, \alpha} a}_{L^2} C_\sigma\left(\lVert (\mS^*_n \mS_n b)*(\mP^*_n\mR_n u) + \mS^*_{n,\alpha}\mS_{n,\alpha}c\rVert_{L^2} +1 \right).\nonumber
    \end{align}

    By \Cref{convergence_spatial_h_alpha_spectral_discretization} and \Cref{convergence-extended-discretization-sn} it holds that 
    \[\norm{\mP^*_n\mS_{n, \alpha} a}_{L^2}\leq \norm{\mP^*_n\mS_{n, \alpha} a-\mS^*_{n, \alpha}\mS_{n, \alpha} a}_{L^2}+\norm{\mS^*_{n, \alpha}\mS_{n, \alpha} a-a}_{L^2} + \norm{a}_{L^2}\xrightarrow{n\to\infty}\norm{a}_{L^2},\] and additionally \Cref{convergence-extended-discretization-sn-pn-convolution} and \Cref{convergence-extended-discretization-sn}
    \[\norm{\mP^*_n\mS_{n, \alpha} a - \mS^*_n\mS_{n, \alpha} a}_{L^2} C_\sigma\left(\lVert (\mS^*_n \mS_n b)*(\mP^*_n\mR_n u) + \mS^*_{n,\alpha}\mS_{n,\alpha}c\rVert_{L^2} +1\right)\xrightarrow{n\to\infty}0,\] so it remains to treat the other term.

    With \Cref{commute-adjoint-activation}, the linearity of $\mP_n^*$, and the Lipschitz property of $\sigma$, we further obtain
    \begin{align*}
        &\norm{\mP_n^*\big(\sigma(\mS_n b *_n \mR_n u + \mS_{n,\alpha} c)\big) - \sigma\big((\mS^*_n \mS_n b)*(\mP^*_n\mR_n u) + \mS^*_{n,\alpha}\mS_{n,\alpha}c\big)}_{L^2} \\
        &\hspace{25pt} = \norm{\sigma\big(\mP_n^*(\mS_n b *_n \mR_n u + \mS_{n,\alpha} c)\big) - \sigma\big((\mS^*_n \mS_n b)*(\mP^*_n\mR_n u) + \mS^*_{n,\alpha}\mS_{n,\alpha}c\big)}_{L^2} \\
        &\hspace{25pt} \leq L_\sigma \norm{\mP_n^*(\mS_n b *_n \mR_n u + \mS_{n,\alpha} c) - \big((\mS^*_n \mS_n b)*(\mP^*_n\mR_n u) + \mS^*_{n,\alpha}\mS_{n,\alpha}c\big)}_{L^2} \\
        &\hspace{25pt} \leq L_\sigma \left(\norm{\mP_n^*(\mS_n b *_n \mR_n u) - (\mS^*_n \mS_n b)*(\mP^*_n\mR_n u)}_{L^2} + \norm{\mP_n^*\mS_{n,\alpha} c - \mS_{n,\alpha}^*\mS_{n,\alpha} c}\right),
    \end{align*}
    which goes to $0$ as $n\to\infty$ by
    \Cref{convergence-discrete-continuum-convolution-operation} and \Cref{convergence_spatial_h_alpha_spectral_discretization}. 
    Hence, we conclude that \[|\psi_n(\mR_n u, \mQ_{n,\alpha}\theta) - \psi(\mP^*_n\mR_n u, \mQ^*_{n,\alpha}\mQ_{n,\alpha}\theta)|\xrightarrow{n\to\infty}0.\qedhere\]
\end{proof}

As a corollary of the past  two lemmas \Cref{convergence-signal-param-spatial-h-alpha-discretization} and \Cref{convergence-neuron-discretization-pn-h-alpha} we obtain immediately:
\begin{cor}\label{convergence-neural-response-discrete-continuum}
Let $\mR_n: L^2(\M, \mu)\to L^2(\M_n, \mu_n)$ satisfy \cref{{consistent-discretization-operators}}, $u\in L^2(\M, \mu)$ and $\theta\in \Theta$. Then it holds that 
    \[|\psi_n(\mR_n u, \mQ_{n,\alpha}\theta) - \psi(u, \theta)|\xrightarrow{n\to\infty}0.\]
\end{cor}

\section{\texorpdfstring{$\Gamma$}{Gamma}-Convergence of the ERM functionals}\label{gamma}
We are now ready to fully state and prove the main results of this paper. 

To this end, as in \Cref{sec:main-results-discussion}, consider continuum training data $\{(u_k, y_k)\}_{k=1}^l\in (L^2(\M)\times\R)^l$ and a sequence of discretization operators $\mR_n: L^2(\M, \mu)\to L^2(\M_n, \mu_n)$ as in \Cref{sec:main-results-discussion}, inducing discrete training data $\{(\mR_nu_k, y_k)\}_{k=1}^l\in (L^2(\M_n, \mu_n)\times \R)^l$ that are consistent across resolutions.
Alongside the discrete and continuum ERM functionals $\tilde{J}_{n, \alpha}$, and $J_\alpha$ (see \eqref{def-tilde-J-alpha-n}, and \eqref{def-J-alpha}) defined in  \Cref{sec:main-results-discussion}, we now introduce the lifted discrete ERM functionals $J_{n,\alpha}:\cM(\Theta)\to\R^+\cup\{+\infty\}$ as
\begin{equation}\label{def-J-alpha-n}
    J_{n,\alpha}(\nu):=\begin{cases}
        \frac{1}{l}\sum\limits_{k=1}^l \ell(f^\bn_{(\mQ_{n,\alpha})_\# \nu}(\mR_nu_k), y_k) + \zeta\norm{(\mQ_{n,\alpha})_\# \nu}_{TV},& (\mQ_{n,\alpha}^*)_\#(\mQ_{n,\alpha})_\#\nu = \nu,\\
        +\infty,& \text{otherwise.}
    \end{cases}
\end{equation}
We immediately observe that this definition is nothing else than setting
\begin{align*}
    J_{n,\alpha}(\nu)=\begin{cases}
        \tilde{J}_{n,\alpha}((\mQ_{n, \alpha})_\#\nu) & (\mQ_{n,\alpha}^*)_\#(\mQ_{n,\alpha})_\#\nu = \nu,\\
        +\infty,& \text{otherwise,}
    \end{cases}
\end{align*}
hence minimizers of the lifted $J_{n,\alpha}$ and non-lifted $\tilde{J}_{n,\alpha}$ coincide under pushforward by the appropriate parameter projection map $\mQ_{n,\alpha}$ or $\mQ_{n,\alpha}^*$. The lifted discrete ERM functionals allow us to finally formulate the consistency of the training of shallow GCNNs results through $\Gamma$-convergence of the lifted ERM functionals toward the continuum ERM functional and weak-* convergence of their minimizers.

Moreover, in the setting under consideration of supervised learning with a loss involving the finite number $l$ of data points, all of these functionals admit discrete measures as minimizers:

\begin{prop}\label{representer-theorem}
Assume that $\zeta >0$ and that for any $y \in \R$, the function $z \mapsto \ell(z,y)$ is convex and coercive. Then, the functionals $J_\alpha, \tilde{J}_{n,\alpha}$, and $J_{n,\alpha}$ admit at least one minimizer supported on at most $l$ points of $\Theta, \Theta_n$ and $\mQ_{n,\alpha}^* \Theta_n$, respectively.
\end{prop}
\begin{proof}We aim to use the general representer theorem of \cite[Thm.~3.3]{BreCar20}. The linear operator $\mathcal{A}$ in their notation is given by the integral representation of \Cref{def:nn-parametrization-measure-continuum} for $J_\alpha$, the composition of the discretization operator $\mR_n$ with the representation of \Cref{def:nn-parametrization-measure-discrete} for $J_{n,\alpha}$, and with an additional precomposition with $(\mQ_{n, \alpha})_\#$ for $\tilde{J}_{n,\alpha}$, in each case applied to all $u_k$ simultaneously to produce a vector in $\R^l$. The measurement space ($H$ in the notation of \cite{BreCar20}) can be picked as the range of $\mathcal{A}$ corresponding to each case. The fidelity term ($F$ in \cite{BreCar20}) corresponds to applying the loss function $\ell$ at all the labels $y_k$ and taking the sum, with the assumptions on it being implied by the ones we impose on $\ell$. The regularizer ($\phi$ in \cite{BreCar20}) is the total variation norm on the corresponding spaces of measures for $J_\alpha$ and $\tilde{J}_{n,\alpha}$, with the constraint in \Cref{def-J-alpha-n} taken to be part of it as well for $J_{n,\alpha}$. Now, we notice that 
\begin{equation}\label{J-constraint-set}\big\{ \nu \in \cM(\Theta) \, \vert \, (\mQ_{n,\alpha}^*)_\#(\mQ_{n,\alpha})_\#\nu = \nu \big\} = ( \mQ_{n, \alpha}^*)_\# \big( \cM(\Theta_n) \big),\end{equation}
as can be seen directly from $\Theta_n = \mQ_{n, \alpha} \Theta$. Moreover, by continuity of $\mQ_{n,\alpha}$ and $\mQ_{n,\alpha}^*$, the maps $(\mQ_{n,\alpha})_\#$ and $(\mQ_{n,\alpha}^*)_\#$ preserve weak-* convergence of measures, so the set in \eqref{J-constraint-set} is a weak-* closed subspace of $\cM(\Theta)$. Therefore, taking into account that the total variation norm vanishes only for the zero measure, the coercivity assumption [H1] of \cite[Thm.~3.3]{BreCar20} is satisfied in all three cases. Moreover, their range assumption [H0] is satisfied for $J_\alpha$ and $\tilde{J}_{n,\alpha}$ directly by our choice of measurement space, and for $\tilde{J}_{n,\alpha}$ because of its relation with $J_{n,\alpha}$ and the surjectivity of $\mQ_{n, \alpha}$ onto $\Theta_n$ from \Cref{q-nalpha-homeomorphism}.

 Being able to apply \cite[Thm.~3.3]{BreCar20}, the claimed result follows directly by interpreting the extreme points of the sublevel sets of the regularizer in each case. For $J_\alpha$ and $\tilde{J}_{n,\alpha}$ the regularizer is the total variation norm for spaces of measures, with the extreme points of its sub level-sets being proportional to Dirac masses \cite[Thm.~2]{Wer84}. For $J_{n,\alpha}$ we can use that the map $(\mQ_{n,\alpha})_\#$ is linear and injective on the set \eqref{J-constraint-set}, which implies \cite[Lem.~3.2]{BreCar20} that the extreme points of the image with respect to this map are the images of the extreme points.
\end{proof}

Now, we will show that for any regularization parameter $\zeta\geq 0$ the functionals $J_{n,\alpha}$ $\Gamma$-converge to $J_\alpha$, i.e $J_{n,\alpha}\xrightarrow{\Gamma} J_\alpha$ on the space of finite Radon measures, equipped with the weak-* topology $(\cM(\Theta), w^*)$.

\begin{thm}[Uniform convergence of the neural response for consistent data]\label{uniform-convergence}
    For any given $u\in L^2(\M, \mu)$, the sequence of maps $$\{\psi_n(\mR_n u, \cdot)\circ \mQ_{n,\alpha} - \psi(u, \cdot)\}_n \subset C(\Theta, d_w)$$ is relatively compact in the topology induced by the uniform norm, where $d_w$ denotes a metric metrizing the weak topology.
    In particular, it holds that 
    \begin{equation}\label{uniform-convergence-param-neural-response}
    \norm{\psi_n(\mR_n u, \cdot)\circ \mQ_{n,\alpha} - \psi(u, \cdot)}_\infty \xrightarrow{n\to\infty}0 \qquad \text{ for all } u\in L^2(\M, \mu).
    \end{equation}
\end{thm}
\begin{rem}\label{topologies-weak-strong-nbhds}
    The weak topology on the unit ball of $H^\alpha(\M, \mu)$ and $L^2(\M, \mu)$ is metrizable, since they are both separable Hilbert spaces \cite[Thm.~3.28]{Bre11}. Furthermore the weak topology of a finite product equipped with the maximum metric clearly coincides with the product of the spaces with their respective weak topology, so $\Theta$ is also metrizable.
    The Rellich-Kondrachov theorem for fractional Sobolev spaces \cite[Prop.~4.3.4]{Tay23} and uniqueness of limits imply that weak $H^\alpha$ convergence yields also strong $L^2$ -convergence. Therefore the mapping \[(\Theta,w)\to B_{{L^2}}\times B_{{L^2}},\quad  (a,b,c)\mapsto (a,c)\] is continuous  with respect to the weak topology on $\Theta$ and the strong product topology on $B_{{L^2}}\times B_{{L^2}}$.
\end{rem}

\begin{proof}[Proof of Theorem \ref{uniform-convergence}]

First, we will show that for all $u\in L^2(\M, \mu)$, $\theta\in\Theta$ and $\eps>0$ there is an $N_\eps\in\N$ and $\delta>0$ such that for all $\theta'\in B_{d_w}(\theta, \delta)\cap\Theta:=\{\vartheta\in\Theta\mid d_w(\vartheta, \theta)<\delta\}$  and $n\geq N_\eps$ it holds that $\left|\psi_n(\mR_n u, \mQ_{n, \alpha} \theta) - \psi_n(\mR_n u, \mQ_{n, \alpha} \theta') \right|<\eps$.

Let $u\in L^2(\M, \mu)$ be arbitrary but fixed. Fix a $\theta=(a,b,c)\in \Theta$ and $\eps>0$. Let \[d := \frac{\eps}{2\left(C_\sigma \left(2 + C_\mR \norm{u}_{L^2}\right) + 5L_\sigma\right)}.\]

The compactness of the convolution operator $(\cdot* u)\in L(L^2(\M, \mu),L^2(\M, \mu))$ (see \Cref{compactness-convolution-operator}) yields that for any $v\in L^2(\M, \mu)$ and any strong $L^2$-neighborhood $V$ of $v*u$ there exists a weak $L^2$-neighborhood $N$ such that $\{v'*u\mid v'\in N\} =: N*u\subset V$.
Following \Cref{topologies-weak-strong-nbhds}, we can pick a $0<\delta$ small enough such that for all $\theta'=(a',b',c')$ with $d_w(\theta,\theta')<\delta$ we have 
\begin{equation}\label{balls-estimates}
    \norm{a - a'}_{L^2(\M, \mu)}< d, \qquad \norm{(b-b')*u}_{L^2}<d, \quad\text{ and }\quad \norm{c - c'}_{L^2(\M, \mu)}< d. 
\end{equation}

Now, using \Cref{adjoint-l2-norm-preserving}, \Cref{convergence-discrete-continuum-convolution-operation} and  \Cref{convergence-extended-discretization-sn-pn-convolution} we obtain for all $\tilde{b}\in L^2(\M, \mu)$ that \begin{align*}\norm{\mS_n \tilde{b} *_n \mR_n u}_{L^2(\M_n, \mu_n)} &= \norm{\mP_n^*(\mS_n \tilde{b} *_n \mR_n u)}_{L^2(\M, \mu)}\nonumber\\
&\leq \norm{\tilde{b}*u}_{L^2} + \norm{\mP_n^*(\mS_n \tilde{b} *_n \mR_n u) - \mS_n^*\mS_n \tilde{b} * \mP_n^*\mR_n u}_{L^2} \nonumber\\
&\quad+ \norm{\mS_n^*\mS_n \tilde{b} * \mP_n^*\mR_n u - \tilde{b}* u}_{L^2}\nonumber\\
&\leq 2\norm{\tilde{b}*u}_{L^2} + \norm{\tilde{b}}_{L^2}\left(2 C_{\mR}\delta^\phi_n \norm{u}_{L^2} + \norm{u-\mS_n^*\mS_n u}_{L^2} + \norm{u-\mP_n^*\mR_n u }_{L^2}\right),
\end{align*} where $2 C_{\mR}\delta^\phi_n \norm{u}_{L^2} + \norm{u-\mS_n^*\mS_n u}_{L^2} + \norm{u-\mP_n^*\mR_n u }_{L^2} \to 0$ for $n\to\infty$. Due to this and using \Cref{convergence-relative-error-bound-delta-cor}, there exists an $N_\eps\in\N$ such that for all $n\geq N_\eps$ and $\tilde{b}\in L^2(\M, \mu)$ it is 
\[2 C_{\mR}\delta^\phi_n \norm{u}_{L^2} + \norm{u-\mS_n^*\mS_n u}_{L^2} + \norm{u-\mP_n^*\mR_n u }_{L^2}\leq d,\] i.e. in particular also 
\begin{equation}
    \norm{\mS_n \tilde{b} *_n \mR_n u}_{L^2(\M_n, \mu_n)} \leq 2\norm{\tilde{b}*u}_{L^2} + \norm{\tilde{b}}_{L^2}d\label{convolution-estimate}
\end{equation}and furthermore \begin{equation}\label{estimate-delta-term}
    (C\delta(\eps_n, h_n, \lambda_{K(n)})^{\alpha/2}+1)^2\leq 2.
\end{equation}
Next, observe that for any $v,w\in L^2(\M_n, \mu_n)$ it holds that 
\begin{align*}\norm{v*_n w}_{L^2(\M_n)}^2&= \sum_{k=1}^n \langle v*_n w, \varphi_k^\bn\rangle^2 = \sum_{k=1}^n \big\langle\sum_{j=1}^n \langle v, \varphi_j^\bn\rangle\langle w, \varphi_j^\bn\rangle \varphi_j^\bn, \varphi_k^\bn\big\rangle^2 \\
&= \sum_{k=1}^n \langle v, \varphi_k^\bn\rangle^2\langle w, \varphi_k^\bn\rangle^2\leq \sum_{k=1}^n \langle v, \varphi_k^\bn\rangle^2\sum_{k=1}^n\langle w, \varphi_k^\bn\rangle^2 = \norm{v}_{L^2(\M_n)}^2\norm{w}_{L^2(\M_n)}^2,\end{align*}
hence due to \Cref{spectral-h-alpha-preservation-innerproduct} with $\alpha=0$ and \Cref{consistent-discretization-operators} we also obtain the following estimate

\begin{align}\norm{\mS_n \tilde{b} *_n \mR_n u}_{L^2(\M_n, \mu_n)}\leq \norm{\mS_n \tilde b}_{L^2(\M_n)}\norm{\mR_n u}_{L^2(\M_n)}\leq \norm{\tilde{b}}_{L^2(\M)}C_\mR\norm{u}_{L^2(\M)}.\label{estimate-convolution-expression}
\end{align}
Furthermore for all $v\in H^\alpha(\M,\mu)$ it holds due to \Cref{convergence-L2-distance-L2-H-alpha-isomorphisms} and again \Cref{spectral-h-alpha-preservation-innerproduct} that
\begin{align}
    \norm{\mS_{n,\alpha}v}_{L^2(\M_n, \mu_n)}&\leq \norm{\mS_{n,\alpha}v-\mS_n v}_{L^2(\M_n, \mu_n)} + \norm{\mS_n v}_{L^2(\M_n,\mu_n)} \nonumber\\
    &\leq (C\delta(\eps_n, h_n, \lambda_{K(n)})^{\alpha/2}+1)\norm{v}_{L^2(\M, \mu)}\label{estimate-snalpha-l2-norm}.
\end{align}
 In the following calculation we are using again the bilinear extension argument, the Lipschitz property of $\sigma$, the linearity of the discretization  operators, the contractiveness regarding the $L^2$-norm of the spatial and spectral discretization operators, and the fact that $\Theta$ is defined as the product of $H^\alpha$ and $L^2$ unit balls and is therefore contained in the product of $L^2$ unit balls.

 For all $n\geq N_\eps$ and any $\theta' = (a',b',c')\in\Theta$, such that $d_w(\theta, \theta')<\delta$ it holds that
 \begin{align*}
     &\left|\psi_n(\mR_n u, \mQ_{n, \alpha} \theta) - \psi_n(\mR_n u, \mQ_{n, \alpha} \theta') \right| \\ 
     &\hspace{25pt} = \left|\langle \mS_{n, \alpha} a, \sigma(\mS_n b *_n \mR_n u + \mS_{n, \alpha}c)\rangle_{L^2(\M_n)} - \langle  \mS_{n, \alpha} a', \sigma(\mS_n b' *_n \mR_n u + \mS_{n, \alpha} c')\rangle_{L^2(\M_n)}\right| \\
     &\hspace{25pt}\leq \norm{\mS_{n, \alpha}(a-a')}_{L^2(\M_n)} C_\sigma\left(1 + \norm{\mS_n b'*_n \mR_n u}_{L^2(\M_n)} + \norm{\mS_{n, \alpha} c'}_{L^2(\M_n)}\right) \\
     &\hspace{50pt}+ \norm{\mS_{n, \alpha}a}_{L^2(\M_n)} L_\sigma\left(\norm{\left(\mS_n (b-b')\right)*_n \mR_n u}_{L^2(\M_n)} + \norm{\mS_{n, \alpha}(c-c')}_{L^2(\M_n)}\right)\\
     &\hspace{25pt}\leq (C\delta(\eps_n, h_n, \lambda_{K(n)})^{\alpha/2}+1)^2\Bigg( \norm{a-a'}_{L^2(\M)} C_\sigma\left(1+ \norm{b'}_{L^2(\M)}C_\mR\norm{u}_{L^2(\M)} + \norm{c'}_{L^2(\M)}\right) \\
     &\hspace{50pt}+ \norm{a}_{L^2(\M)}L_\sigma\left(2\norm{(b-b')*u}_{L^2(\M)} + \norm{b-b'}_{L^2(\M)} d +\norm{c-c'}_{L^2(\M)}\right)\Bigg) \\
     &\hspace{25pt}< 2\,\big(dC_\sigma(1+1\cdot C_\mR\norm{u}_{L^2(\M)} + 1) + 1\cdot L_\sigma(2d + 2d + d)\big) \\
     &\hspace{25pt}= 2\left( d C_\sigma \left(2 + C_\mR \norm{u}_{L^2(\M)}\right) + 5 dL_\sigma \right) =  2 d\left(C_\sigma \left(2 + C_\mR \norm{u}_{L^2(\M)}\right) + 5L_\sigma\right)=\eps.\end{align*}Here,  the first inequality comes from the computation in \Cref{estimate-neuron-different-signal-and-parameter-inputs} and the second inequality uses \eqref{convolution-estimate}, \eqref{estimate-convolution-expression} and \eqref{estimate-snalpha-l2-norm}. The third inequality follows from the fact that since $b,b'\in B_{L^2}$ it holds that $\norm{b-b'}_{L^2(\M)}\leq 2$, as well as \eqref{estimate-delta-term} and \eqref{balls-estimates}.
 Hence for all $u\in L^2(\M, \mu)$, $\theta\in\Theta$ and $\eps>0$ there is an $N_\eps\in\N$ and $\delta>0$ such that for all $\theta'\in\Theta$ with $d_w(\theta, \theta')<\delta$ and $n\geq N_\eps$ it holds that $\left|\psi_n(\mR_n u, \mQ_{n, \alpha} \theta) - \psi_n(\mR_n u, \mQ_{n, \alpha} \theta') \right|<\eps$.\vspace{0.2cm}\\
Now we show uniform convergence of $\{\psi_{n}(\mR_n u, \cdot)\circ \mQ_{n,\alpha}\}_n$ for all $u\in L^2(\M, \mu)$.
 Fix $u\in L^2(\M, \mu)$ and $\eps>0$. The map $\psi(u, \cdot)$ is weakly continuous (see \Cref{weak-continuity-neural-response-map}), and $(\Theta, d_w)$ is a compact metric space, so in particular it is uniformly continuous, so there is a $d_\eps>0$, such that for $\theta, \theta'\in\Theta$ with $d_w(\theta,\theta')<d_\eps$ it holds that $|\psi(u, \theta)-\psi(u, \theta')|<\eps/3$. The first part of the proof yields that for every $\theta\in \Theta$, there is an $N_{\eps,\theta}\in\N$ and $\delta_\theta>0$ such that for all $n\geq N_{\eps, \theta}$ it holds that \[\left|\psi_n(\mR_n u, \mQ_{n, \alpha} \theta) - \psi_n(\mR_n u, \mQ_{n, \alpha} \theta') \right|<\eps/3 \quad \text{for all }\theta'\in\Theta\text{ with }d_w(\theta, \theta')<\delta_\theta.\]
Since $\Theta$ is compact, the open cover $\Theta \subset \bigcup_{\theta\in\Theta} B_{d_w}(\theta, \min\{\delta_\theta, d_\eps\})$ has a finite subcover, i.e. there are $\{\theta_k\}_{k=1}^K$ such that $\Theta\subset \bigcup_{k=1}^K B_{d_w}(\theta_k, \min\{\delta_{\theta_k}, d_\eps\})$. Let $N_\eps = \max_{k=1, \ldots K} N_{\eps, \theta_k}$. Then  for all $n\geq N_\eps$ and all $\theta\in \Theta$ it holds that 
\[\left|\psi_n(\mR_n u, \mQ_{n, \alpha} \theta) - \psi_n(\mR_n u, \mQ_{n, \alpha} \theta_k) \right|<\eps/3, \quad \text{ where } k\in\{1, \ldots K\} \text{ s.t. } \theta\in B_{d_w}(\theta_k, \delta_{\theta_k}).\] 
Finally the pointwise convergence result \Cref{convergence-neural-response-discrete-continuum} yields that there are $N_{\eps,k}\in\N$ such that for all $k=1, \ldots K$ and $n\geq N_{\eps,k}$ we have \[|\psi_n(\mR_n u, \mQ_{n,\alpha} \theta_k) - \psi(u, \theta_k)|<\eps/3,\] and thus if we set $N'_\eps:=\max_{k=1, \ldots K} N_{\eps, k}$, then for all $n\geq N'_\eps$ it holds \[|\psi_n(\mR_n u, \mQ_{n,\alpha} \theta_k) - \psi(u, \theta_k)|<\eps/3\qquad \text{ for all } k=1, \ldots K.\]

Putting these together gives for every $n\geq \max\{N_\eps, N'_\eps\}$ and for every $\theta\in\Theta$, that there is a $\theta_k$ such that $\theta\in B_{d_w}(\theta_k, \min\{\delta_{\theta_k}, d_\eps\})$ and 
\begin{align*}
    |\psi_n(\mR_n u, \mQ_{n,\alpha}\theta) - \psi(u,\theta)|&
    \leq |\psi_n(\mR_n u, \mQ_{n,\alpha}\theta) - \psi_n(\mR_n u, \mQ_{n,\alpha}\theta_k)| \\
    &\quad + |\psi_n(\mR_n u, \mQ_{n,\alpha}\theta_k) - \psi(u,\theta_k)| + |\psi(u,\theta_k) - \psi(u,\theta)| \\
    &<\eps/3+\eps/3 +\eps/3 =\eps.
\end{align*}
Since $\eps>0$ and $u\in L^2(\M, \mu)$ were arbitrary, this yields
\begin{equation}
    \norm{\psi_n(\mR_n u, \cdot)\circ \mQ_{n,\alpha} - \psi(u, \cdot)}_\infty \xrightarrow{n\to\infty}0 \qquad \text{ for all } u\in L^2(\M, \mu),
\end{equation} i.e. uniform convergence of the neural response as maps over the parameter space for consistent data.
\end{proof}
\begin{cor}\label{cor:convergence-recovery-sequence-measure}
Let $\nu\in \cM(\Theta)$. Then it holds
\[|f^\bn_{(\mQ_{n,\alpha})_{\#}\nu}(\mR_n u) - f_{\nu}(u)| \xrightarrow{n\to\infty} 0 \quad\text{for all } u\in L^2(\M, \mu).\]
\end{cor}
\begin{proof}
It holds that
    \begin{align*}
        \left|f^\bn_{(\mQ_{n,\alpha})_\# \nu}(\mR_n u)- f_{\nu}(u)\right| 
        &= \left|\langle (\mQ_{n, \alpha})_\#\nu, \psi_n(\mR_n u, \cdot)\rangle_{\cM(\Theta_n), C(\Theta_n)} - \langle \nu, \psi(u, \cdot)\rangle_{\cM(\Theta), C(\Theta, w)}\right|\\
        &= \left|\langle \nu, \psi_n(\mR_n u, \cdot)\circ \mQ_{n,\alpha}\rangle_{\cM(\Theta), C(\Theta)} - \langle \nu, \psi(u, \cdot)\rangle_{\cM(\Theta), C(\Theta, w)}\right|\\
        &= \left|\langle \nu, \psi_n(\mR_n u, \cdot)\circ \mQ_{n,\alpha} - \psi(u, \cdot)\rangle_{\cM(\Theta), C(\Theta, w)}\right|\\
        &\leq \norm{\nu}_{TV}\norm{\psi_n(\mR_n u, \cdot)\circ \mQ_{n,\alpha} - \psi(u, \cdot)\rangle_{\cM(\Theta), C(\Theta, w)}}_\infty
    \end{align*}
    which goes to $0$ as $n\to\infty$ as a consequence of \eqref{uniform-convergence-param-neural-response}. 
\end{proof}

\begin{cor}\label{cor:pointwise-convergence-trained-networks-discrete-continuum}Let $\{\nu_n\}_n\subset \cM(\Theta)$ be weak-* converging to a $\nu\in \cM(\Theta)$. 
     Then for all $u\in L^2(\M, \mu)$ it holds that \begin{equation}\label{pointwise-convergence-trained-networks-discrete-continuum}
    \lim\limits_{n\to\infty} |f^\bn_{(\mQ_{n, \alpha})_\#\nu_n}(\mR_n u)-f_\nu(u)| =0.
\end{equation}
\end{cor}

\begin{proof}We can write
    \begin{align*}\left|f^\bn_{(\mQ_{n, \alpha})_\#\nu_n}(\mR_n u)-f_\nu(u)\right| &\leq \left|f^\bn_{(\mQ_{n, \alpha})_\#\nu_n}(\mR_n u)-f^\bn_{(\mQ_{n, \alpha})_\#\nu}(\mR_n u)\right| + \left|f^\bn_{(\mQ_{n, \alpha})_\#\nu}(\mR_n u)-f_\nu(u)\right|.
    \end{align*}
    The last term converges to $0$ by \Cref{cor:convergence-recovery-sequence-measure}.
    The first term we can split up further 
    \begin{align*}
        &\left|f^\bn_{(\mQ_{n,\alpha})_\# \nu_n}(\mR_n u)- f^\bn_{(\mQ_{n,\alpha})_\# \nu}(\mR_n u)\right| \\&\hspace{25pt}= \left|\langle(\mQ_{n,\alpha})_\#(\nu_n-\nu), \psi_n(\mR_n u, \cdot)\rangle_{\cM(\Theta_n), C_0(\Theta_n)}\right|\\
        &\hspace{25pt}= \left|\langle \nu_n-\nu, \psi_n(\mR_n u, \cdot)\circ \mQ_{n,\alpha}\rangle_{\cM(\Theta), C(\Theta,w)}\right|\\
        &\hspace{25pt}\leq  |\langle \nu_n-\nu, \psi_n(\mR_n u, \cdot)\circ \mQ_{n,\alpha} - \psi(u, \cdot)\rangle_{\cM(\Theta), C(\Theta,w)}| + |\langle \nu_n-\nu, \psi(u, \cdot)\rangle_{\cM(\Theta), C(\Theta,w)}|.
    \end{align*}
Now, both terms of this bound converge to zero, for the first because we have weak-* convergence $\nu_n\rightharpoonup\nu $ in the space of measures and norm convergence to $0$ in the space $\big(C(\Theta, w),\norm{\cdot}_{\infty}\big)$ for the right side of the dual pairing (see \eqref{uniform-convergence-param-neural-response}), and for the second it follows immediately from the weak-* convergence of the measures. 
\end{proof}

We are now ready to prove $\Gamma$-convergence (see \cite[Def.~1.5]{Bra02} for a definition) of the training functionals:

\begin{thm}[Discrete-to-Continuum $\Gamma$-convergence of ERM functionals]\label{gamma-convergence}
    It holds that $J_{n,\alpha}\xrightarrow{\Gamma}J_\alpha$ on $(\cM(\Theta,w), \norm{\cdot}_{TV})$. 
\end{thm}

\begin{proof}
    Let $\zeta\geq 0$ arbitrary but fixed.\vspace{-\parskip}
\paragraph*{Liminf inequality:}
Let $\nu\in \cM(\Theta)$ arbitrary but fixed. Let $\{\nu_n\}_n\subset \cM(\Theta)$ such that $\nu_n\xrightharpoonup{*}\nu$ as $n\to\infty$. W.l.o.g. assume that $\nu_{n}=(\mQ_{n,\alpha}^*)_\#(\mQ_{n,\alpha})_\#\nu_n$ for every $n\in\N$, because if it does not hold for infinitely many $n\in\N$, then the desired inequality is satisfied in any case, since $\liminf\limits_{n\to\infty} J_{n,\alpha}(\nu_n)=+\infty \geq J_\alpha(\nu)$, and else the sequence can be reindexed appropriately.

Using \eqref{q-nalpha-continuous-fcts-isomorphism}, for all $n\in\N$ we observe that 
\begin{align*}
    \norm{\nu_n}_{TV} &= \sup\limits_{\substack{\phi\in C(\Theta, w)\\\norm{\phi}_\infty\leq 1 }}|\langle \nu_n, \phi\rangle_{\cM(\Theta), C(\Theta,w)}| 
    = \sup\limits_{\substack{\phi\in C(\Theta, w)\\\norm{\phi}_\infty\leq 1 }}|\langle (\mQ_{n,\alpha}^*\mQ_{n,\alpha})_\#\nu_n, \phi\rangle_{\cM(\Theta), C(\Theta,w)}|\\
    &= \sup\limits_{\substack{\phi\in C(\Theta, w)\\\norm{\phi}_\infty\leq 1 }}|\langle (\mQ_{n,\alpha})_\#\nu_n, \phi\circ\mQ_{n,\alpha}^*\rangle_{\cM(\Theta_n), C(\Theta_n)}| \\
    &= \sup\limits_{\substack{\phi\in C(\mQ_{n, \alpha}^*\Theta_n, w)\\\norm{\phi}_\infty\leq 1 }}|\langle (\mQ_{n,\alpha})_\#\nu_n, \phi\circ\mQ_{n,\alpha}^*\rangle_{\cM(\Theta_n), C(\Theta_n)}| 
    \\&= \sup\limits_{\substack{\phi\in C(\Theta_n)\\\norm{\phi}_\infty\leq 1 }}|\langle (\mQ_{n,\alpha})_\#\nu_n, \phi\rangle_{\cM(\Theta_n), C(\Theta_{n})}| = \norm{(\mQ_{n,\alpha})_\#\nu_n}_{TV}.
\end{align*}
Furthermore from applying \Cref{cor:pointwise-convergence-trained-networks-discrete-continuum} we obtain for all $k\in \{1, \ldots, l\}$ that
\begin{align*}
    \left|f^\bn_{(\mQ_{n,\alpha})_\# \nu_n}(\mR_nu_k)- f_\nu(u_k)\right|\xrightarrow{n\to\infty} 0.
\end{align*}
So with that and the continuity of $\ell$ we may conclude
\begin{align*}
    \liminf\limits_{n\to\infty} J_{n,\alpha}(\nu_n)& = \liminf\limits_{n\to\infty}\frac{1}{l}\sum\limits_{k=1}^l \ell(f_{(\mQ_{n,\alpha})_\# \nu_n}(\mR_nu_k), y_k) + \zeta \norm{(\mQ_{n,\alpha})_\# \nu_n}_{TV} \\
    &\geq \frac{1}{l}\sum\limits_{k=1}^l \liminf\limits_{n\to\infty} \ell(f_{(\mQ_{n,\alpha})_\# \nu_n}(\mR_nu_k), y_k) + \zeta \liminf\limits_{n\to\infty} \norm{\nu_n}_{TV}\\
    &\geq \frac{1}{l}\sum\limits_{k=1}^l  \ell(\lim\limits_{n\to\infty}f^\bn_{(\mQ_{n,\alpha})_\# \nu_n}(\mR_nu_k), y_k) +  \zeta \norm{\nu}_{TV}\\
    &= \frac{1}{l}\sum\limits_{k=1}^l  \ell(f_\nu(u_k) , y_k) + \zeta \norm{\nu}_{TV} = J_\alpha(\nu).
\end{align*}
Here the second to last inequality follows from the weak-* lower semicontinuity of the $TV$-norm.\vspace{-\parskip}
\paragraph*{Existence of a recovery sequence: }
Let $\nu\in \cM(\Theta)$ arbitrary but fixed, and define the recovery sequence $\{\nu_n\}_n\subset \cM(\Theta)$ as $\nu_n:=(\mQ_{n,\alpha}^*)_\#(\mQ_{n,\alpha})_\#\nu$. Note that therefore $(\mQ_{n,\alpha})_\# \nu_n = (\mQ_{n,\alpha})_\#(\mQ_{n,\alpha}^*)_\#(\mQ_{n,\alpha})_\#\nu = (\mQ_{n,\alpha})_\#\nu$.
Because of the pointwise convergence \[\norm{\mQ_{n,\alpha}^*\mQ_{n,\alpha} \theta -\theta}_{(L^2)^3}\xrightarrow{n\to\infty}0\] for all $\theta\in\Theta$, as well as the fact that all $\phi\in C(\Theta, w)$ are bounded by the compactness of $(\Theta, w)$, we have in particular that every $\phi\circ \mQ_{n,\alpha}^*\mQ_{n,\alpha}\in C(\Theta,w)$ is dominated by the constant function $\norm{\phi}_\infty$, i.e. 
\[|\phi(\mQ_{n,\alpha}^*\mQ_{n,\alpha}(\theta))|\leq \sup\limits_{\vartheta\in\Theta}|\phi(\vartheta)|= \norm{\phi}_\infty<\infty,\]\vskip -1.5em and thus $$\lim\limits_{n\to\infty}\langle \nu_n, \phi\rangle =\lim\limits_{n\to\infty}\langle \nu, \phi\circ (\mQ_{n,\alpha}^*\mQ_{n,\alpha})\rangle = \lim\limits_{n\to\infty}\int_\Theta \phi(\mQ_{n,\alpha}^*\mQ_{n,\alpha} \theta)d\nu =\int_\Theta \phi(\theta)d\nu = \langle \nu, \phi\rangle$$ for all $\phi\in C(\Theta,w)$. Thus $\nu_n\xrightharpoonup[n\to\infty]{*}\nu$.
Since $\mQ_{n,\alpha}^*\mQ_{n,\alpha}$ is a continuous operator mapping from $(L^2(\M,\mu))^3$ to $(L^2(\M, \mu))^3$,  for any $\phi\in C(\Theta, w)$ also their composition $\phi\circ \mQ_{n,\alpha}^*\mQ_{n,\alpha}$ is weakly continuous, therefore
\[|\langle \nu_n, \phi\rangle| = |\langle (\mQ^*_n\mQ_n)_\#\nu, \phi\rangle| = |\langle \nu, \phi\circ (\mQ^*_n\mQ_n)\rangle| \leq \sup\limits_{\substack{\psi\in C(\Theta, w)\\ \norm{\psi}_\infty \leq 1}}|\langle \nu, \psi\rangle|. \]\vskip -1.5em
Then 
\begin{align*}
\limsup\limits_{n\to\infty} \norm{\nu_n}_{TV} &= \limsup\limits_{n\to\infty} \sup\limits_{\substack{\phi\in C(\Theta,w)\\ \norm{\phi}_\infty \leq 1}}|\langle \nu_n, \phi\rangle| 
= \limsup\limits_{n\to\infty} \sup\limits_{\substack{\phi\in C(\Theta,w)\\ \norm{\phi}_\infty \leq 1}}|\langle (\mQ_{n,\alpha}^*\mQ_{n,\alpha})_\#\nu, \phi\rangle| \\
&= \limsup\limits_{n\to\infty} \sup\limits_{\substack{\phi\in C(\Theta,w)\\ \norm{\phi}_\infty \leq 1}}|\langle \nu, \phi\circ (\mQ_{n,\alpha}^*\mQ_{n,\alpha})\rangle| \leq \limsup\limits_{n\to\infty} \sup\limits_{\substack{\phi\in C(\Theta,w)\\ \norm{\phi}_\infty \leq 1}}|\langle \nu, \phi \rangle| \\
&= \limsup\limits_{n\to\infty}\norm{\nu}_{TV} = \norm{\nu}_{TV}.
\end{align*}
With the continuity of $\ell$, the construction of the recovery sequence and \Cref{cor:convergence-recovery-sequence-measure} we obtain
\begin{align*}
    \limsup\limits_{n\to\infty} J_{n,\alpha}(\nu_n)&= \limsup\limits_{n\to\infty}\frac{1}{l}\sum\limits_{k=1}^l \ell(f^\bn_{(\mQ_{n,\alpha})_\# \nu_n}(\mR_nu_k), y_k) + \zeta \norm{(\mQ_{n,\alpha})_\# \nu_n}_{TV} \\ 
    &=\limsup\limits_{n\to\infty}\frac{1}{l}\sum\limits_{k=1}^l \ell(f^\bn_{(\mQ_{n,\alpha})_\# \nu}(\mR_nu_k), y_k) + \zeta \norm{\nu_n}_{TV}\\ 
    &\leq\frac{1}{l}\sum\limits_{k=1}^l \ell(\lim\limits_{n\to\infty}f^\bn_{(\mQ_{n,\alpha})_\# \nu}(\mR_nu_k), y_k) + \zeta \limsup\limits_{n\to\infty} \norm{\nu_n}_{TV}\\ 
    &\leq\frac{1}{l}\sum\limits_{k=1}^l \ell(f_\nu(u_k), y_k) + \zeta \norm{\nu}_{TV} = J_\alpha(\nu).\qedhere
\end{align*}
\end{proof}
As a consequence of \Cref{gamma-convergence} cluster points of every sequence of minimizers of $J_{n,\alpha}$ are minimizers of $J_\alpha$. Next, we show that we also have pointwise convergence of the trained graph networks, parametrized by a minimizer of $\tilde{J}_{n, \alpha}$ to a trained manifold network, up to a subsequence, and consistency of the inputs:

\begin{thm}[Uniform convergence of networks on compact sets up to subsequences for consistent inputs and consistent parametrization]\label{thm:uniform-convergence-networks-compact-sets}
    For a fixed $M>0$, let $\tilde{\rho}_n\in B_{\cM(\Theta_n)}(0,M)$, where $n\in\N$, such that there is a $\rho\in \cM(\Theta)$, such that $(\mQ_{n, \alpha}^*)_\#\tilde{\rho}_n=:\rho_n \xrightharpoonup[n\to\infty]{*} \rho$. 
    Then for every compact $K\subset L^2(\M, \mu)$ the sequence of maps \[\{f_{\tilde{\rho}_n}\circ \mR_n  - f_\rho\}_n \subset C(K,\norm{\cdot}_{L^2})\] is relatively compact in the topology induced by the uniform norm.
\end{thm}
\begin{proof}
    Fix $\eps>0$, $u\in L^2(\M, \mu)$, and $\theta\in\Theta$ and let \[\delta<\frac{\eps}{L_\sigma C_\mR M}.\] Then we observe from the uniform boundedness of the discretization operators $\mR_n$ (see \Cref{consistent-discretization-operators}), the Lipschitzianity of $\sigma$ (see \eqref{lipschitz-property-sigma}), and \cref{spectral-h-alpha-preservation-norm} that for all $u'\in B_{{L^2}}(u, \delta)$ it holds 
    \begin{align*}
        |\psi_n(\mR_n u, \mQ_{n, \alpha} \theta)  - \psi_n(\mR_n u', \mQ_{n, \alpha} \theta)|
        &= |\langle \mS_{n,\alpha} a, \sigma(\mS_n b *_n \mR_n u + \mS_{n, \alpha} c)- \sigma(\mS_n b *_n \mR_n u' + \mS_{n, \alpha} c)|\\
        &\leq \norm{\mS_{n,\alpha} a}_{L^2} L_\sigma \norm{\mS_n b *_n (\mR_n u - \mR_n u')}_{L^2}\\
        &\leq \norm{\mS_{n,\alpha}a}_{H^\alpha} L_\sigma \norm{\mS_n b}_{L^2} \norm{\mR_n (u - u')}_{L^2}\\
        &\leq \norm{a}_{H^\alpha} L_\sigma C_\mR \norm{b}_{L^2} \norm{u - u'}_{L^2} <L_\sigma C_\mR \delta < \eps/ M.
    \end{align*}
    This implies that for all for all $u'\in B_{\delta}(L^2(\M, \mu), u)$ we have further 
    \begin{align*}
    \norm{\left(\psi_n(\mR_n u, \cdot) - \psi_n(\mR_n u', \cdot)\right)\circ \mQ_{n, \alpha}}_{\infty} &= \sup_{\substack{\theta\in\Theta, \\\norm{\theta}_{\Theta}\leq 1}}|\psi_n(\mR_n u, \mQ_{n, \alpha} \theta)  - \psi_n(\mR_n u', \mQ_{n, \alpha} \theta)|<\frac{\eps}{M},
    \end{align*}
    which finally yields equicontinuity of $\{f_{\tilde{\rho}_n}\circ \mR_n\}_n$ since for all $u'\in B_{{L^2}}(u, \delta)$:
    \begin{align*}
        |f_{\tilde{\rho}_n}(\mR_n(u))- f_{\tilde{\rho}_n}(\mR_n(u'))| &= |\langle \tilde{\rho}_n, \psi_n(\mR_n u, \cdot) - \psi_n(\mR_n u', \cdot)\rangle_{\cM(\Theta_n), C(\Theta_n)} | \\
        &= |\langle \tilde{\rho}_n, \left(\psi_n(\mR_n u, \cdot) - \psi_n(\mR_n u', \cdot)\right)\circ \mQ_{n, \alpha}\mQ_{n, \alpha}^*\rangle_{\cM(\Theta_n), C(\Theta_n)} | 
        \\&= |\langle (\mQ_{n, \alpha}^*)_\# \tilde{\rho}_n, \left(\psi_n(\mR_n u, \cdot) - \psi_n(\mR_n u', \cdot)\right)\circ \mQ_{n, \alpha}\rangle_{\cM(\Theta), C(\Theta)} | 
        \\& \leq  \norm{(\mQ_{n, \alpha}^*)_\# \tilde{\rho}_n}_{TV} \norm{\left(\psi_n(\mR_n u, \cdot) - \psi_n(\mR_n u', \cdot)\right)\circ \mQ_{n, \alpha}}_\infty < M\frac{\eps}{M}=\eps.
    \end{align*}
    Employing the same strategy as in the first estimate of the proof one observes using \eqref{linear-growth-sigma} instead of \eqref{lipschitz-property-sigma} that for all $n\in\N$ and $u\in L^2(\M, \mu)$, it holds that 
    \begin{align*}
        \norm{\psi_n(\mR_n u, \cdot)\circ \mQ_{n, \alpha}}_\infty \leq C_\sigma\left(1+C_\mR\norm{u}_{L^2}\right)<\infty,
    \end{align*} and hence $\{f_{\tilde{\rho}_n}\circ \mR_n\}_n$ is also pointwise bounded, since then
    \[\left|f_{\tilde{\rho}_n}(\mR_n u)\right| \leq \norm{(\mQ_{n, \alpha}^*)_\#\tilde{\rho}_n}_{TV}\norm{\psi_n(\mR_n u, \cdot)\circ \mQ_{n, \alpha}}_\infty \leq M C_\sigma\left(1+C_\R\norm{u}_{L^2}\right)<\infty. \]
    With the above, Arzelà-Ascoli \cite[Thm. IV.6.7]{Dunford:1966} yields the claim.
\end{proof}
So, we may conclude that up to subsequences and w.r.t. consistent data, trained graph convolutional neural networks with consistent parametrization converge on the discretization of compact sets of manifold signals uniformly to a trained manifold convolutional neural network.

\begin{cor}[Convergence of trained networks]\label{cor:convergence-trained-networks}
    Let $M>0$ and $\{\rho_n\}_n\subset \cM(\Theta)$ be a sequence consisting of minimizers $\rho_n$ of $J_{n,\alpha}$ with $\norm{\rho_n}_{TV}\leq M$ for all $n\in\N$. Then there exists a subsequence $\{\rho_{n_k}\}_k$ that converges weak-* to a $\rho\in\cM(\Theta)$ that is a minimizer of $J_\alpha$, i.e. $\rho_{n_k}\xrightharpoonup[k\to\infty]{*}\rho$. \\Furthermore,  $\tilde{\rho}_{n}:=(\mQ_{n,\alpha})_\#\rho_n$ is a minimizer of $\tilde{J}_{n,\alpha}$ and it holds that
    \begin{lemlist}
        \item $f_{\tilde{\rho}_{n_k}}(\mR_{n_k} u)\xrightarrow{k\to\infty} f_\rho(u)$ for all $u\in L^2(\M, \mu)$, and
        \item for every compact $K\subset L^2(\M, \mu)$ the sequence of maps $\{f_{\tilde{\rho}_{n_k}}\circ \mR_{n_k}  - f_\rho\}_k\subset C(K,\norm{\cdot}_{L^2})$ is relatively compact in the topology induced by the uniform norm.
    \end{lemlist}
\end{cor}
\begin{proof}
    Since $\{\rho_n\}_n$ is bounded in $TV$-norm, we can extract a subsequence $\{\rho_{n_k}\}_k$ that converges weak-* to some $\rho\in\cM(\Theta)$. Hence, as a consequence of \Cref{gamma-convergence}, $\rho$ is a minimizer of $J_\alpha$ \cite[Thm.~1.21]{Bra02}.
    Clearly this implies that for any $u\in L^2(\M, \mu)$ we have \[f_{\rho_{n_k}}(u) = \langle \rho_{n_k}, \psi(u, \cdot)\rangle_{\cM(\Theta), C(\Theta, w)} \xrightarrow{k\to\infty} \langle \rho, \psi(u, \cdot)\rangle_{\cM(\Theta), C(\Theta, w)}=f_\rho(u).\] 
    If $\rho_n$ is a minimizer of $J_{n,\alpha}$ then $\tilde{\rho}_n:=(\mQ_{n,\alpha})_\#\rho_n$ is a minimizer of $\tilde{J}_{n,\alpha}$ and we have $(\mQ_{n,\alpha}^*)_\#(\mQ_{n,\alpha})_\#\rho_n = \rho_n$ for every $n\in\N$, hence $(\mQ_{n, \alpha}^*)_\#\tilde{\rho}_n=\rho_n$. Thus \Cref{cor:pointwise-convergence-trained-networks-discrete-continuum} yields pointwise convergence $f_{\tilde{\rho}_{n_k}}(\mR_{n_k} u)\xrightarrow{k\to\infty} f_\rho(u)$ for all $u\in L^2(\M, \mu)$, and ii) follows from \Cref{thm:uniform-convergence-networks-compact-sets}.
\end{proof}
\begin{rem}
If the regularization parameter $\zeta>0$, then any sequence $\{\rho_n\}_n\subset (\cM(\Theta), \norm{\cdot}_{TV})$ of minimizers of $J_{n, \alpha}$ is bounded by some $M>0$, hence $\{\rho_n\}_n\subset B_{\cM(\Theta)}(0,M)$ and \Cref{cor:convergence-trained-networks} applies.
\end{rem}

\subsection{Existence of minimizers for the continuum functionals}
For completeness, we prove in detail the existence of minimizers of the continuum ERM functionals, even in the absence of the convexity of the loss function $\ell$ that was assumed in \Cref{representer-theorem}.
\begin{thm}\label{thm:erm-lower-semicontinuous}
    $J_\alpha$ is weak-* sequentially lower-semicontinuous.
\end{thm}
\begin{proof}
    We have to show $$\liminf\limits_{j\to\infty} J_{\alpha}(\nu_j)\geq J_{\alpha}(\nu)$$ for any weak-* convergent sequence $\nu_j \wsto \nu$, where $\nu_j, \nu\in \cM(\Theta)$ for all $j\in \N$.
We observe by the non-negativity of all summands in $J_{\alpha}(\cdot)$ and the lower semi-continuity of the total variation norm, and the continuity of the loss function, that 
\begin{align*}
    \liminf\limits_{j\to\infty} J_{\alpha}(\nu_j)&=\liminf\limits_{j\to\infty}\frac{1}{l}\sum\limits_{k=1}^l \ell(f_{\nu_j}(u_k), y_k) + \zeta\norm{\nu_j}_{TV} \\& \geq \frac{1}{l}\sum\limits_{k=1}^l \liminf\limits_{j\to\infty} \ell(f_{\nu_j}(u_k), y_k) + \liminf\limits_{j\to\infty} \zeta\norm{\nu_j}_{TV} \\&\geq \frac{1}{l}\sum\limits_{k=1}^l  \ell(\liminf\limits_{j\to\infty}f_{\nu_j}(u_k), y_k) + \zeta\norm{\nu}_{TV}.
\end{align*}

Recall, that $f_{\rho}$ is of the form $f_{\rho}(u) = \int_{\Theta}\psi(u, \theta)d\rho(\theta)$, so since $\psi(u, \cdot)\in C(\Theta,w)$ for all $u\in L^2(\M, \mu)$ we can once again use Riesz-Markov and observe
$$\ell(\liminf\limits_{j\to\infty}f_{\nu_j}(u_k), y_k) = \ell(\lim\limits_{j\to\infty}f_{\nu_j}(u_k), y_k) = \ell(f_{\nu}(u_k), y_k),$$\vskip -1.2em
since \vspace{-0.5em}\[f_{\nu_j}(u_k)= \int_{\Theta}\psi(u_k, \theta)d\nu_j(\theta) = \langle \nu_j, \psi(u_k, \cdot)\rangle \xrightarrow{j\to\infty} \langle \nu, \psi(u_k, \cdot)\rangle = f_\nu(u_k).\]
This yields sequential lower semi-continuity of $J_{\alpha}$ w.r.t. the weak-* topology.
\end{proof}

\begin{thm}\label{thm:existence-of-minimizers-restriction-to-balls}
    For all $\zeta\geq 0$ and any $M>0$ the restricted functional $J_{\alpha}|_{B_{\cM(\Theta)}(0,M)}$ admits a minimizer. 
\end{thm}
\begin{proof}
By definition $J_{\alpha}$ is bounded below, since both the fidelity term with the loss function and the regularizing term are nonnegative. Hence, so is its restriction to the ball $B_{\cM(\Theta)}(0,M)$, so there exists a minimizing sequence $\{\rho_i\}_i\subset B_{\cM(\Theta)}(0,M)$ for the (regularized) ERM functional $J_{\alpha}$, meaning \vskip -1.5em\[J_{\alpha}(\rho_i)\xrightarrow{i\to\infty} \inf\{J_{\alpha}(\rho)\mid \rho \in B_{\cM(\Theta)}(0,M)\}.\]
We want to show that this minimizing sequence admits a convergent subsequence.
Banach-Alaoglu yields weak-* compactness of the ball $B_{\cM(\Theta)}(0,M)$, and since $\{\rho_i\}_i\subset B_{\cM(\Theta)}(0,M)$, we can extract a weak-* converging subsequence $\{\rho_{i_j}\}_j\subset \{\rho_{i}\}_i$ with limit $\rho_\infty\in B_M(\cM(\Theta))$, i.e. $\rho_{i_j}\wsto\rho_\infty$ for  $j\to\infty.$
In other words\vskip -1em\[\langle \rho_{i_j}, \psi\rangle_{\cM(\Theta), C(\Theta)}\xrightarrow{j\to\infty} \langle \rho_\infty, \psi\rangle_{\cM(\Theta), C(\Theta)} \quad \textrm{for all } \psi\in C(\Theta).\] 
Note, that this is\vskip -1.5em \[\int_{\Theta}\psi(\theta)d\rho_{i_j}(\theta) \xrightarrow{j\to\infty}\int_{\Theta}\psi(\theta)d\rho_\infty(\theta) \quad\textrm{for all }\psi\in C(\Theta).\]
Furthermore $J_\alpha$ is lower-semicontinuous by the previous \Cref{thm:erm-lower-semicontinuous}, i.e. \[\inf\{J_{\alpha}(\rho)\mid \rho \in B_{\cM(\Theta)}(0,M)\} = \lim_{j\to\infty}J_{\alpha}(\rho_{i_j})\geq J_{\alpha}(\rho_\infty).\]
Hence $\rho_\infty\in B_M(\cM(\Theta))$ minimizes the ERM functional on the unit ball of the Radon measures on the parameter space. 
\end{proof}

\begin{thm}
    If $\zeta>0$, then $J_{\alpha}$ admits a minimizer. 
\end{thm}

\begin{proof}
Since $J_\alpha$ is bounded below it admits a minimizing sequence $\{\rho_i\}_i\subset \cM(\Theta)$ for the (regularized) ERM functional $J_{\alpha}$, meaning \vskip -1em\[J_{\alpha}(\rho_i)\xrightarrow{i\to\infty} \inf\{J_{\alpha}(\rho)\mid \rho \in \cM(\Theta)\}.\]
By construction, $J_{\alpha}$ is bounded from below, hence $\{\rho_i\}_i$ being a minimizing sequence implies boundedness of the sequence $\{J_{\alpha}(\rho_i\}_i$.
Since $J_{\alpha}$ is coercive due to the regularizing term, any minimizing sequence $\{\rho_i\}_i$ of $J_{\alpha}$ is bounded, so there is an $M\geq 0$ such that $\{\rho_i\}_i\subset B_{\cM(\Theta)}(0,M)$, which is weak-* compact, so there is a subsequence $\{\rho_{i_j}\}_j$ and a $\rho_\infty\in B_{\cM(\Theta)}(0,M)$, such that $\rho_{i_j}\xrightharpoonup[j\to\infty]{*}\rho_\infty$. Hence the lower-semicontinuity of $J_\alpha$ (\Cref{thm:erm-lower-semicontinuous}) yields \[\inf\{J_{\alpha}(\rho)\mid \rho \in \cM(\Theta)\} = \lim_{j\to\infty}J_\alpha(\rho_{i_j})\geq J_{\alpha}(\rho_\infty),\] so $\rho_\infty$ is a minimizer of $J_\alpha$.
\end{proof}

\subsection{Parametrization with probability measures instead of signed measures}\label{probability-measures}
We have chosen one particular way to parametrize the networks with signed measures in all parameters, which particularly in the case of activation functions such as ReLU introduces redundancies, which gives rise to many different normalization procedures \cite{Weinan:2022}, in particular some leading to working with probability measures. That is also the natural setting for working with Wasserstein gradient flows reflecting first-order optimization on the parameters, as detailed in the next subsection. To this end, define the unreguralized continuum ERM functional 
\[E(\nu) := \frac{1}{l}\sum\limits_{k=1}^l \ell(f_\nu(u_k), y_k) \quad\text{ for }\nu\in\cM(\Theta),\]
and the unreguralized discrete ERM functional $E_n:\cM(\Theta)\to\R^+\cup\{+\infty\}$ as
\begin{align*}
    E_n(\nu)=\begin{cases}
        \frac{1}{l}\sum\limits_{k=1}^l \ell(f_{(\mQ_{n,\alpha})_\# \nu}(S_nu_k), y_k),& \nu\in \cM(\Theta)_n,\\
        +\infty,& \text{otherwise.}
    \end{cases}
\end{align*}

Note that the subset of Radon probability measures $\cP(\Theta)\subset \cM(\Theta)$ is weakly-$*$ closed in the space of finite Radon measures. Thus the liminf inequality in the proof for \Cref{gamma-convergence} still holds if we restrict ourselves to Radon probability measures.

Moreover, for every $\nu\in \cP(\Theta)$, we can also find a recovery sequence $\{\nu_n\}_n\subset \cP(\Theta)$, i.e. s.t. $\nu_n\xrightharpoonup[n\to\infty]{*}\nu$ and $E(\nu)\geq \limsup\limits_{n\to\infty}E_n(\nu_n)$. Namely, for $\nu\in \cP(\Theta)$ and $n\in\N$, we let ${\nu}_n:=(\mQ_{n,\alpha}^*\mQ_{n,\alpha})_\#\nu$. Using the same arguments as for showing the existence of a recovery sequence in the proof for \Cref{gamma-convergence}, we have ${\nu}_n\xrightharpoonup[n\to\infty]{*}\nu$ in $\cM(\Theta)$ and by definition the pushforward preserves mass and positivity, hence $\nu_n\in \cP(\Theta)$ as well.

\section{Outlook}\label{outlook}

One direction for future work would be to consider function outputs instead of scalars, which in combination with the continuum limit would provide a rigorous formalization of an architecture for operator learning on manifolds. One such approach, utilizing the existing neural response map $\psi_n$, would be the approach with measure-valued measures introduced in \cite{DumHerIgl25}. In particular, it enables a hypernetwork setup of networks which map any given input into another network (parametrized in the same form), which is trained as a whole using observation operators with finite-dimensional range. Such an approach requires the measures to have values on a parametrized class of functions $\M \to \R$, which is not immediately comparable with the architectures we used here. A possible setup would be to consider this class of functions to be a reproducing kernel Hilbert space $\mathcal{H}$ (see \cite{DevMueRos21} for an overview of such spaces on Riemannian manifolds), leading to a representation of the form \[f_\nu(u) = \int_{\Theta_n} \psi_n(u,\theta) d\nu(\theta) \quad \text{for }\nu \in \cM(\Theta_n, \mathcal{H}),\]
noting that the duality between continuous functions and measures carries over to such a vector-valued setting by Singer's representation theorem \cite{Sin57}. Convergence of global minimizers of the training problems on measures, as treated here, would require making concrete the meaning of consistent data in that setting. In this context, it would be even more important to understand which kind of nonlinear operators can be approximated with such a procedure. To that end, the use of Sobolev spaces for parameters as done here makes interpretable characterizations that might be related to PDE solution operators more likely.

Another direction would be more realistic settings for training. While the convex formulation over measures is very convenient, it fails to reflect many additional aspects appearing when training with a procedure involving gradient descent on the parameters themselves. On the one hand, the global minimum may not be attained. On the other, overparametrized situations in which there might be many global minimizers are also not accurately reflected since the purely energy-based approach cannot distinguish which such minimizer is more likely to be reached. Similarly to the way in which working with measures allows for both finite and infinite-width instances simultaneously, formulating the training as a Wasserstein gradient flow on a space of parameter probability measures as in \cite{ChiBac18} makes optimization insights applicable also in our setting. A natural question to ask would be whether the discrete trajectories converge to the continuum ones when the training data is assumed to arise from continuum one, such as with the $\mR_n$ discretization operators we have used here.

Moreover, the focus on this work has been on shallow networks, in which the representation from measures to network realizations is linear, and the training problems admit a convex reformulation. On the other hand, most architectures currently used in practical settings include at least some level of depth, and since we are focused on training and not just results on the forward pass of networks, the results also do not generalize directly. To extend our discrete-to-continuum approach to deep networks, it would be conceivable to attempt to combine it with one of the recent frameworks that extend the measure parametrization point of view to deep networks, like \cite{BarDevRosVig24, HeeSpeBru25} which take a point of view of global minimization and representer theorems, or \cite{BarPeyVia25} in the direction of gradient descent for the case of ResNets. The interactions of our framework with such theories remain to be explored.

\bibliographystyle{plain}
\bibliography{gcnn-manifold}

\end{document}